\documentclass[twoside]{article}
\usepackage{graphicx}
\usepackage{subcaption}
\usepackage[accepted]{aistats2026}
\usepackage[round]{natbib}

\usepackage{url}
\usepackage{xcolor}
\definecolor{aistatsblue}{RGB}{0, 20, 115}

\usepackage[colorlinks=true,
            citecolor=aistatsblue,
            linkcolor=aistatsblue,
            urlcolor=aistatsblue]{hyperref}
\usepackage{algorithm}
\usepackage{algorithmic}
\usepackage{amsfonts}
\usepackage{amsmath}
\usepackage{amsthm}
\usepackage{amssymb}

\newcommand{\mc}{\mathcal}

\usepackage{amsmath}

\usepackage{booktabs}
\usepackage{graphicx}
\usepackage{thmtools, thm-restate} 
\declaretheorem[name=Theorem,numberwithin=section]{theorem}
\declaretheorem[name=Proposition,sibling=theorem]{proposition}
\declaretheorem[name=Lemma,sibling=theorem]{lemma}
\declaretheorem[name=Corollary,sibling=theorem]{corollary}

\declaretheorem[name=Definition,sibling=theorem]{definition}

\usepackage{optidef}
\newtheorem{problem}{Problem}
\usepackage{booktabs}
\usepackage{tabularx}
\usepackage{makecell}
\usepackage{array}
\usepackage{longtable}

\newcommand{\notationgroup}[1]{%
    \multicolumn{2}{@{}l}{\textbf{#1}}\\[0.15em]
}

\usepackage[utf8]{inputenc}
\usepackage{bbding}
\usepackage{pifont}
\usepackage{wasysym}
\newtheorem{remark}[theorem]{Remark}

\begin{document}

%

%

\twocolumn[

\aistatstitle{Online Learning-to-Defer with Varying Experts}

\runningauthor{Dang, Montreuil, Meyer, Carlier, Ng, and Tsang}

\aistatsauthor{
Dang Hoang Duy$^{*, 1}$ \And
Yannis Montreuil$^{*,1,4,5}$ \And
Maxime Meyer$^{*,4,6}$ \AND
Axel Carlier$^{2,4}$ \And
Lai Xing Ng$^{3,4}$ \And
Wei Tsang Ooi$^{1,4}$
}

\aistatsaddress{
$^{1}$School of Computing, National University of Singapore, Singapore \\
$^{2}$Fédération ENAC ISAE-SUPAERO ONERA, Université de Toulouse, France \\
$^{3}$Institute for Infocomm Research, A*STAR , Singapore \\
$^{4}$IPAL, IRL 2955, Singapore \\
$^{5}$CNRS@CREATE LTD, 1 Create Way, Singapore \\
$^{6}$Department of Mathematics, National University of Singapore, Singapore, 117543 \\

}

]

\begin{abstract}
Learning-to-Defer (L2D) methods route each query either to a predictive model or to external experts. Real-world deployments require handling streaming data, changing expert availability, shifting expert reliability, and feedback observed only for the selected action. We introduce an online multiclass L2D algorithm that combines queried-action bandit feedback with a dynamically varying pool of experts. Let $N=n+n_e$, let $B$ bound the Frobenius norm of the linear score matrix, and let $\rho$ bound the augmented input norm. Assuming linear calibration and zero surrogate minimizability gap for the projected comparator class, our method achieves expected true-deferral regret $O((BN^{3/2}\rho+1)T^{2/3})$, improving to $O(BN^{3/2}\rho\sqrt T+B^2N^3\rho^2)$ under a concentrated-score condition. The analysis combines an online $\mathcal H$-consistency transfer bound with projected online convex optimization. Experiments on synthetic and real-world datasets demonstrate selective routing under varying expert availability and reliability.

\end{abstract}

\section{INTRODUCTION}

Learning-to-Defer (L2D) is a framework for integrating machine learning models with human experts---or more broadly, auxiliary decision-makers---to enhance the reliability of automated systems. 
The central idea is to let a predictive model handle routine or low-risk inputs autonomously, while deferring uncertain or high-stakes cases to more competent agents such as domain experts, larger models, or specialized subsystems~\citep{madras2018predict, mozannar2021consistent, Verma2022LearningTD}. 
This deference mechanism enables L2D systems to balance predictive accuracy against consultation cost, leveraging the scalability of machine learning while preserving safety through expert oversight. 
For instance, in medical diagnostics, a model may issue a confident prediction for straightforward cases but defer ambiguous ones to a clinician, thereby reducing risk without overburdening human resources~\citep{johnson2016,Johnson2023,strong2024towards}. 

While L2D has shown promise, existing formulations rely on assumptions that rarely hold in practice. 
Standard fully supervised score-based formulations commonly posit a fixed pool of experts available throughout deployment and use predictions from every expert in training, although limited-expert-prediction variants reduce this requirement~\citep{hemmer2023limited}. In reality, expert availability and reliability are dynamic: clinicians may be available only during scheduled shifts, their accuracy may degrade under fatigue, and automated subsystems may be restricted by budget or resource constraints. 
Moreover, many applications generate data sequentially, demanding decisions in an online rather than batch setting. 
These conditions expose a fundamental gap between classical L2D assumptions and the requirements of real-world deployment.

To our knowledge, we present the first multiclass online Learning-to-Defer framework that combines queried-action bandit correctness feedback with a dynamically varying multi-expert pool. Both the data distribution and the available pool may evolve over time. 
The formulation further accounts for experts whose reliability may drift and requires feedback only for the action actually queried, rather than full expert feedback on each round. 
To assess the theoretical performance of our approach, we analyze regret, which measures the gap between the learner’s cumulative performance and that of the best strategy in hindsight over $T$ rounds. 
We prove that our algorithm achieves regret sublinear in $T$, ensuring that the average per-round regret vanishes asymptotically.

Our main contributions are as follows: (\textbf{i}) We formulate multiclass online Learning-to-Defer with queried-action bandit feedback, time-varying expert availability, and drifting conditional performance. (\textbf{ii}) We derive an online $\mathcal{H}$-consistency transfer bound that explicitly accounts for minimizability gaps and exploration. (\textbf{iii}) Assuming linear calibration and zero surrogate minimizability gap for the projected comparator class, for $N=n+n_e$, Frobenius radius $B$, and augmented input radius $\rho$, our projected first-order algorithm achieves expected true-deferral regret $O((BN^{3/2}\rho+1)T^{2/3})$, improving to $O(BN^{3/2}\rho\sqrt T+B^2N^3\rho^2)$ under a concentrated-score condition. (\textbf{iv}) We empirically validate selective routing on synthetic and real-world datasets across fixed and drifting expert scenarios.

\section{RELATED WORK}

\paragraph{Learning-to-Defer.} 
Learning-to-Defer (L2D) generalizes selective prediction~\citep{Chow_1970, Bartlett_Wegkamp_2008, cortes, Geifman_El-Yaniv_2017} by allowing a model not only to abstain on uncertain inputs but also to defer them to external experts~\citep{madras2018predict, mozannar2021consistent, Verma2022LearningTD}. Much of the foundational score-based L2D literature focuses on batch learning; our setting instead combines sequential decisions, queried-action feedback, and a varying multi-expert pool.

Recent work also studies online deferral with budgeted partial feedback
\citep{reid2024online} and sequential human review with endogenous delays
\citep{lykouris2024content}. Our formulation is complementary: it focuses on
multiclass score-based learning with a dynamically varying pool of multiple
experts and observes correctness only for the queried action.

The score-based formulation of \citet{mozannar2021consistent} introduced the first Bayes-consistent surrogate, which has since become the standard for one-stage frameworks. Subsequent work has refined calibration~\citep{Verma2022LearningTD, Cao_Mozannar_Feng_Wei_An_2023}, improved surrogate design~\citep{charusaie2022sample, mao2024principledapproacheslearningdefer, montreuil2026why, montreuil2026beyond, montreuil2026learningtodeferexpertconditionedadvice}, and established guarantees such as $\mathcal{H}$-consistency and realizability~\citep{Mozannar2023WhoSP, mao2024realizablehconsistentbayesconsistentloss, mao2025mastering, mao2025thesis}. Further refinements address budgeted or imbalanced deferral regimes~\citep{desalvo2025budgeted, cortes2026optimized} and non-stationary streams~\citep{montreuil2026learningdefernonstationarytime}. Applications span diverse classification tasks~\citep{Kerrigan, Keswani, Hemmer, Verma2022LearningTD, Benz, Cao_Mozannar_Feng_Wei_An_2023, Tailor, liu2024mitigating}, while a special case of this formulation is the \emph{two-stage} framework, in which the base predictor and experts are trained offline and only the allocation function is learned~\citep{Narasimhan, mao2023twostage}. Extensions of two-stage approaches further address regression~\citep{mao2024regressionmultiexpertdeferral}, multi-task learning~\citep{montreuil2024twostagelearningtodefermultitasklearning}, adversarial robustness~\citep{montreuil2025adversarialrobustnesstwostagelearningtodefer, montreuil2026adversarial}, and applied systems~\citep{strong2024towards, palomba2025a, montreuil2025optimalqueryallocationextractive}. The $\mathcal{H}$-consistency bounds we invoke to convert surrogate regret into target deferral regret originate in a broad line of consistency theory~\citep{Awasthi_Mao_Mohri_Zhong_2022_multi, mao2024h}, with instantiations for abstention and rejection~\citep{theoretically, Mao_Mohri_Zhong_2023, mohri2024learningreject}, adversarial losses~\citep{Grounded}, ranking and cardinality-aware prediction~\citep{mao2023pairwisemisranking, mao2023rankingabstention, cortes2024cardinalityaware}, structured prediction~\citep{mao2023structuredprediction}, and recent generalized-metric and robust generative algorithms~\citep{mohri2026generalized, mohri2026principled, cortes2026theoretical}.

\paragraph{Online-Learning.} 

The most relevant line of work to our paper is bandit multiclass classification. A common approach employs \emph{proper learning algorithms} that fix a hypothesis before the next input \citep{kakade2008efficient}. Subsequent papers refine this idea \citep{hazan2011newtron, beygelzimer2017efficient, van2020exploiting, van2021beyond}, with \citet{van2020exploiting} exploiting the gap between surrogate and zero-one loss. We also use a proper learner, but we evaluate performance directly in terms of the zero-one loss rather than only through a surrogate. Our distributional assumptions are broader than those in \citet{crammer2013multiclass}, which treat a specific label distribution, while our results cover arbitrary label distributions.

Another relevant line of work is online logistic regression. \citet{foster2018logistic} give a statistically optimal strategy that is not computationally practical. \citet{jezequel2020efficient} provide an efficient algorithm in the binary case. More recent work \citep{jezequel2021mixability,agarwal2022efficient} develops algorithms with near-optimal regret and efficient running time, and shows how to use them for bandit classification. Our analysis follows this template. We first study the underlying online convex optimization problem and then combine it with a separate argument to obtain guarantees for multiclass bandit classification.

With a fixed number of experts per round, our setting contains contextual multi-armed bandit instances \citep{auer2002nonstochastic,li2010}. Under the usual richness assumptions on the comparator class, these instances inherit the classical $\sqrt T$ bandit scale. This motivates exploiting the linear score structure while controlling the true zero-one loss within a proper-learner framework. Further discussion is deferred to Appendix~\ref{appendix:related_works}.



\section{PRELIMINARIES}
\label{section:preliminaries}

\paragraph{Multiclass Classification.} We consider the standard multiclass classification setting. The input space is $\mathcal{X} = \{x \in \mathbb{R}^{d} : \|x\|_{2} \leq R\}$ for some radius $R > 0$, and the label space is $\mathcal{Y} = [n] := \{1, \ldots, n\}$ with $n \geq 2$ classes \citep{Foundations}. Fix a class $\mathcal H$ of score functions $h:\mathcal X\times\mathcal Y\to\mathbb R$; $\mathcal H_{\mathrm{all}}$ denotes all measurable such functions. Predictions use $h(x)=\arg\max_{j\in\mathcal Y}h(x,j)$. The $0$–$1$ loss is $\ell_{01}(h,x,y)=\mathbf1\{h(x)\neq y\}$. Since it is non-differentiable, learning algorithms typically minimize a surrogate \citep{Statistical,Steinwart2007HowTC}. We denote by $\Phi_{01}:\mathcal H\times\mathcal X\times\mathcal Y\to\mathbb R_+$ a general surrogate loss. Two prominent families are the cross-entropy family $\Phi^{\mathrm{cse}}_{01}$ and constrained surrogates $\Phi^{\mathrm{cstnd}}_{01}$ \citep{awasthi2021calibrationconsistencyadversarialsurrogate,Awasthi_Mao_Mohri_Zhong_2022_multi,mao2023crossentropylossfunctionstheoretical}. A surrogate is $\mathcal H$-\emph{consistent} if its excess risk controls the target excess risk relative to $\mathcal H$.
For any loss $L$, define $\mathcal E_L(h)=\mathbb E_P[L(h,X,Y)]$, $\mathcal E^*_{L,\mathcal H}=\inf_{h\in\mathcal H}\mathcal E_L(h)$, and $\mathcal U_{L,\mathcal H}=\mathcal E^*_{L,\mathcal H}-\mathbb E_X[\inf_{h\in\mathcal H}\mathbb E[L(h,X,Y)\mid X]]$.

\begin{definition}[$\mathcal{H}$-Consistency Bounds] Suppose the surrogate $\Phi_{01}$ is $\mathcal{H}$-consistent with respect to $\ell_{01}$ for every joint data distribution $P$. Then there exists a non-decreasing function $\Gamma : \mathbb{R}_+ \to \mathbb{R}_+$ such that, for all $h \in \mathcal{H}$,
\begin{equation*} \begin{aligned} \mc{E}_{\ell_{01}}(h)& - \mc{E}^\ast_{\ell_{01}, \mc{H}} + \mathcal{U}_{\ell_{01}, \mc{H}} \\ & \leq \Gamma\!\left( \mc{E}_{\Phi_{01}}(h) - \mc{E}^\ast_{\Phi_{01}, \mc{H}} + \mathcal{U}_{\Phi_{01}, \mc{H}} \right). \end{aligned} \end{equation*}
\end{definition}
The quantity $\mathcal U_{L,\mathcal H}$ is the \emph{minimizability gap}. It vanishes for sufficiently rich classes, such as $\mathcal H=\mathcal H_{\mathrm{all}}$, recovering standard Bayes consistency~\citep{Steinwart2007HowTC,Awasthi_Mao_Mohri_Zhong_2022_multi}.

\paragraph{Batch Learning-to-Defer.} Standard batch L2D assumes access to a fixed pool of $n_e$ experts $\mc{G}=\{g_1, g_2, \dots, g_{n_e}\}$, each being a mapping $g_i : \mathcal{X} \to \mathcal{Y}$~\citep{madras2018predict, mozannar2021consistent}. Training examples $(x,y)$ are drawn i.i.d.\ from an unknown joint distribution $P$, and for every input $x$ we observe the tuple $(y,g_1(x), g_2(x), \dots, g_{n_e}(x))$.

The learner then has to decide, for each input, whether to predict directly or defer to an expert.
To formalize this, we define the augmented label space $\overline{\mathcal{Y}} = \mathcal{Y} \cup \{n+j: g_j \in \mc{G}\}$, where indices $1,\dots,n$ correspond to class predictions and indices $n+j$ correspond to deferral to expert $g_j$.
Let $\mathcal H_{\mathrm{def}}^{\mathrm{batch}}\subseteq\mathbb R^{\mathcal X\times\overline{\mathcal Y}}$ be the batch deferral score class, and let $\Phi_{01}^{\mathrm{batch}}:\mathcal H_{\mathrm{def}}^{\mathrm{batch}}\times\mathcal X\times\overline{\mathcal Y}\to\mathbb R_+$ be its augmented multiclass surrogate.

The goal of the learner is to find a hypothesis $h : \mathcal{X} \times \overline{\mathcal{Y}} \to \mathbb{R}$ that minimizes the true score-based loss:
\begin{restatable}[True Score-Based Loss]{definition}{claloss}\label{def_cla_l2d}
Let $x \in \mathcal{X}$, $y \in \mathcal{Y}$, and $h \in \mathcal H_{\mathrm{def}}^{\mathrm{batch}}$. The true score-based loss is
\[
\ell_{\mathrm{def}}^{\mathrm{batch}}(h(x),y) =
\begin{cases}
\mathbf{1}\{h(x)\neq y\} & \text{if } h(x)\in\mathcal{Y},\\[0.4em]
c_j(x,y) & \text{if } h(x)=n+j.
\end{cases}
\]
\end{restatable}
Here, \(c_j(x,y)\) accounts for both the misclassification and the query cost of expert \(g_j\). Throughout the theory, expert costs are normalized to \([0,1]\): we denote by \(0\leq\underline c_j\leq\overline c_j\leq1\) bounds satisfying \(c_j(x,y)\in[\underline c_j,\overline c_j]\) for every \((x,y)\).

In practice, let the raw expert cost be
\(r_j(x,y)=\alpha_j\mathbf 1\{g_j(x)\neq y\}+\beta_j\), where
\(\alpha_j,\beta_j\geq0\). We use the normalized cost
\[
c_j(x,y)=\frac{r_j(x,y)}{Q_j},\qquad
Q_j=\max\{1,\alpha_j+\beta_j\}.
\]
This leaves the classifier mistake cost equal to one and guarantees
\(c_j\in[0,1]\). In the common case \(\alpha_j=1\), a wrong expert action has
normalized cost one, while a correct expert action still incurs the positive
query cost \(\beta_j/(1+\beta_j)\).

Directly minimizing $\ell_{\mathrm{def}}^{\mathrm{batch}}$ is intractable due to its non-differentiability \citep{Statistical, Steinwart2007HowTC}. 
Following \citet{mozannar2021consistent, mao2024principledapproacheslearningdefer}, we instead optimize a convex surrogate loss that is Bayes and $\mathcal H_{\mathrm{def}}^{\mathrm{batch}}$-consistent for $\ell_{\mathrm{def}}^{\mathrm{batch}}$. 
\begin{restatable}[Score-Based Surrogate]{definition}{surrclaloss}\label{def_surr_l2d}
Let $x \in \mathcal{X}$, $y \in \mathcal{Y}$, and $h \in \mathcal H_{\mathrm{def}}^{\mathrm{batch}}$. The score-based surrogate is
\begin{equation*}
\begin{aligned}
\Phi_{\mathrm{def}}^{\mathrm{batch}}(h&,x,y) \; = \; \Phi_{01}^{\mathrm{batch}}(h,x,y) \\
& +\sum_{j=1}^{n_e} (1-c_j(x,y))\,\Phi_{01}^{\mathrm{batch}}(h,x,n+j),
\end{aligned}
\end{equation*}
\end{restatable}

\section{EXTENDING LEARNING-TO-DEFER TO THE ONLINE-LEARNING SETTING}
\label{section:extending_to_online}

\subsection{Problem Formulation}
\label{section:problem_formulation}
\paragraph{Problem Setup.} Prior work assumes access to a fixed pool of experts~\citep{madras2018predict, mozannar2021consistent, Verma2022LearningTD}, but this assumption rarely holds in practice. In realistic deployments, the set of available experts may evolve over time: human annotators can join or leave, domain specialists may be consulted only intermittently, and automated subsystems may become unavailable due to resource or reliability constraints. This motivates the introduction of the online framework, in which the learner receives information and makes decisions sequentially. To capture this, we denote by $\mathcal{A} \subseteq \mathcal{P}(\mathcal{G})$ the collection of feasible expert subsets, and $\overline{\mathcal X}=\mathcal X\times \mathcal A$ the augmented input space. Thus the augmented inputs $\overline x=(x,A)$ consist of both the classical input $x\in\mathcal X$ and the set of available experts $A\in\mathcal A$.

The interaction protocol is as follows: (\textbf{i}) At each round $t$, the adversary selects an augmented input $\overline{x}_t \in \overline{\mathcal{X}}$ adversarially \citep{introductiononlineconvex}, which is then revealed to the learner. (\textbf{ii}) A label $y_t$ is drawn from a fixed but unknown conditional kernel $p(\cdot \mid \overline{x}_t)$. (\textbf{iii}) The learner commits to a prediction $y_t'$ from the augmented label set with experts available at round $t$, namely $\overline{\mathcal{Y}}_{A_t} = [n] \cup \{n+j : g_j \in A_t\}$. (\textbf{iv}) The learner then incurs a loss determined by its prediction.

We allow expert predictions and costs to vary with time. Formally, the
predictable system state, including the round index whenever needed, is included
in $x_t$; hence $g_j(x_t)$ and $c_j(x_t,y_t)$ remain well-defined functions of
the augmented input. The input, availability, cost, and conditional-distribution
sequences are fixed before the learner's randomization, but may vary arbitrarily
with $t$. We use bold notation for sequences, so
$(\overline x_1,\ldots,\overline x_T)$ is denoted $\overline{\mathbf x}$.

\paragraph{Bandit Feedback.} We work in the bandit feedback setting, where the learner only observes feedback about the chosen action. If it predicts a label $y_t' \in [n]$, it does not receive the true label $y_t$ but only a binary signal indicating correctness, $\mathbf{1}\{y_t' = y_t\}$. If it defers to expert $g_j$ (i.e., $y_t' = n+j$), it observes whether the expert was correct, $\mathbf{1}\{g_j(x_t) = y_t\}$, together with the associated cost $c_j(x_t,y_t)$. 

In contrast, standard fully supervised batch Learning-to-Defer training observes the true label and the correctness of all experts. Limited-expert-prediction batch methods relax that data requirement, while our protocol is sequential and reveals feedback only for the selected action. This makes the bandit setting more restrictive than standard supervised classification, where the learner receives the full label.

\paragraph{Hypothesis Set.} As for standard batch L2D, we consider that the learner predicts through a hypothesis $h : \overline{\mathcal{X}} \times \overline{\mathcal{Y}} \to \mathbb{R}$. For every augmented input $\overline{x}_t=(x_t,A_t)$, the prediction of the learner is then $h(\overline{x}_t) = \arg\max_{y \in \overline{\mathcal{Y}}_{A_t}} h(\overline{x}_t,y)$. We only consider the linear hypothesis set, \textit{id est} hypothesis in the form \(h(\overline{x},y) = \langle w_y,x \rangle + b_y\) where \(w_y\) is a vector of some weight matrix \(W \in \mathbb{R}^{N \times d}\), \(b_y\) is the corresponding bias, and $N=n+n_e$ is the size of the augmented label set. We denote by $\mathcal H$ the set of such hypotheses. Note that in this work, we only consider proper learning algorithms. That is the hypothesis \(h_t \in \mathcal{H}\) at each round $t$ is generated before seeing input \(\overline{x}_t\). Finally, the goal of the learner is to minimize the gap between his sequence of hypotheses $\mathbf h$ and the best hypothesis in hindsight $h$.

\paragraph{Semi-Adversarial Setting.} We deliberately adopt a semi-adversarial framework, where inputs $\overline{x}_t$ may be chosen adversarially while labels are drawn from the fixed conditional kernel $p$. This choice reflects the core structure of Learning-to-Defer: in realistic deployments, the instances presented to the system—such as medical cases, flagged content, or decision-support queries—may arrive in arbitrary or even adversarial ways, yet both the learner and the experts must ultimately be evaluated against the same ground-truth mechanism. Without this structure, the notion of deferring to experts would be ill-posed, since expert reliability can only be assessed relative to a consistent source of truth.

To denote the induced distributions, given any finite set \(\mathcal S\), we thus define \(\Delta_{\mathcal S} \subset \mathbb{R}^{|\mathcal S|}\) as the set of standard \(|\mathcal S|\)-dimensional simplexes representing distributions over \(\mathcal S\). We denote by \(\mathcal{D}\) a set of conditional distributions over \(\mathcal{Y}\), indexed by \(\overline{x}\), that is, \(\mathcal{D} \subseteq \{p: \overline{\mc{X}} \to \Delta_{\mathcal{Y}}\}\). For ease of notation, write \(p(\overline{x},y) = \mathbb{P}(Y = y | \overline{X} = \overline{x})\). 

This setting forces us to randomize the learner's predictions, as discussed in Appendix~\ref{appendix:benchmark}. At each round $t$, we thus sample the predictions from a distribution \(q_t \in \Delta_{\overline{\mathcal{Y}}_{A_t}}\): 
\[
q_t = (1 - \gamma_t)e_{h_t(\overline{x}_t)} + \frac{\gamma_t}{|\overline{\mathcal{Y}}_{A_t}|}\mathbf{1},
\]
where \(e_y\) is the unit vector for action \(y\) and \(\gamma_t\in[0,1]\) is the exploration rate; importance-weighted bandit rounds require \(\gamma_t>0\), whereas full-information rounds use \(\gamma_t=0\).

\subsection{Loss Formulation}

As introduced in Definition~\ref{def_cla_l2d}, the classical Learning-to-Defer formulation~\citep{mozannar2021consistent} is defined in a batch-learning setting, assuming access to a fixed dataset and a fixed pool of experts. To relax these assumptions, we introduce a \emph{true deferral loss}, which naturally generalizes Definition~\ref{def_cla_l2d} to accommodate scenarios where the set of available experts may be restricted.

\begin{restatable}[True Deferral Loss]{definition}{truedeferralloss}\label{true_deferral_loss}
    Fix \(h \in \mathcal{H}\) and
    \((\overline{x},y)=((x,A),y)\in\overline{\mathcal{X}}\times\mathcal{Y}\),
    and write \(a_h=h(\overline x)\). The true deferral loss is
    \[
    \ell_{\mathrm{def}}(h,\overline{x},y)=
    \begin{cases}
        \mathbf{1}\{a_h\neq y\}, & a_h\in\mathcal Y,\\
        c_j(x,y), & a_h=n+j,\quad g_j\in A.
    \end{cases}
    \]
Moreover, let the online learner predict according to a distribution \(q \in \Delta_{\overline{\mathcal{Y}}_{A}}\). The corresponding true deferral loss with respect to the learner is defined by
    \begin{equation*}
        \begin{aligned}
        \ell_{\mathrm{def}}(q,\overline{x},y)
        &=\sum_{y'\in\mathcal Y}q(y')\mathbf 1\{y'\neq y\}\\
        &\quad+\sum_{n+j\in\overline{\mathcal Y}_A}q(n+j)c_j(x,y).
        \end{aligned}
    \end{equation*}
\end{restatable}
Note that our true deferral loss, introduced in Definition~\ref{true_deferral_loss}, directly extends the formulations of \citet{mozannar2021consistent,Verma2022LearningTD,mao2024principledapproacheslearningdefer}. 

Since $\ell_{\mathrm{def}}$ is non-differentiable, the standard practice in online classification is to replace it with a convex surrogate $\Phi_{\mathrm{def}}$, optimize $\Phi_{\mathrm{def}}$ using Online Convex Optimization (OCO) \citep{introductiononlineconvex} methods such as Online Gradient Descent (OGD) \citep{zinkevich2003online}, and relate surrogate regret to true regret through a consistency bound.

\begin{restatable}[Varying Multiclass Surrogate Loss]{definition}{varyingsurrogateloss}
\label{def:varying_01_surrogate}
    For each expert set \(A\), let \(\Phi_{01}^A: \mathcal{H} \times \overline{\mathcal{X}} \times \overline{\mathcal{Y}}_A\) be the standard \(\overline{\mathcal{Y}}_A\)-class classification surrogate. The varying multiclass surrogate loss $\widetilde{\Phi}_{01}: \mathcal{H} \times \overline{\mathcal{X}} \times \overline{\mathcal{Y}}$ is defined as follows
    $\widetilde{\Phi}_{01}(h, \overline{x},j) =
        \begin{cases}
        \Phi_{01}^A(h,\overline x,j) & \text{if } j \in \overline{\mathcal{Y}}_A, \\
        0 & \text{otherwise}.
        \end{cases}$
\end{restatable}

Intuitively, this definition guarantees that the surrogate $\widetilde{\Phi}_{01}$ is unaffected by the scores of experts that are not available. Building on this property, any multiclass surrogate loss $\widetilde{\Phi}_{01}$ can be naturally extended to obtain an online equivalent of Definition~\ref{def_surr_l2d}.

\begin{restatable}[Surrogate Deferral Loss]{definition}{surrogatedeferralloss}
\label{def:general_surrogate_def_loss}
    For any \(h \in \mathcal{H}\) and \((\overline{x}, y) = ((x,A),y) \in \overline{\mathcal{X}} \times \mathcal{Y}\),
    \begin{equation*}
        \begin{aligned}
        \Phi_{\mathrm{def}}(h, \overline{x} , y) & = \widetilde{\Phi}_{01}(h, \overline{x}, y) \\
        & + \sum_{n + j \in \overline{\mathcal{Y}}_A } \mspace{-10mu}(1 - c_j(x, y)) \widetilde{\Phi}_{01}(h, \overline{x}, n + j).
        \end{aligned}   
    \end{equation*} 
\end{restatable}

\subsection{Regret Formulation}

We can now formalize the goal of the learner. For any loss $L$, define
\(\mathcal C_L(h,\overline x)=\mathbb E_{Y\sim p(\cdot\mid\overline x)}[L(h,\overline x,Y)]\),
\(\mathcal R_L(\mathbf h,\overline{\mathbf x})=\sum_{t=1}^T\mathcal C_L(h_t,\overline x_t)\),
and \(\mathcal R_L(h,\overline{\mathbf x})=\sum_{t=1}^T\mathcal C_L(h,\overline x_t)\).
For randomized actions, define \(\mathcal C_{\ell_{\mathrm{def}}}(q,\overline x)=\mathbb E_{Y\sim p(\cdot\mid\overline x)}[\ell_{\mathrm{def}}(q,\overline x,Y)]\) and \(\mathcal R_{\ell_{\mathrm{def}}}(\mathbf q,\overline{\mathbf x})=\sum_{t=1}^T\mathcal C_{\ell_{\mathrm{def}}}(q_t,\overline x_t)\).
Also let \(\mathcal C^*_{L,\mathcal H}(\overline x)=\inf_{h\in\mathcal H}\mathcal C_L(h,\overline x)\) and
\(\mathcal R^*_{L,\mathcal H}(\overline{\mathbf x})=\inf_{h\in\mathcal H}\mathcal R_L(h,\overline{\mathbf x})\).
The objective is the true deferral regret
\begin{equation}\label{regret}
    R_{\ell_{\mathrm{def}}}(T) = \mathcal{R}_{\ell_{\mathrm{def}}}(\mathbf{q}, \overline{\mathbf{x}}) - \mathcal{R}_{\ell_{\mathrm{def}}, \mathcal{H}}^*(\overline{\mathbf{x}}).
\end{equation}
In particular, the goal is to achieve \emph{sublinear} regret scaling in $T$, that is $R_{\ell_{\mathrm{def}}}(T) = o(T)$. 
This ensures that the average regret $R_{\ell_{\mathrm{def}}}(T)/T$ converges to zero as $T \to \infty$, meaning that the learner’s performance asymptotically approaches that of the best fixed strategy in hindsight.

The comparator in this generic formulation is the best fixed deferral strategy
in \(\mathcal H\), rather than a classifier-only strategy. It may use
available experts whenever doing so reduces conditional cost. Comparisons with a
classifier-only benchmark require that the corresponding classifier be embedded
in the same parameter class and norm budget. Section~\ref{section:regret_bound}
instantiates the generic class $\mathcal H$ as the projected class $\mathfrak H_B$.

The regret typically depends on the \((\ell_{\mathrm{def}}, \mathcal{H})\)-minimizability gap \(\mathcal{M}_{\ell_{\mathrm{def}}, \mathcal{H}}(\overline{\mathbf{x}}) = \mathcal{R}_{\ell_{\mathrm{def}}, \mathcal{H}}^*(\overline{\mathbf{x}}) - \sum_{t \in [T]} \mathcal{C}_{\ell_{\mathrm{def}}, \mathcal{H}}^*(\overline{x}_t)\). The minimizability gap measures the discrepancy between selecting the best hypothesis adaptively at each round and selecting the best single hypothesis in hindsight. 


\section[H-Consistency Bound for Online Learning]{\(\mathcal{H}\)\nobreakdash-CONSISTENCY BOUND FOR ONLINE LEARNING}
\label{section:general_theorems}

A central property in batch learning is \emph{consistency}, which ensures that minimizing a surrogate risk recovers the Bayes-optimal solution. This property has been extensively studied in the classical setting \citep{Statistical, Steinwart2007HowTC, bartlett1}. Here we derive the corresponding regret transfer for time-varying active action sets and randomized exploration, drawing on the bounds developed in \citet{bartlett1, Awasthi_Mao_Mohri_Zhong_2022_multi, mao2023crossentropylossfunctionstheoretical, mao2024hconsistencyregression, mao2024multilabel, mao2024universalgrowth, mao2025enhanced, zhong2025thesis, mohri2026beyond, mohri2026linear, mohri2026mind} and extensions to imbalanced and metric-based settings~\citep{cortes2025improvedbalanced, cortes2025balancingscales, mao2025principledbinary}.  

Formally, our goal is to relate true regret, evaluated at randomized predictions \(\mathbf q\), to surrogate regret, evaluated at score functions \(\mathbf h\). Fix \(\overline x=(x,A)\), \(p\in\mathcal D\), and costs \(c_j:\mathcal X\times\mathcal Y\to[0,1]\). Define an induced distribution on \(\overline{\mathcal Y}_A\) by
\[
s(\overline x,a)=
\begin{cases}
p(\overline x,a),&a\in\mathcal Y,\\
1-\sum_{y\in\mathcal Y}p(\overline x,y)c_j(x,y),&a=n+j,
\end{cases}
\]
Let \(S(\overline x)=\sum_{a\in\overline{\mathcal Y}_A}s(\overline x,a)\) and
\(\overline s(\overline x,a)=s(\overline x,a)/S(\overline x)\).
For a multiclass loss \(L\), let \(\mathcal C_L^{\overline s}(h,\overline x)=\sum_{a\in\overline{\mathcal Y}_A}\overline s(\overline x,a)L(h,\overline x,a)\) and \(\Delta\mathcal C_{L,\mathcal H}^{\overline s}(h,\overline x)=\mathcal C_L^{\overline s}(h,\overline x)-\inf_{h'\in\mathcal H}\mathcal C_L^{\overline s}(h',\overline x)\). We assume that there exist a concave, non-decreasing \(\Gamma:\mathbb R_+\to\mathbb R_+\), with \(\Gamma(0)=0\), and \(\epsilon\geq0\) such that every distribution induced above satisfies
\[
\Delta \mathcal{C}^{\overline s}_{\ell_{01},\mathcal{H}}(h,\overline{x})
\leq \Gamma\!\big(\Delta \mathcal{C}^{\overline s}_{\widetilde{\Phi}_{01},\mathcal{H}}(h,\overline{x})\big)+\epsilon
\]
for every \(h\in\mathcal H\). We suppress the superscript \(\overline s\) below.

\begin{restatable}{theorem}{GammaBound} (Online $\Gamma$-Transfer) \label{thm:gamma-bound}
    Assume that \(\widetilde{\Phi}_{01}\) satisfies the pointwise
    $\mathcal H$-consistency bound above. Fix augmented inputs
    \(\overline{\mathbf x}=(\overline x_t)_{t=1}^T\), a conditional label
    kernel \(p\in\mathcal D\), and hypotheses \(h_t\in\mathcal H\).
    For \(\gamma_t\in[0,1]\), let
    \[
    q_t=(1-\gamma_t)e_{h_t(\overline x_t)}
    +\frac{\gamma_t}{|\overline{\mathcal Y}_{A_t}|}\mathbf1,
    \qquad
    \overline S_T=\frac1T\sum_{t=1}^T S(\overline x_t).
    \]
    For \(L\in\{\ell_{\mathrm{def}},\Phi_{\mathrm{def}}\}\), abbreviate
    \[
    E_L(T):=R_L(T)+\mathcal M_{L,\mathcal H}(\overline{\mathbf x}).
    \]
    Then
    \begin{equation*}
        \begin{aligned}
        \frac{E_{\ell_{\mathrm{def}}}(T)}{T}
        &\leq \overline S_T\,\Gamma\!\left(
        \frac{E_{\Phi_{\mathrm{def}}}(T)}{T\overline S_T}
        \right)\\
        &\quad+\overline S_T\epsilon+\frac1T\sum_{t=1}^T\gamma_t.
        \end{aligned}
    \end{equation*}
    Here \(S\) is the score mass defined above. Since \(c_j\in[0,1]\),
    \(1\leq S(\overline x_t)\leq|\overline{\mathcal Y}_{A_t}|\).
\end{restatable}

The proof is deferred to Appendix~\ref{appendix:consistency_bound}. The following specialization applies whenever the surrogate minimizability gap is zero; this includes the pointwise realizable case in which a single hypothesis in $\mathcal H$ attains the optimal conditional surrogate risk on every round.

\begin{corollary}
\label{cor:consistency_bound}
Under the assumptions of Theorem~\ref{thm:gamma-bound}, suppose additionally that $\mathcal M_{\Phi_{\mathrm{def}},\mathcal H}(\overline{\mathbf x})=0$. Then
\begin{equation*}
\begin{aligned}
\frac{R_{\ell_{\mathrm{def}}}(T)}{T}
&\leq \overline S_T\,\Gamma\!\left(
\frac{R_{\Phi_{\mathrm{def}}}(T)}{T\overline S_T}
\right)
+\overline S_T\epsilon+\frac1T\sum_{t=1}^T\gamma_t.
\end{aligned}
\end{equation*}
\end{corollary}
In the limit, this bound recovers an online analogue of the standard Bayes-consistency notion defined in \citet{bartlett1}.

\section{REGRET BOUNDS FOR ONLINE L2D}
\label{section:regret_bound}

We can now derive explicit regret bounds from the H-Consistency Bound defined in the previous section.
Indeed, such bounds can be obtained directly by minimizing the right-hand side in Theorem~\ref{thm:gamma-bound}.

In this online section, \(\mathcal H_{\mathrm{all}}\) denotes all measurable score functions on \(\overline{\mathcal X}\times\overline{\mathcal Y}\).
Since the minimizability gap $\mathcal{M}_{\Phi_{\text{def}}, \mathcal{H}}(\overline{\mathbf{x}})$ is fixed by the comparator class and the environment sequence, minimizing the right-hand side effectively becomes an online optimization problem of minimizing the surrogate deferral regret $R_{\Phi_{\mathrm{def}}}(T)$ with bandit feedback. The resulting bound is especially tight when the best-in-class hypothesis coincides with the Bayes hypothesis, i.e., $\mathcal{R}_{\Phi_{\text{def}}, \mathcal{H}}^{*}(\overline{\mathbf{x}}) = \mathcal{R}_{\Phi_{\text{def}},\mathcal{H}_{\text{all}}}^{*}(\overline{\mathbf{x}})$, as Corollary~\ref{cor:consistency_bound} then implies a true deferral regret bound of $R_{\ell_{\mathrm{def}}}(T) = O\!\left(T\Gamma\!\left(\tfrac{R_{\Phi_{\mathrm{def}}}(T)}{T}\right)\right)$.

With a fixed expert set, the problem contains ordinary contextual-bandit instances. Consequently, under the usual richness assumptions on the comparator class, the classical $\sqrt T$ bandit scale is unavoidable \citep{auer2002nonstochastic}. Generic policy-class algorithms such as EXP4 attain this scale for finite classes but are computationally expensive when the class is large.


For zero-sum score classes that pointwise contain the canonical margin vectors,
the constrained hinge surrogate admits the linear calibration function $\Gamma(u)=u$.
Without the concentrated-score assumption, our efficient online algorithm achieves an expected $T^{2/3}$ regret rate. Under the concentrated-score condition stated below, a calibration-slack
term cancels the high-variance importance-weighting contribution and yields an
expected $\sqrt T$ true-deferral regret rate.

\subsection{Constrained OCO with constrained hinge loss \texorpdfstring{$\Phi_{\text{hinge}}^{\text{cstnd}}$}{Phi}}
\label{subsection:cstnd_hinge_loss}

When the surrogate is the constrained hinge loss $\Phi_{01}^A = \Phi_{\text{hinge}}^{\text{cstnd}}$ \citep{lee2004multicategory}, write $\widetilde\Phi_{\mathrm{hinge}}(u)=\max\{0,1-u\}$. The induced surrogate deferral loss is
\begin{equation*}
    \begin{aligned}
        & \Phi_{\text{def}}(h,\overline x,y) = \sum_{y'\in \overline{\mathcal{Y}}_A \setminus \{y\}}\mspace{-15mu} \widetilde{\Phi}_{\text{hinge}}(-h(\overline x,y'))
         \\
        &+ \sum_{n + j \in \overline{\mathcal{Y}}_A}\mspace{-10mu}(1-c_j(x,y))\mspace{-20mu}\sum_{y' \in \overline{\mathcal{Y}}_A \setminus \{n+j\}}\mspace{-20mu}\widetilde{\Phi}_{\text{hinge}}(-h(\overline x,y')),
    \end{aligned}
\end{equation*}
subject to the constraint $\sum_{y \in \overline{\mathcal{Y}}_A} h(\overline{x},y) = 0$. The resulting learning problem is therefore a constrained online convex optimization problem.

\begin{problem}[Constrained Online Convex Optimization]
\begin{equation*}
    \begin{aligned}
        & \mathrm{minimize } \quad R_{\Phi_{\mathrm{def}}}(T),\\
        & \mathrm{subject}\ \mathrm{to} \quad  \sum_{y \in \overline{\mathcal{Y}}_{A_t}} h_t(\overline{x}_t,y) = 0, \;\forall t \in [T].
    \end{aligned}
\end{equation*}
\end{problem}
Since \(\mathcal{H}\) is the set of linear hypotheses, write
$\widetilde W=(W,b)\in\mathbb R^{N\times(d+1)}$ and
$\widetilde x=(x,1)$. Then
$h(\overline x,y)=\langle\widetilde w_y,\widetilde x\rangle$.
When a matrix appears as a loss or risk argument below, it denotes its induced score function \(h_{\widetilde W}\).
Let $\rho=\sqrt{R^2+1}$, so $\|\widetilde x\|_2\leq\rho$, and define the
Frobenius ball
$\mathcal H_B=\{\widetilde W:\|\widetilde W\|_F\leq B\}$.
The discussion of a sufficient radius is deferred to Appendix~\ref{appendix:bound}.

For each expert set \(A\in\mathcal A\), define the matrix subspace
\(\mathcal K_A=\{\widetilde W\in\mathbb R^{N\times(d+1)}:
\sum_{a\in\overline{\mathcal Y}_A}\widetilde w_a=0\}\).
Write \(\widetilde W|_{\mathcal K_A}:=\Pi_{\mathcal K_A}(\widetilde W)\)
for its Euclidean projection onto this subspace.

Varying active sets impose different zero-sum constraints. Define the projected
Frobenius-ball class
\[
\begin{aligned}
\mathfrak H_B&=\{h_{\widetilde W}^{\mathrm{proj}}:\widetilde W\in\mathcal H_B\},\\
h_{\widetilde W}^{\mathrm{proj}}((x,A),a)
&=\left\langle(\Pi_{\mathcal K_A}\widetilde W)_a,(x,1)\right\rangle.
\end{aligned}
\]
We compare against a fixed matrix after the round-specific projection:
\[
\mathcal R^*_{\Phi_{\mathrm{def}},\mathfrak H_B}(\overline{\mathbf x})
=\inf_{\widetilde W\in\mathcal H_B}\sum_{t=1}^T
\mathcal C_{\Phi_{\mathrm{def}}}(\Pi_{\mathcal K_{A_t}}\widetilde W,\overline x_t).
\]
We distinguish surrogate and true regret:
\(R_{\Phi_{\mathrm{def}},\mathfrak H_B}(T)=
\mathcal R_{\Phi_{\mathrm{def}}}(\mathbf h,\overline{\mathbf x})-
\mathcal R^*_{\Phi_{\mathrm{def}},\mathfrak H_B}(\overline{\mathbf x})\), whereas
\(R_{\ell_{\mathrm{def}},\mathfrak H_B}(T)=
\mathcal R_{\ell_{\mathrm{def}}}(\mathbf q,\overline{\mathbf x})-
\mathcal R^*_{\ell_{\mathrm{def}},\mathfrak H_B}(\overline{\mathbf x})\).
Every calibration gap and minimizability gap in this section also uses
$\mathfrak H_B$. The relevant projection properties follow.

\begin{proposition} (Properties of \(\Pi_{\mathcal{K}_A}\))
    \label{prop:proj_properties}
    For any \(\widetilde W\in\mathbb R^{N\times(d+1)}\), its linear score function \(h_{\widetilde W}\), and any \(A\in\mathcal A\), the following hold:
    \begin{enumerate}
        \item \(\widetilde{W}|_{\mathcal{K}_A} = \widetilde{W} - \mathbf{1}_{\overline{\mathcal{Y}}_A}\mu^{\intercal} = (\widetilde{w}_{i}|_{\mathcal{K}_A})_{i = 1}^{N}\), where \(\mathbf1_{\overline{\mathcal Y}_A}\in\{0,1\}^N\) is the active-action indicator, \(\mu = \frac{1}{|\overline{\mathcal{Y}}_A|} \sum_{i \in \overline{\mathcal{Y}}_A} \widetilde{w}_{i}\), \(\widetilde{w}_{i}|_{\mathcal{K}_A} = \widetilde{w}_i - \mu\) if \(i \in \overline{\mathcal{Y}}_A\) and \(\widetilde{w}_{i}|_{\mathcal{K}_A} = \widetilde{w}_i\) otherwise.
        \item (Maxima preserving) For every \(\overline x=(x,A)\), \(h_{\widetilde W}(\overline x)=h_{\Pi_{\mathcal K_A}\widetilde W}(\overline x)\).
    \end{enumerate}
\end{proposition}

Proposition~\ref{prop:proj_properties} implies that projection onto
\(\mathcal K_A\) preserves all active score differences and hence the
predicted action. Therefore the true loss of a matrix $\widetilde W\in\mathcal H_B$
equals the true loss of its representative in $\mathfrak H_B$ on every round.
The true and surrogate comparators below are both taken over this explicitly
defined projected class.

\begin{algorithm}
\caption{OGD with Projected Gradients}
\label{algo:cstnd_hinge_loss}
\begin{algorithmic}[1]
\STATE \textbf{Input:} Learning rate $\eta_t > 0$, and exploration rate \(\gamma_t\)
\STATE \textbf{Initialize:} $\widetilde{W}_1 = \mathbf{0} \in \mathbb{R}^{N \times (d+1)}$
\FOR{$t = 1$ {\bfseries to} $T$}
    \STATE Obtain $\overline{x}_t = (x_t, A_t)$
    \STATE Project \(\widetilde{W}_t\) to feasible hypothesis set: \(\widetilde{W}_t|_{\mathcal{K}_{A_t}} = \Pi_{\mathcal{K}_{A_t}}\big(\widetilde{W}_t\big)\) 
    \STATE Let $y_t^* = \arg\max_{y\in\overline{\mathcal{Y}}_{A_t}} \langle \widetilde w_{t,y}|_{\mathcal K_{A_t}},\widetilde x_t\rangle$
    \STATE Set \(q_t = (1 - \gamma_t)\mathbf{e}_{y_t^*} + \frac{\gamma_t}{|\overline{\mathcal{Y}}_{A_t}|}\mathbf{1} \)
    \STATE Predict with label $y'_t \sim q_t$
    \STATE Obtain feedback; if $|\overline{\mathcal Y}_{A_t}|=2$, infer $y_t$ from the correctness bit and use the exact surrogate; otherwise form the importance-weighted surrogate $\widehat\Phi_t$
    \STATE Choose \(\widehat G_t\in\partial\widehat\Phi_t(\widetilde W_t|_{\mathcal K_{A_t}})\) and set \(\widehat G_t|_{\mathcal K_{A_t}}=\Pi_{\mathcal K_{A_t}}(\widehat G_t)\)
    \STATE Update \(\widetilde{W}_{t+1} = \Pi_{\mathcal H_B}\!\left(\widetilde{W}_t - \eta_t \widehat G_t|_{\mathcal{K}_{A_t}}\right)\)
\ENDFOR
\end{algorithmic}
\end{algorithm}
With dense parameters, scoring, active-row centering, gradient construction, and projection onto $\mathcal H_B$ each require $O(Nd)$ arithmetic per round, and the algorithm stores $O(Nd)$ parameters. These bounds use $N=n+n_e$; sparse feature representations can reduce the scoring and gradient costs in the usual way.

Let \(\widehat\Phi_t\) be the estimated loss at round \(t\). Under full feedback it equals \(\Phi_{\mathrm{def}}(\cdot,\overline x_t,y_t)\). The same exact loss is available when \(|\overline{\mathcal Y}_{A_t}|=2\), because the correctness bit then identifies \(y_t\). On all other bandit rounds, for \(\widetilde W\in\mathcal H_B\), abbreviate \(c_{t,j}:=c_j(x_t,y_t)\) and \(\widetilde\Phi_{t,a}(\widetilde W):=\widetilde\Phi_{01}(\widetilde W,\overline x_t,a)\), and set
\begin{equation*}
    \begin{aligned}
        \widehat{\Phi}_t(\widetilde W)
        &=v_{t,0}\widetilde\Phi_{t,y'_t}(\widetilde W)\\
        &\quad+\sum_{g_j\in A_t}
        v_{t,j}(1-c_{t,j})\widetilde\Phi_{t,n+j}(\widetilde W).
    \end{aligned}
\end{equation*}
When $\Phi_{01}^A = \Phi_{\text{hinge}}^{\text{cstnd}}$, let $y_t'$ denote the learner’s randomized prediction sampled from $q_t$, and define $v_{t,0} = \frac{\mathbf{1}\{y_t' \in [n]\}\mathbf{1}\{y_t' = y_t\}}{q_t(y_t')}$ and $v_{t,j} = \frac{\mathbf{1}\{y_t' = n+j\}}{q_t(n+j)}$. The learner’s exploration rate $\gamma_t$, introduced in Section~\ref{section:problem_formulation}, is zero in the full-information case. In the bandit case, the theorem without the concentrated-score assumption uses $\gamma_t=\min\{1/2,t^{-1/3}\}$, while the concentrated-score theorem uses the schedule specified in its statement.

Let \(\mathcal F_{t-1}\) be the history before round \(t\). The estimate is convex and conditionally unbiased: \(\mathbb E[\widehat\Phi_t(\widetilde W)\mid\mathcal F_{t-1},\overline x_t,y_t]=\Phi_{\mathrm{def}}(\widetilde W,\overline x_t,y_t)\). Let \(\widetilde W^*\in\mathcal H_B\) minimize the cumulative projected surrogate risk above. For any expert set \(A\), define \(\mu_A^*=|\overline{\mathcal Y}_A|^{-1}\sum_{i\in\overline{\mathcal Y}_A}\widetilde w_i^*\).

\begin{restatable}{lemma}{cstndperroundbound}
    \label{lemma:cstnd_per_round_bound}
    Let \(\widehat G_t\in\partial\widehat\Phi_t(
    \widetilde W_t|_{\mathcal K_{A_t}})\), and let
    \(r_{t,i}:=(\Pi_{\mathcal K_{A_t}}\widehat G_t)_i\). Then
    \begin{equation*}
        \begin{aligned}
        &\widehat{\Phi}_t(\widetilde W_t|_{\mathcal K_{A_t}})
        -\widehat{\Phi}_t(\widetilde W^*|_{\mathcal K_{A_t}})\\
        &\qquad\leq\sum_{i\in\overline{\mathcal Y}_{A_t}}
        \langle r_{t,i},\widetilde w_{t,i}-\widetilde w_i^*\rangle.
        \end{aligned}
    \end{equation*}
\end{restatable}

The proof is deferred to Appendix~\ref{appendix:cstnd_hinge_loss}. The input,
availability, cost, and conditional-distribution sequences are fixed in advance,
although they may vary arbitrarily with $t$.

\begin{restatable}{theorem}{theoremCstndBound} (Surrogate regret bound)
    \label{thm:cstnd_bound}
    Let $B>0$. Running Algorithm~\ref{algo:cstnd_hinge_loss} with
    \(\eta_t=B/(N^{3/2}\rho t^{2/3})\) and
    \(\gamma_t=\min\{1/2,t^{-1/3}\}\), the expected surrogate regret satisfies
    \begin{equation*}
        \begin{aligned}
            \mathbb E\!\left[R_{\Phi_{\mathrm{def}},\mathfrak H_B}(T)\right]
            =O\!\left(BN^{3/2}\rho T^{2/3}\right).
        \end{aligned}
    \end{equation*}
\end{restatable}
The proof applies projected OGD globally on $\mathcal H_B$ and controls the
conditional second moment of the importance-weighted projected gradient. It is
deferred to Appendix~\ref{appendix:theorem_4}.

Note that by Proposition~\ref{prop:proj_properties}, we have 
\begin{equation*}
    \begin{aligned}
    & \sum_{t \in [T]} \big(\ell_{\text{def}}\big(\widetilde{W}_t\big|_{\mathcal{K}_{A_t}}, \overline{x}_t, y_t\big)
    -\ell_{\text{def}}\big(\widetilde{W}^*\big|_{\mathcal{K}_{A_t}}, \overline{x}_t, y_t\big)\big)
    \\
    & \qquad = \sum_{t \in [T]} \big(\ell_{\text{def}}\big(\widetilde{W}_t, \overline{x}_t, y_t\big)
    -\ell_{\text{def}}\big(\widetilde{W}^*, \overline{x}_t, y_t\big)\big)
    \end{aligned}
\end{equation*}
Combining Corollary~\ref{cor:consistency_bound} and Theorem~\ref{thm:cstnd_bound}, we achieve the following true deferral regret bound.

\begin{restatable}{theorem}{theoremTrueRegretBound} (True Deferral Regret Bound)
    Apply Corollary~\ref{cor:consistency_bound} with
    $\mathcal H=\mathfrak H_B$, $\Gamma(u)=u$, $\epsilon=0$, and
    $\mathcal M_{\Phi_{\mathrm{def}},\mathfrak H_B}(\overline{\mathbf x})=0$.
    Use the parameters of
    Theorem~\ref{thm:cstnd_bound}. Then
    \begin{equation*}
        \begin{aligned}
            \mathbb E\!\left[R_{\ell_{\mathrm{def}},\mathfrak H_B}(T)\right]
            =O\!\left((BN^{3/2}\rho+1)T^{2/3}\right).
        \end{aligned}
    \end{equation*}
\end{restatable}
\textit{Proof.} Since $\sum_{t=1}^T\gamma_t=O(T^{2/3})$, the result follows
from Corollary~\ref{cor:consistency_bound} and Theorem~\ref{thm:cstnd_bound}.

\textbf{Concentrated Active-Action Score Assumption.}
We make use of the following definition for any \(\overline{x} = (x,A) \in \overline{\mathcal{X}}\) and \(y \in \overline{\mathcal{Y}}\)
\[
s(\overline{x},y) =
\begin{cases}
p(\overline{x},y) & \mbox{if } y \in \mathcal{Y}, \\[6pt]
1 - \sum_{y' \in \mathcal{Y}} p(\overline{x},y')c_{j}(x,y') & \mbox{if } y=n+j.
\end{cases}
\]
Define $\overline s(\overline x,y)=s(\overline x,y)/S(\overline x)$,
where $S(\overline x)=\sum_{y\in\overline{\mathcal Y}_A}s(\overline x,y)$.
Thus $\overline s(\overline x,\cdot)$ is a distribution on the full active
action set. Let
\[
y_{\max}(\overline x)\in\arg\max_{y\in\overline{\mathcal Y}_A}
\overline s(\overline x,y).
\]

\begin{restatable}{theorem}{theoremLowNoise} \label{thm:low-noise_assum}
(Concentrated-Score Regret) Assume $T\geq16$, $B\geq N$, and apply
Corollary~\ref{cor:consistency_bound} to $\mathfrak H_B$ with
$\Gamma(u)=u$, $\epsilon=0$, and zero surrogate minimizability gap. Suppose
that, for every $t\in[T]$,
\[
\overline s\bigl(\overline x_t,y_{\max}(\overline x_t)\bigr)
\geq1-\frac1{\sqrt T}.
\]
Set $\kappa=BN^{3/2}\rho$, $\gamma_t=\min\{1/2,\kappa/\sqrt t\}$, and
$\eta_t=\gamma_t/(4N^3\rho^2)$. Then Algorithm~\ref{algo:cstnd_hinge_loss}
satisfies
\begin{equation*}
    \begin{aligned}
        \mathbb E\!\left[R_{\ell_{\mathrm{def}},\mathfrak H_B}(T)\right]
        =O\!\left(BN^{3/2}\rho\sqrt T+B^2N^3\rho^2\right).
    \end{aligned}
\end{equation*}
\end{restatable}

The proof combines a calibration-slack inequality with a refined conditional
second-moment bound. Details are deferred to Appendix~\ref{appendix:low-noise_assum}.
This concentrated-score condition is a theoretical low-noise regime and is not
asserted for the experimental streams in Section~\ref{section:experiment}.

\begin{remark}
    Further analysis of the minimizability gap is conducted in Appendix~\ref{appendix:minimizability_gap}.
\end{remark}

\section{EXPERIMENTS}
\label{section:experiment}
Expert systems in deployment rarely operate under fixed conditions. Human and automated experts differ not only in accuracy but also in when they can be consulted, and both factors fluctuate over time due to workload, fatigue, or shifting domains. To examine how an online Learning-to-Defer algorithm adapts under such nonstationarity, we vary two fundamental axes of uncertainty: expertise and availability. This yields three regimes of increasing complexity—ranging from fully static to fully dynamic—allowing us to isolate the effect of each source of variability.

\subsection{Overall Setting}

Our experiments study Algorithm~\ref{algo:cstnd_hinge_loss} under three expert settings: fixed expertise and availability, drifting availability, and jointly drifting availability and expertise. Full definitions are in Appendix~\ref{appendix:exp_settings}. For the synthetic and Reuters4 experiments, each expert uses normalized cost \(c_j(x,y)=(\mathbf{1}\{g_j(x)\neq y\}+0.1)/1.1\); CIFAR10H uses the separately stated \(0.05\) consultation cost. Each learning curve reports the across-run mean and standard deviation over five independent runs.

\begin{figure*}[t]
    \centering
    \captionsetup{font=footnotesize,skip=2pt}
    \captionsetup[subfigure]{font=footnotesize,skip=1pt}
    \begin{subfigure}[t]{0.32\textwidth}
        \centering
        \includegraphics[width=\linewidth,height=3.25cm,keepaspectratio]{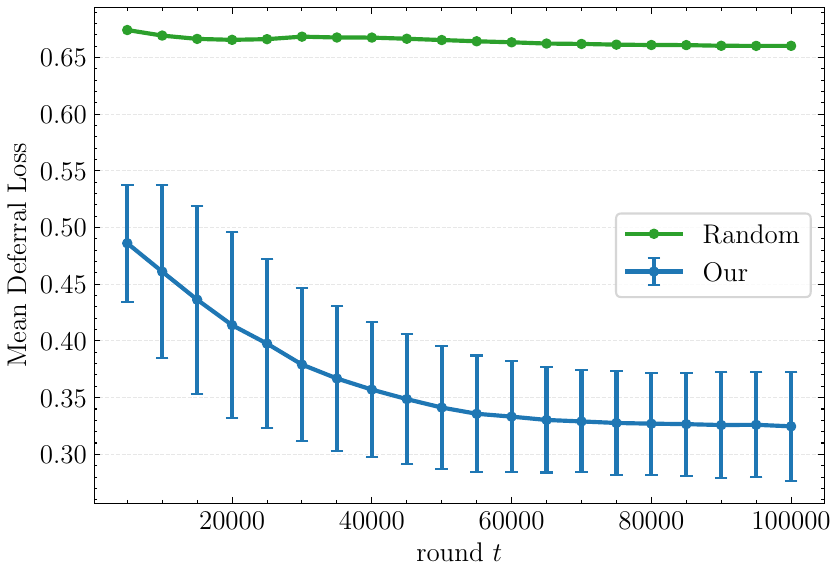}
        \caption{Normalized deferral loss.}
        \label{fig:real_exp3_mean_def_loss_main}
    \end{subfigure}\hfill
    \begin{subfigure}[t]{0.32\textwidth}
        \centering
        \includegraphics[width=\linewidth,height=3.25cm,keepaspectratio]{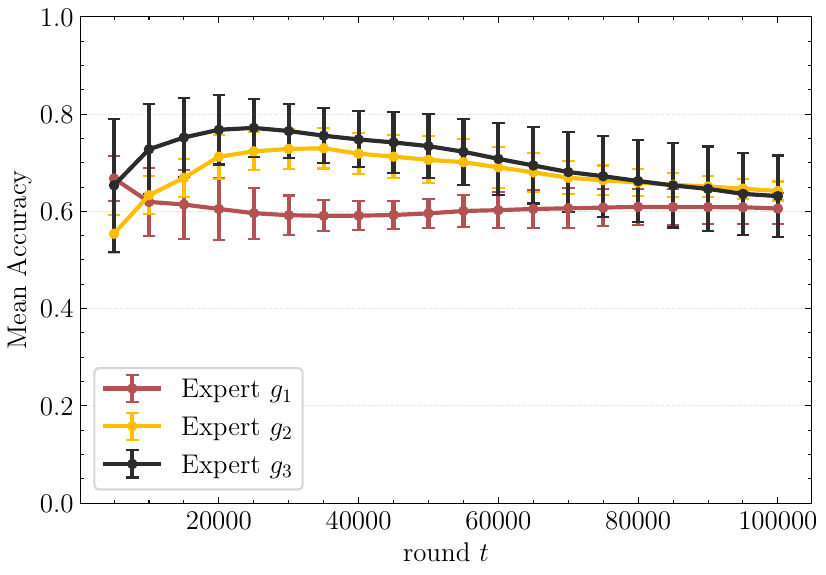}
        \caption{Accuracy of queried experts.}
        \label{fig:real_exp3_expert_acc_queried_main}
    \end{subfigure}\hfill
    \begin{subfigure}[t]{0.32\textwidth}
        \centering
        \includegraphics[width=\linewidth,height=3.25cm,keepaspectratio]{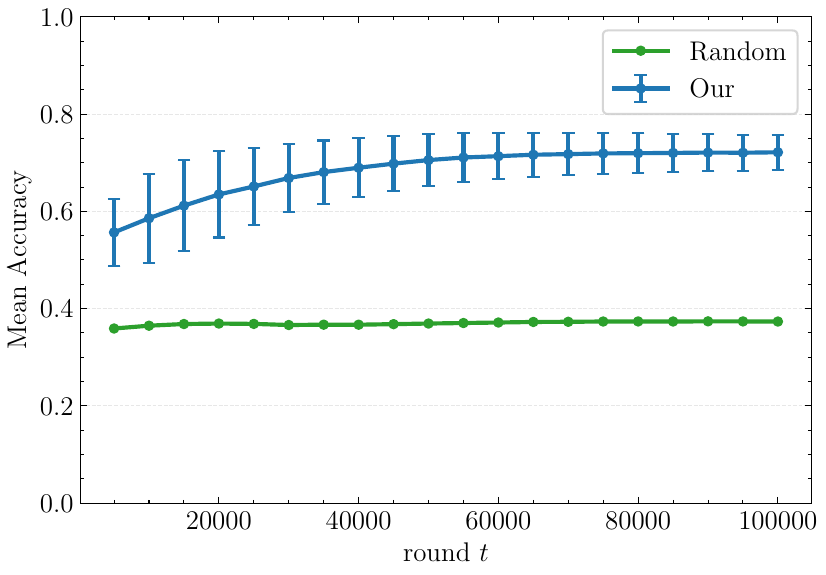}
        \caption{Average accuracy.}
        \label{fig:real_exp3_mean_acc_main}
    \end{subfigure}
    \caption{Reuters4 under drifting availability and expertise. Whiskers denote standard deviations over five runs; regional expertise and availability are in Appendix~\ref{appendix:real_world_dataset}.}
    \label{fig:real_exp3_main}
\end{figure*}

The third setting allows an expert to drift between $100\%$ accuracy and random prediction. We report this most challenging Reuters4 setting in the main paper; additional Reuters4, synthetic, and CIFAR10H experiments appear in Appendix~\ref{appendix:experiments}.

\subsection{Real-world Dataset: Reuters4}

For real-world datasets, we make use of the \textsc{Reuters4} dataset, commonly used for bandit classification experiments \citep{kakade2008efficient, crammer2013multiclass, beygelzimer2017efficient}. This dataset is generated from the \textsc{RCV1} dataset \citep{lewis2004rcv1}, extracting 685,071 examples that have exactly one label from the set \textsc{\{CCAT, ECAT, GCAT, MCAT\}}. It contains 47,236 features and 4 classes, denoted \(\{1,2,3,4\}\). For this experiment, we use the first 100,000 examples.

We generate 3 experts with overlapping knowledge regions. At initialization \(t = 0\), expert \(g_j\) is knowledgeable on labels \(\{j, j+1\}, \; j =1,2,3\). That is, expert \(g_j\) predicts correctly on inputs from classes \(\{j, j+1\}\) while predicting uniformly at random on other classes. These regions remain fixed for the first two settings, and evolve in the final setting.

For the experiments, we replace the scalar OGD step size in Algorithm~\ref{algo:cstnd_hinge_loss} by the AdaGrad update \citep{duchi2011adaptive}, with base learning rate \(0.1\), while retaining the active-set gradient projection and the final projection onto the Frobenius ball. We use exploration rate \(\gamma_t = \min\{\frac{1}{2}, \frac{10}{\sqrt{t}}\}\). The theoretical guarantees analyze the stated OGD schedules; AdaGrad is an empirical variant.

\paragraph{Results under Drifting Availability and Drifting Expertise.} 
Figure~\ref{fig:real_exp3_main} summarizes the loss and accuracy results, while Figure~\ref{fig:real_exp3_def_ratio_main} reports expert-specific deferral frequencies. The changing regional expert accuracies are shown separately in Appendix Figure~\ref{fig:real_exp3_expert_acc_by_region}.

The normalized deferral loss lies in $[0,1]$, placing consultation outcomes on the same scale as the classifier's zero-one loss. Queried-expert accuracy is conditional on rounds assigned to each expert and therefore measures routing quality rather than unconditional expert quality; the corresponding unconditional accuracies and availability trajectories are reported in Appendix~\ref{appendix:real_world_dataset}.

Appendix Figure~\ref{fig:real_exp3_expert_acc_by_region} shows \(g_1\) shifting from classes \(\{1,2\}\) to \(\{3,4\}\), \(g_2\) from \(\{2,3\}\) to \(\{1,4\}\), and \(g_3\) from \(\{3,4\}\) to \(\{1,2\}\). Figures~\ref{fig:real_exp3_main} and~\ref{fig:real_exp3_def_ratio_main} show high queried-expert accuracy, non-uniform routing, stable normalized loss, and increasing task accuracy over the drift interval.

\begin{figure}[H]
    \centering
    \captionsetup{font=footnotesize,skip=2pt}
    \includegraphics[width=1\columnwidth,height=3.1cm,keepaspectratio]{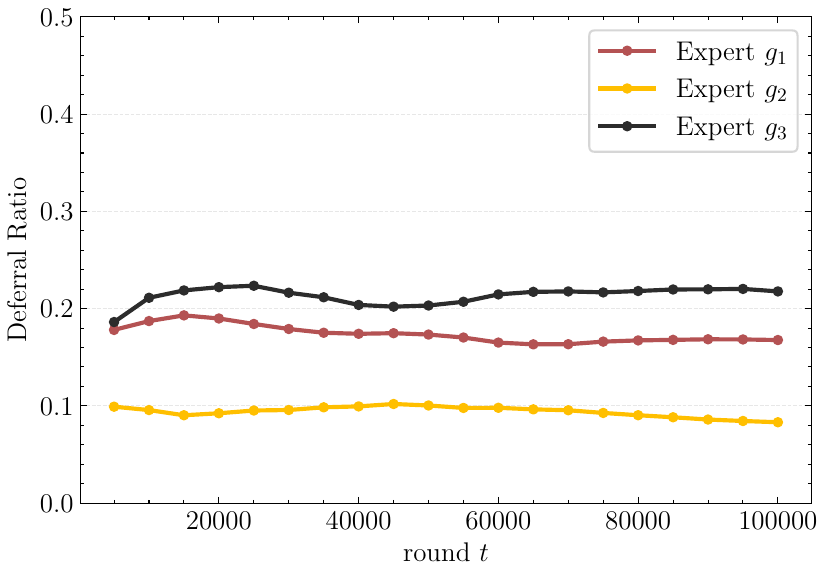}
    \caption{Expert-specific deferral ratios on Reuters4 under drifting availability and expertise.}
    \label{fig:real_exp3_def_ratio_main}
\end{figure}

Remaining Reuters4, synthetic, and CIFAR10H experiments are reported in Appendix~\ref{appendix:experiments}. They include the first two availability/expertise settings, controlled synthetic data with an explicit optimal routing policy, and a WideResNet experiment using human annotations. On CIFAR10H, our learned policy obtains lower normalized deferral loss and higher accuracy than the confidence-based baseline described in Appendix~\ref{appendix:cifar10h_dataset}.

\section{CONCLUSION}
We introduced multiclass \emph{Online Learning-to-Defer} with queried-action bandit feedback and dynamically varying experts. Combining an online $\mathcal H$-consistency transfer with projected online convex optimization, we prove expected true-deferral regret $O((BN^{3/2}\rho+1)T^{2/3})$ under linear calibration and zero surrogate minimizability gap for the projected comparator class, improving to $O(BN^{3/2}\rho\sqrt T+B^2N^3\rho^2)$ under concentrated scores. The oblivious-sequence protocol permits time-varying inputs, conditional outcomes, expert costs, and availability. Experiments show selective routing under drifting availability and expertise.

\clearpage
\section*{Acknowledgments}
Yannis Montreuil and Maxime Meyer are supported by the National Research Foundation, Singapore under its AI Singapore Programme (AISG Award No: AISG2-PhD-2023-01-041-J and AISG Award No: AISG3-PhD-2026-01-068T). Yannis Montreuil is also supported by A*STAR, and is part of the programme DesCartes which is supported by the National Research Foundation, Prime Minister’s Office, Singapore under its Campus for Research Excellence and Technological Enterprise (CREATE) programme. 

\bibliographystyle{plainnat}
\bibliography{biblio1.bib}
\section*{Checklist}

\begin{enumerate}

  \item For all models and algorithms presented, check if you include:
  \begin{enumerate}
    \item A clear description of the mathematical setting, assumptions, algorithm, and/or model. [Yes]
    The mathematical setting and assumptions are described in Sections~\ref{section:preliminaries} and~\ref{section:extending_to_online}. Our main algorithm is clearly described in Algorithm~\ref{algo:cstnd_hinge_loss}.
    \item An analysis of the properties and complexity (time, space, sample size) of any algorithm. [Yes]
    It is provided in Appendix.
    \item (Optional) Anonymized source code, with specification of all dependencies, including external libraries. [No]
  \end{enumerate}

  \item For any theoretical claim, check if you include:
  \begin{enumerate}
    \item Statements of the full set of assumptions of all theoretical results. [Yes]
    \item Complete proofs of all theoretical results. [Yes]
    \item Clear explanations of any assumptions. [Yes]     
  \end{enumerate}

  \item For all figures and tables that present empirical results, check if you include:
  \begin{enumerate}
    \item The code, data, and instructions needed to reproduce the main experimental results (either in the supplemental material or as a URL). [No]
    \item All the training details (e.g., data splits, hyperparameters, how they were chosen). [No]
    \item A clear definition of the specific measure or statistics and error bars (e.g., with respect to the random seed after running experiments multiple times). [Yes]
    \item A description of the computing infrastructure used. (e.g., type of GPUs, internal cluster, or cloud provider). [No]
  \end{enumerate}

  \item If you are using existing assets (e.g., code, data, models) or curating/releasing new assets, check if you include:
  \begin{enumerate}
    \item Citations of the creator If your work uses existing assets. [Yes]
    \item The license information of the assets, if applicable. [No]
    \item New assets either in the supplemental material or as a URL, if applicable. [Not Applicable]
    \item Information about consent from data providers/curators. [No]
    \item Discussion of sensitive content if applicable, e.g., personally identifiable information or offensive content. [Not Applicable]
  \end{enumerate}

  \item If you used crowdsourcing or conducted research with human subjects, check if you include:
  \begin{enumerate}
    \item The full text of instructions given to participants and screenshots. [Not Applicable]
    \item Descriptions of potential participant risks, with links to Institutional Review Board (IRB) approvals if applicable. [Not Applicable]
    \item The estimated hourly wage paid to participants and the total amount spent on participant compensation. [Not Applicable]
  \end{enumerate}

\end{enumerate}

\clearpage
\appendix
\thispagestyle{empty}

\onecolumn

\section*{NOTATION TABLE}

\begin{longtable}{@{}>{$}p{0.23\textwidth}<{$} p{0.69\textwidth}@{}}

\notationgroup{Basic Sets and Indices}

[n] & Set $\{1,\ldots,n\}$. \\
T & Time horizon / total number of online rounds. \\
t & Round index, where $t\in[T]$. \\
n & Number of original classes. \\
n_e & Number of experts. \\
N & Size of the full augmented label space, $N=n+n_e$. \\
d & Input dimension. \\
R & Radius bound for the input space. \\
B & Frobenius-norm bound on the linear hypothesis class. \\[0.4em]

\notationgroup{Inputs, Labels, Experts, and Availability}

\mathcal{X} & Original input space, $\mathcal{X}=\{x\in\mathbb{R}^{d}:\|x\|_2\leq R\}$. \\
x_t & Original input at round $t$. \\
\mathcal{Y} & Original label space, $\mathcal{Y}=[n]$. \\
y_t & True label at round $t$. \\
y'_t & Learner's sampled prediction at round $t$. \\
\mathcal{G} & Pool of experts, $\mathcal{G}=\{g_1,\ldots,g_{n_e}\}$. \\
g_j & The $j$-th expert. \\
A_t & Set of available experts at round $t$. \\
\mathcal{A} & Collection of feasible expert subsets, $\mathcal{A}\subseteq\mathcal{P}(\mathcal{G})$. \\
\overline{\mathcal{X}} & Augmented input space, $\overline{\mathcal{X}}=\mathcal{X}\times\mathcal{A}$. \\
\overline{x}_t & Augmented input at round $t$, $\overline{x}_t=(x_t,A_t)$. \\
\overline{\mathbf{x}} & Sequence of augmented inputs, $(\overline{x}_1,\ldots,\overline{x}_T)$. \\
\overline{\mathcal{Y}} & Full augmented label space, $\overline{\mathcal{Y}}=\mathcal{Y}\cup\{n+j:g_j\in\mathcal{G}\}$. \\
\overline{\mathcal{Y}}_A & Available augmented label set, $[n]\cup\{n+j:g_j\in A\}$; use $A=A_t$ at round $t$. \\
n+j & Deferral action corresponding to expert $g_j$. \\[0.4em]

\notationgroup{Distributions and Randomization}

\Delta_{\mathcal{S}} & Probability simplex over a finite set $\mathcal{S}$. \\
P & Joint data distribution in the batch setting. \\
\mathcal{D} & Set of conditional label distributions indexed by augmented inputs. \\
p(\cdot\mid \overline{x}_t) & Conditional distribution of the true label given $\overline{x}_t$. \\
p(\overline{x},y) & Shorthand for $\mathbb{P}(Y=y\mid \overline{X}=\overline{x})$. \\
q_t & Learner's randomized prediction distribution at round $t$. \\
\mathbf{q} & Sequence of randomized prediction distributions. \\
\gamma_t & Exploration rate at round $t$. \\[0.4em]

\notationgroup{Hypotheses and Linear Predictors}

h & Hypothesis / score function. \\
\mathcal{H} & Hypothesis class. \\
\mathcal{H}_{\mathrm{all}} & All measurable scores on the current section's stated domain. \\
\mathcal H_{\mathrm{def}}^{\mathrm{batch}} & Batch deferral score class on $\mathcal X\times\overline{\mathcal Y}$. \\
h_t & Hypothesis used by the learner at round $t$. \\
\mathbf{h} & Sequence of hypotheses $(h_1,\ldots,h_T)$. \\
h^* & Best-in-class comparator hypothesis in hindsight. \\
h(\overline{x}_t,y) & Score assigned to label $y$ on augmented input $\overline{x}_t$. \\
h(\overline{x}_t) & Prediction induced by $h$: $\arg\max_{y\in\overline{\mathcal{Y}}_{A_t}} h(\overline{x}_t,y)$. \\
W & Weight matrix in $\mathbb R^{N\times d}$, excluding bias. \\
w_y & Weight vector corresponding to label/action $y$. \\
b_y & Bias corresponding to label/action $y$. \\
\widetilde{W} & Augmented weight matrix including bias, $\widetilde{W}=(W,b)$. \\
\widetilde{W}_t & Augmented weight matrix at round $t$. \\
\widetilde{W}^* & Best-in-class augmented weight matrix. \\
\widetilde{w}_y & Row of $\widetilde{W}$ corresponding to label/action $y$. \\
\widetilde{x} & Augmented input $(x,1)$. \\
\rho & Uniform augmented-input bound, $\rho=\sqrt{R^2+1}$. \\
\mathcal H_B & Frobenius ball $\{\widetilde W:\|\widetilde W\|_F\le B\}$. \\
\mathfrak H_B & Roundwise projected linear class induced by $\mathcal H_B$. \\
\|\widetilde W\|_F & Frobenius norm of the augmented matrix. \\[0.4em]

\notationgroup{Expert Costs}

c_j(x,y) & Cost of deferring to expert $g_j$ on example $(x,y)$. \\
c_j(x_t,y_t) & Cost of expert $g_j$ at round $t$. \\
\underline{c}_j,\overline{c}_j & Lower and upper bounds on the cost of expert $g_j$. \\
\alpha_j & Weight on the misclassification cost of expert $g_j$. \\
\beta_j & Query cost of expert $g_j$. \\
Q_j & Expert-specific cost-rescaling factor, $Q_j=\max\{1,\alpha_j+\beta_j\}$. \\[0.4em]

\notationgroup{Losses and Surrogates}

\ell_{01} & Multiclass zero-one classification loss. \\
\ell_{\mathrm{def}}^{\mathrm{batch}} & Batch true deferral loss. \\
\ell_{\mathrm{def}} & Online true deferral loss. \\
\ell_{\mathrm{def}}(h,\overline{x},y) & True deferral loss of hypothesis $h$ on $(\overline{x},y)$. \\
\ell_{\mathrm{def}}(q_t,\overline{x}_t,y_t) & Expected true deferral loss when prediction is sampled from $q_t$. \\
\Phi_{01} & Multiclass surrogate loss for $\ell_{01}$. \\
\Phi_{01}^{\mathrm{batch}} & Augmented multiclass surrogate for batch deferral. \\
\Phi_{01}^{A} & Multiclass surrogate loss defined on the available label set $\overline{\mathcal{Y}}_A$. \\
\widetilde{\Phi}_{01} & Varying multiclass surrogate loss that ignores unavailable experts. \\
\Phi_{\mathrm{def}}^{\mathrm{batch}} & Batch surrogate deferral loss. \\
\Phi_{\mathrm{def}} & Online surrogate deferral loss. \\
\Phi_{\mathrm{hinge}}^{\mathrm{cstnd}} & Constrained hinge surrogate used in the regret analysis. \\
\widetilde\Phi_{\mathrm{hinge}}(u) & Scalar hinge $\max\{0,1-u\}$. \\
\widehat{\Phi}_t & Estimated surrogate loss at round $t$ under bandit feedback. \\[0.4em]

\notationgroup{Risks, Regrets, and Consistency Quantities}

\mathcal{E}_{\ell}(h) & Batch risk of loss $\ell$ under hypothesis $h$. \\
\mathcal{E}_{\ell,\mathcal{H}}^* & Optimal batch risk over $\mathcal{H}$. \\
\mathcal{C}_{\ell_{\mathrm{def}}}(h,\overline{x}) & Conditional true deferral risk of $h$ at $\overline{x}$. \\
\mathcal{C}_{\Phi_{\mathrm{def}}}(h,\overline{x}) & Conditional surrogate deferral risk of $h$ at $\overline{x}$. \\
\mathcal C^*_{L,\mathcal H}(\overline x) & Optimal conditional risk $\inf_{h\in\mathcal H}\mathcal C_L(h,\overline x)$. \\
\Delta\mathcal{C}_{\ell_{01},\mathcal{H}}(h,\overline{x}) & Conditional excess zero-one risk relative to $\mathcal{H}$. \\
\Delta\mathcal C_{\widetilde\Phi_{01},\mathcal H}(h,\overline x) & Conditional excess varying-surrogate risk relative to $\mathcal H$. \\
\Delta\mathcal{C}_{\ell_{\mathrm{def}},\mathcal{H}}(h,\overline{x}) & Conditional excess true deferral risk relative to $\mathcal{H}$. \\
\Delta\mathcal{C}_{\Phi_{\mathrm{def}},\mathcal{H}}(h,\overline{x}) & Conditional excess surrogate deferral risk relative to $\mathcal{H}$. \\
\mathcal{R}_{\ell_{\mathrm{def}}}(\mathbf{h},\overline{\mathbf{x}}) & Cumulative true deferral risk of a hypothesis sequence. \\
\mathcal{R}_{\ell_{\mathrm{def}}}(\mathbf{q},\overline{\mathbf{x}}) & Cumulative true deferral risk of randomized predictions. \\
\mathcal{R}_{\ell_{\mathrm{def}},\mathcal{H}}^*(\overline{\mathbf{x}}) & Optimal cumulative true deferral risk in hindsight over $\mathcal{H}$. \\
\mathcal{R}_{\Phi_{\mathrm{def}}}(\mathbf{h},\overline{\mathbf{x}}) & Cumulative surrogate deferral risk. \\
\mathcal{R}_{\Phi_{\mathrm{def}},\mathcal{H}}^*(\overline{\mathbf{x}}) & Optimal cumulative surrogate deferral risk in hindsight over $\mathcal{H}$. \\
R_{\ell_{\mathrm{def}}}(T) & True deferral regret up to time $T$. \\
R_{\Phi_{\mathrm{def}}}(T) & Surrogate deferral regret up to time $T$. \\
\Gamma & Calibration function. \\
\epsilon & Tolerance term in the $\mathcal{H}$-consistency bound. \\
\mathcal{U}_{\ell,\mathcal{H}} & Batch minimizability gap for loss $\ell$ and class $\mathcal{H}$. \\
\mathcal{M}_{\ell_{\mathrm{def}},\mathcal{H}}(\overline{\mathbf{x}}) & Online minimizability gap for the true deferral loss. \\
\mathcal{M}_{\Phi_{\mathrm{def}},\mathcal{H}}(\overline{\mathbf{x}}) & Online minimizability gap for the surrogate deferral loss. \\[0.4em]

\notationgroup{Constrained OCO and Bandit Feedback}

\mathcal{K}_A & Matrix subspace with zero sum over active rows for expert set $A$. \\
\Pi_{\mathcal{K}_{A_t}} & Projection operator onto $\mathcal{K}_{A_t}$. \\
\widetilde W|_{\mathcal K_A} & Euclidean projection $\Pi_{\mathcal K_A}(\widetilde W)$. \\
\widetilde{W}_t|_{\mathcal{K}_{A_t}} & Projected learner weight matrix at round $t$. \\
\eta_t & Learning rate at round $t$. \\
\widehat G_t & Chosen subgradient of the estimated surrogate loss at round $t$. \\
\widehat G_t^{\mathrm{proj}} & Projection of $\widehat G_t$ onto the active zero-sum subspace. \\
r_{t,i} & Row $i$ of the active-subspace projection of $\widehat G_t$. \\
v_{t,0} & Importance-weighted feedback term for direct classification. \\
v_{t,j} & Importance-weighted feedback term for deferral to expert $g_j$. \\
y_t^* & Greedy action under the current projected hypothesis at round $t$. \\[0.4em]

\notationgroup{Concentrated-Score Analysis}

s(\overline{x},y) & Expected score of action $y$ at augmented input $\overline{x}$. \\
\overline{s}(\overline{x},y) & Normalized score, $\overline{s}(\overline{x},y)=s(\overline{x},y)/S(\overline{x})$. \\
S(\overline{x}) & Total score mass, $S(\overline{x})=\sum_{y\in\overline{\mathcal{Y}}_A}s(\overline{x},y)$. \\
y_{\max} & Maximizing active action in the concentrated-score analysis. \\
\kappa & Scale $BN^{3/2}\rho$ in the concentrated-score schedule. \\
\lambda & Step-size factor $(4N^3\rho^2)^{-1}$ in its proof. \\

\end{longtable}

\section{Discussion of Related Works}
\label{appendix:related_works}

\textbf{Contextual multi-arm bandits}. When the expert set is fixed, our problem reduces to a contextual multi-armed bandit with \(n+n_e\) arms \citep{lu2010contextual}. At each round, the learner observes a context \(\overline x_t\) and selects either a prediction arm \(y \in [n]\) or an expert arm \(j \in [n_e]\). The reward for each arm is \(p(\overline x,y)\) for \(y \in [n]\) and \(1 - \sum_{y' \in \mathcal{Y}}p(\overline x,y')c_j(x,y')\) for \(y=n+j\). This fits the standard contextual multi-armed bandit model in which the context may be adversarial while the arm rewards depend on that context. A related line of work is linear contextual bandits, where arm rewards are assumed to be linear functions of the context with fixed coefficients \citep{li2010, chu2011contextual}. In contrast, we make no such reward-distribution assumption. This places our setting closer to contextual bandits with a policy class, for which efficient algorithms are typically oracle-based \citep{dudik2011efficient, agarwal2014taming}. Our approach instead exploits the structure of a linear hypothesis class to obtain an efficient non-oracle solution. Contextual bandits with a policy class also motivate our benchmarks discussion in Appendix~\ref{appendix:benchmark}.

\textbf{Sleeping experts}. When \(A_t\) varies over time, our setting is related in spirit to multi-armed bandits with varying arms, including
sleeping experts, where only a subset of arms is available at each round \citep{kanade2014learning, kleinberg2010regret}. In our formulation, this corresponds exactly to the fact that only experts in \(A_t\) may be selected at time \(t\). However, the underlying formulation is different. In sleeping bandits, no contextual information is provided, and the rewards are chosen adversarially or drawn from a fixed distribution. In contrast, our costs are generated conditionally on the observed context \(\overline x_t\), which may be chosen adversarially.

This places our formulation closer to contextual multi-armed bandits, where the arm feedback depends on the context. Because of this contextual dependence, existing techniques for sleeping bandits do not directly extend to our setting. Nevertheless, the connection is conceptually useful, and we draw inspiration from this line of work when discussing appropriate regret benchmarks, as noted in Appendix~\ref{appendix:benchmark}.

\textbf{Learning with abstention}. Online learning with abstention \citep{cortes2018online, neu2020fast} extends online prediction by allowing the learner to abstain, but the abstain option is typically treated as sending the input to an external oracle or as incurring a fixed penalty. This does not reflect real human–AI systems, where deferral does not guarantee perfect resolution; in practice, the agent handling the abstained input may even perform worse than the classifier. In contrast, our formulation models deferral to multiple experts whose predictions may be imperfect and incur consultation costs. The learner must therefore decide when deferring is worthwhile, based on both its own performance and the varying quality of the available experts.

\section{Discussion of benchmarks}
\label{appendix:benchmark}

The objective in online learning is to minimize regret, the most general formulation of which is

\[
R(T) = \mathcal{R}_{\ell_{\mathrm{def}}}(\mathbf{q}, \overline{\mathbf{x}}) - \mathcal{R}_{\ell_{\mathrm{def}}}(\mathbf{h}^*, \overline{\mathbf{x}}).
\]

Here \(\mathbf{q} = (q_1, q_2, \dots, q_T)\) is the sequence of learner prediction distributions defined in Section~\ref{section:problem_formulation}, and \(\mathbf{h}^* = (h_1^*, h_2^*, \dots, h_T^*) \in \mathcal{H}^{T}\) is the \textit{best} sequence of hypotheses. How \(\mathbf{h}^*\) is defined determines the difficulty of the problem. We present and discuss some natural definitions below.

\textbf{No Restrictions}. Note that if no restriction is placed on \(h^*_t\), then the regret is always \(R(T) = \Omega(T)\) for any online algorithm. To see this, consider the setting where in each round \(t \in [T]\), the adversary chooses \(y_t\) uniformly at random. The optimal sequence of hypotheses is \(h^*_t\) such that \(h_t^*(\overline{x}_t) = y_t\). Assuming all experts are wrong, the learner incurs an expected regret of at least \((n-1)/n\) each round while the comparator \(h_t^*\) incurs zero loss.

\textbf{Availability-dependent policy comparators}. A stronger benchmark may
select one hypothesis $h_A\in\mathcal H$ for each availability set $A$ and compare
the learner with $h_{A_t}$ on round $t$. This comparator class can contain up to
$2^{n_e}$ availability-indexed components and is therefore substantially richer
than a single fixed hypothesis. Its attainable regret depends on which
availability patterns occur and on the joint capacity of the corresponding
policy class; the number of experts alone does not determine a universal lower
bound.

Our main results use the standard fixed best-in-class comparator. This benchmark
still permits arbitrary time variation in the observed availability sets while
keeping the comparator statistically meaningful and the regret guarantees
independent of the number of distinct availability patterns.

\section{Discussion of bound B on the hypothesis set}
\label{appendix:bound}

In this section, we present a suitable choice for bound B. We make use of the characteristic of the minimizability gap for the constrained hinge loss. From Theorem~\ref{thm:minimizability_gap}, we have shown that the minimal conditional surrogate \(\Phi_{\text{def}}\)-risk can be obtained by a hypothesis of the form

\begin{equation} \label{eq:minimal_hypo}
h(\overline{x}, y) = \begin{cases}
    |\overline{\mathcal{Y}}_A| - 1 & \text{if } y = y_{\max}, \\
    -1 & \text{otherwise}.
\end{cases}
\end{equation}

It suffices to show that the pointwise minimizer in \eqref{eq:minimal_hypo} can be represented inside a common bounded parameter set. Fix $\overline x=(x,A)$, write $\widetilde x=(x,1)$, and let $v\in\mathbb R^{|\overline{\mathcal Y}_A|}$ be the target score vector whose $y_{\max}$ coordinate is $|\overline{\mathcal Y}_A|-1$ and whose other coordinates equal $-1$. Set
\[
\widetilde W|^{\overline{\mathcal Y}_A}
=\frac{v\widetilde x^\intercal}{\|\widetilde x\|_2^2},
\]
and set the unavailable rows to zero. Then
$\widetilde W|^{\overline{\mathcal Y}_A}\widetilde x=v$, the active row
sum is zero, and
\[
\|\widetilde W\|_F=\frac{\|v\|_2}{\|\widetilde x\|_2}
=\frac{\sqrt{|\overline{\mathcal Y}_A|(|\overline{\mathcal Y}_A|-1)}}{\|\widetilde x\|_2}
\leq N,
\]
because $\|\widetilde x\|_2\geq1$. Hence $B=N$ is sufficient for every pointwise conditional minimizer used in the minimizability-gap calculation. The minimizability gap separately records whether one fixed matrix attains these pointwise minima simultaneously across all rounds.

\textbf{Projection onto \(\mathcal{K}_A\) preserves the radius}. For every \(\widetilde W\in\mathcal H_B\), orthogonal projection is nonexpansive and fixes the origin, so \(\|\Pi_{\mathcal K_A}\widetilde W\|_F\leq\|\widetilde W\|_F\leq B\). Hence the projected matrix remains in \(\mathcal H_B\).

\section{Proof of \(\mathcal{H}\)-consistency bound for online learning}
\label{appendix:consistency_bound}

For any \(\overline{x} = (x,A) \in \overline{\mathcal{X}}\), we denote by \(H(\overline{x})\) the set of labels generated by hypotheses in \(\mathcal{H}\): 
\(H(\overline{x}) = \{h(\overline{x}) : h \in \mathcal{H}\} \cap \overline{\mathcal{Y}}_A\). With this, we can write the conditional risk and calibration gap of the deferral loss as follows.

\subsection{Calibration gap for true deferral loss}
\begin{lemma}
\label{lemma:cali_gap_L_def}
For any \(\overline{x} \in \overline{\mathcal{X}}\), the minimal conditional \(\ell_{\mathrm{def}}\)-risk and the calibration gap for \(\ell_{\mathrm{def}}\) can be expressed as follows:
\[
\mathcal{C}^{*}_{\ell_{\mathrm{def}}, \mathcal{H}}(\overline{x}) 
= 1 - \max_{y \in H(\overline{x})} s(\overline{x},y),
\]
\[
\Delta \mathcal{C}_{\ell_{\mathrm{def}}, \mathcal{H}}(h, \overline{x}) 
= \max_{y \in H(\overline{x})} s(\overline{x},y) - s(\overline{x},h(\overline{x})).
\]
\end{lemma}

\begin{proof}
The conditional \(\ell_{\mathrm{def}}\)-risk of \(h\) can be expressed as follows:
\begin{equation*}
\begin{aligned}
\mathcal{C}_{\ell_{\mathrm{def}}}(h,\overline{x})
&= \mathbb{E}_{y \sim p(\cdot\mid\overline{x})}\!\big[\ell_{\mathrm{def}}(h,\overline{x},y)\big] \\[4pt]
&= \mathbb{E}_{y \sim p(\cdot\mid\overline{x})}\!\!\left[\mathbf{1}\{h(\overline{x}) \neq y\}\mathbf{1}\{h(\overline{x}) \in [n]\}
+ \sum_{n+j \in \overline{\mathcal{Y}}_A} c_j(x,y)\mathbf{1}\{h(\overline{x}) = n+j\}\right] \\[4pt]
&= \sum_{y \in \mathcal{Y}} s(\overline{x},y)\mathbf{1}\{h(\overline{x}) \neq y\}\mathbf{1}\{h(\overline{x}) \in [n]\}
+ \sum_{n+j \in \overline{\mathcal{Y}}_A} \bigl(1 - s(\overline{x},n+j)\bigr)\mathbf{1}\{h(\overline{x})=n+j\} \\[4pt]
&= \bigl(1 - s(\overline{x},h(\overline{x}))\bigr)\mathbf{1}\{h(\overline{x}) \in [n]\}
+ \sum_{n+j \in \overline{\mathcal{Y}}_A} \bigl(1 - s(\overline{x},h(\overline{x}))\bigr)\mathbf{1}\{h(\overline{x})=n+j\} \\[4pt]
&= 1 - s(\overline{x},h(\overline{x})).
\end{aligned}
\end{equation*}

Then, the minimal conditional \(\ell_{\mathrm{def}}\)-risk is given by
\begin{equation*}
\mathcal{C}^{*}_{\ell_{\mathrm{def}}}(\mathcal{H},\overline{x})
= 1 - \max_{y \in H(\overline{x})} s(\overline{x},y),
\end{equation*}

and the calibration gap can be expressed as follows:
\begin{equation*}
\Delta \mathcal{C}_{\ell_{\mathrm{def}},\mathcal{H}}(h,\overline{x})
= \mathcal{C}_{\ell_{\mathrm{def}}}(h,\overline{x})
- \mathcal{C}^{*}_{\ell_{\mathrm{def}}}(\mathcal{H},\overline{x})
= \max_{y \in H(\overline{x})} s(\overline{x},y) - s(\overline{x},h(\overline{x})),
\end{equation*}

which completes the proof.
\end{proof}
\subsection{Calibration gap for surrogate deferral risk}
\begin{lemma}
\label{lemma:cali_gap_L}
For any \(\overline{x} \in \overline{\mathcal{X}}\), the conditional surrogate \(\Phi_{\text{def}}\)-risk and calibration gap can be expressed as follows:
\[
\mathcal{C}_{\Phi_{\text{def}}}(h,\overline{x}) = \sum_{y \in \overline{\mathcal{Y}}_A} s(\overline{x},y) \widetilde\Phi_{01}(h,\overline{x},y),
\]
\[
\Delta \mathcal{C}_{\Phi_{\text{def}},\mathcal H}(h,\overline{x}) = \sum_{y \in \overline{\mathcal{Y}}_A} s(\overline{x},y) \widetilde\Phi_{01}(h,\overline{x},y) - \inf_{h' \in \mathcal{H}} \sum_{y \in \overline{\mathcal{Y}}_A} s(\overline{x},y)\widetilde\Phi_{01}(h',\overline{x},y).
\]
\end{lemma}

\begin{proof}
By definition, the conditional \(\Phi_{\text{def}}\)-risk \(\mathcal{C}_{\Phi_{\text{def}}}(h,\overline{x})\) can be expressed as follows:

\begin{equation*}
\begin{split}
\mathcal{C}_{\Phi_{\text{def}}}(h,\overline{x}) 
&= \mathbb{E}_{y \sim p(\cdot\mid\overline{x})}[\Phi_{\text{def}}(h,\overline{x},y)] \\[6pt]
&= \mathbb{E}_{y \sim p(\cdot\mid\overline{x})}[\widetilde\Phi_{01}(h,\overline{x},y)] 
   + \sum_{n + j \in \overline{\mathcal{Y}}_A} \mathbb{E}_{y \sim p(\cdot\mid\overline{x})} \big[(1-c_j(x,y))\big]\widetilde\Phi_{01}(h,\overline{x},n+j) \\[6pt]
&= \sum_{y \in \mathcal{Y}} s(\overline{x},y)\widetilde\Phi_{01}(h,\overline{x},y) 
   + \sum_{n + j \in \overline{\mathcal{Y}}_A} s(\overline{x},n+j)\widetilde\Phi_{01}(h,\overline{x},n+j) \\[6pt]
&= \sum_{y \in \overline{\mathcal{Y}}_A} s(\overline{x},y)\widetilde\Phi_{01}(h,\overline{x},y),
\end{split}
\end{equation*}

which ends the proof.

\end{proof}

We also make use of the following result for the zero-one loss 
\(\ell_{01}(h,\overline{x},y) = \mathbf{1}\{h(\overline{x}) \neq y\}\) 
with label space \(\overline{\mathcal{Y}}_A\) and the conditional probability vector \(\overline{s}(x,\cdot)\), 
which characterizes the minimal conditional \(\ell_{01}\)-risk and the corresponding calibration gap.

\subsection{Calibration gap for zero-one loss}
\begin{lemma}(\citet{Awasthi_Mao_Mohri_Zhong_2022_multi} Lemma 3)
\label{lemma:cali_gap_ell}
For any \(\overline x \in \overline{\mathcal X}\), the minimal conditional \(\ell_{01}\)-risk and the calibration gap for \(\ell_{01}\) can be expressed as follows:
\[
\mathcal{C}^{*}_{\ell_{01},\mathcal H}(\overline x) = 1 - \max_{y \in H(\overline x)} \overline{s}(\overline x,y),
\]
\[
\Delta \mathcal{C}_{\ell_{01},\mathcal H}(h,\overline x) = \max_{y \in H(\overline x)} \overline{s}(\overline x,y) - \overline{s}(\overline x,h(\overline x)).
\]
\end{lemma}

Now we show that the surrogate upper bound on the calibration gap of the zero-one loss \(\ell_{01}\) implies a surrogate upper bound on the calibration gap of the deferral loss \(\ell_{\mathrm{def}}\).

\begin{lemma}
    \label{lemma:cali_gap_def}
    Suppose that the varying multiclass surrogate $\widetilde\Phi_{01}$ satisfies
    \[
    \Delta\mathcal C_{\ell_{01},\mathcal H}(h,\overline x)
    \leq \Gamma\!\left(\Delta\mathcal C_{\widetilde\Phi_{01},\mathcal H}(h,\overline x)\right)+\epsilon
    \]
    for every $h\in\mathcal H$ and $\overline x\in\overline{\mathcal X}$. Let
    $q=(1-\gamma)e_{h(\overline x)}+\gamma\mathbf1/|\overline{\mathcal Y}_A|$.
    Then
    \[
    \Delta \mathcal{C}_{\ell_{\mathrm{def}}, \mathcal{H}}(q, \overline{x})
    \leq S(\overline{x}) \Gamma\!\left(
    \frac{\Delta \mathcal{C}_{\Phi_{\mathrm{def}}, \mathcal{H}}(h, \overline{x})}{S(\overline{x})}
    \right)+S(\overline x)\epsilon+\gamma.
    \]
\end{lemma}

\begin{proof}
By Lemmas~\ref{lemma:cali_gap_L_def}--\ref{lemma:cali_gap_ell} and
$\overline s(\overline x,a)=s(\overline x,a)/S(\overline x)$,
\[
\Delta\mathcal C_{\ell_{\mathrm{def}},\mathcal H}(h,\overline x)
=S(\overline x)\Delta\mathcal C_{\ell_{01},\mathcal H}(h,\overline x),
\qquad
\Delta\mathcal C_{\widetilde\Phi_{01},\mathcal H}(h,\overline x)
=\frac{\Delta\mathcal C_{\Phi_{\mathrm{def}},\mathcal H}(h,\overline x)}{S(\overline x)}.
\]
Every action has conditional loss in $[0,1]$, so
\[
\Delta\mathcal C_{\ell_{\mathrm{def}},\mathcal H}(q,\overline x)
\leq\Delta\mathcal C_{\ell_{\mathrm{def}},\mathcal H}(h,\overline x)+\gamma.
\]
Substituting the assumed multiclass bound proves the result.
\end{proof}

\subsection{Proof of \(\mathcal{H}\)-consistency bounds for surrogate deferral losses (Theorem~\ref{thm:gamma-bound})}

\GammaBound*

\begin{proof}
Write $S_t=S(\overline x_t)$ and
$d_t=\Delta\mathcal C_{\Phi_{\mathrm{def}},\mathcal H}(h_t,\overline x_t)$.
By Lemma~\ref{lemma:cali_gap_def},
\[
\sum_{t=1}^T\Delta\mathcal C_{\ell_{\mathrm{def}},\mathcal H}(q_t,\overline x_t)
\leq\sum_{t=1}^TS_t\Gamma(d_t/S_t)
+\epsilon\sum_{t=1}^TS_t+\sum_{t=1}^T\gamma_t.
\]
Weighted Jensen's inequality gives
\[
\sum_{t=1}^TS_t\Gamma(d_t/S_t)
\leq\left(\sum_{t=1}^TS_t\right)
\Gamma\!\left(\frac{\sum_{t=1}^Td_t}{\sum_{t=1}^TS_t}\right).
\]
Moreover,
\[
\sum_{t=1}^T\Delta\mathcal C_{\ell_{\mathrm{def}},\mathcal H}(q_t,\overline x_t)
=R_{\ell_{\mathrm{def}}}(T)+\mathcal M_{\ell_{\mathrm{def}},\mathcal H}(\overline{\mathbf x}),
\]
and
\[
\sum_{t=1}^Td_t
=R_{\Phi_{\mathrm{def}}}(T)+\mathcal M_{\Phi_{\mathrm{def}},\mathcal H}(\overline{\mathbf x}).
\]
Dividing by $T$ completes the proof.

\end{proof}

\section{Examples of \(\mathcal{H}\)-consistency bound for common surrogate losses}
\label{appendix:examples}

\subsection{\(\Phi_{01}\) being adopted as comp\textnormal{-}sum losses}

\textbf{Example:} \(\Phi_{01} = \Phi_{\exp}\).
Plug in \(\Phi_{01}=\Phi_{\exp}=\sum_{y'\neq y} e^{h(\overline{x},y')-h(\overline{x},y)}\) in Definition~\ref{def:general_surrogate_def_loss}, we obtain
\[
\Phi_{\text{def}} = \sum_{y'\neq y} e^{h(\overline{x},y')-h(\overline{x},y)}
+ \sum_{n+j \in \overline{\mathcal{Y}}_A}\!\bigl(1-c_j(x,y)\bigr)\!\!\sum_{y'\neq n+j} e^{h(\overline{x},y')-h(\overline{x},n+j)}.
\]

By \cite{mao2023crossentropylossfunctionstheoretical}, Theorem 1, \(\Phi_{\exp}\) admits an \(\mathcal{H}\)-consistency bound with respect to \(\ell_{01}\) with 
\(\Gamma(t)=\sqrt{2t}\); using Corollary~\ref{cor:consistency_bound}, we obtain
\[
\frac{R_{\ell_{\mathrm{def}}}(T)}{T}-\frac1T\sum_{t=1}^T\gamma_t
\le \sqrt{2}\,\overline S_T
\left(\frac{R_{\Phi_{\mathrm{def}}}(T)}{T\overline S_T}\right)^{1/2}.
\]

\textbf{Example:} \(\Phi_{01}=\Phi_{\log}\).
Plug in \(\Phi_{01}=\Phi_{\log}=-\log\!\Bigl(\tfrac{e^{h(\overline{x},y)}}{\sum_{y'\in\overline{\mathcal{Y}}_A} e^{h(\overline{x},y')}}\Bigr)\) in Definition~\ref{def:general_surrogate_def_loss}, we obtain
\[
\Phi_{\text{def}} = -\log\!\Bigl(\frac{e^{h(\overline{x},y)}}{\sum_{y'\in\overline{\mathcal{Y}}_A} e^{h(\overline{x},y')}}\Bigr)
- \sum_{n+j \in \overline{\mathcal{Y}}_A}\!\bigl(1-c_j(x,y)\bigr)\,
\log\!\Bigl(\frac{e^{h(\overline{x},n+j)}}{\sum_{y'\in\overline{\mathcal{Y}}_A} e^{h(\overline{x},y')}}\Bigr).
\]

By \cite{mao2023crossentropylossfunctionstheoretical}, Theorem 1, \(\Phi_{\log}\) admits an \(\mathcal{H}\)-consistency bound with respect to \(\ell_{01}\) with 
\(\Gamma(t)=\sqrt{2t}\); using Corollary~\ref{cor:consistency_bound}, we obtain
\[
\frac{R_{\ell_{\mathrm{def}}}(T)}{T}-\frac1T\sum_{t=1}^T\gamma_t
\le \sqrt{2}\,\overline S_T
\left(\frac{R_{\Phi_{\mathrm{def}}}(T)}{T\overline S_T}\right)^{1/2}.
\]

\textbf{Example:} \(\Phi_{01}=\Phi_{\mathrm{gce}}\).
Plug in \(\Phi_{01}=\Phi_{\mathrm{gce}}=\frac{1}{\alpha}\bigl[1-(\tfrac{e^{h(\overline{x},y)}}{\sum_{y'\in\overline{\mathcal{Y}}_A} e^{h(\overline{x},y')}})^{\alpha}\bigr]\) in Definition~\ref{def:general_surrogate_def_loss}, we obtain
\[
\Phi_{\text{def}} = \frac{1}{\alpha}\Bigl[1-\Bigl(\frac{e^{h(\overline{x},y)}}{\sum_{y'\in\overline{\mathcal{Y}}_A} e^{h(\overline{x},y')}}\Bigr)^{\alpha}\Bigr]
+ \frac{1}{\alpha}\sum_{n+j \in \overline{\mathcal{Y}}_A}(1-c_j(x,y))
\Bigl[1-\Bigl(\frac{e^{h(\overline{x},n+j)}}{\sum_{y'\in\overline{\mathcal{Y}}_A} e^{h(\overline{x},y')}}\Bigr)^{\alpha}\Bigr].
\]

By \cite{mao2023crossentropylossfunctionstheoretical}, Theorem 1, \(\Phi_{\mathrm{gce}}\) admits an \(\mathcal{H}\)-consistency bound with respect to \(\ell_{01}\) with 
\(\Gamma(t)=\sqrt{2N^{\alpha}t}\); using Corollary~\ref{cor:consistency_bound}, we obtain
\[
\frac{R_{\ell_{\mathrm{def}}}(T)}{T}-\frac1T\sum_{t=1}^T\gamma_t
\le \sqrt{2N^{\alpha}}\,\overline S_T
\left(\frac{R_{\Phi_{\mathrm{def}}}(T)}{T\overline S_T}\right)^{1/2}.
\]

\textbf{Example:} \(\Phi_{01}=\Phi_{\mathrm{mae}}\).
Plug in \(\Phi_{01}=\Phi_{\mathrm{mae}}=1-\tfrac{e^{h(\overline{x},y)}}{\sum_{y'\in\overline{\mathcal{Y}}_A} e^{h(\overline{x},y')}}\) in Definition~\ref{def:general_surrogate_def_loss}, we obtain
\[
\Phi_{\text{def}} = 1-\frac{e^{h(\overline{x},y)}}{\sum_{y'\in\overline{\mathcal{Y}}_A} e^{h(\overline{x},y')}}
+ \sum_{n+j \in \overline{\mathcal{Y}}_A}(1-c_j(x,y))
\left(1-\frac{e^{h(\overline{x},n+j)}}{\sum_{y'\in\overline{\mathcal{Y}}_A} e^{h(\overline{x},y')}}\right).
\]

By \cite{mao2023crossentropylossfunctionstheoretical}, Theorem 1, \(\Phi_{\mathrm{mae}}\) admits an \(\mathcal{H}\)-consistency bound with respect to \(\ell_{01}\) with 
\(\Gamma(t)=Nt\); using Corollary~\ref{cor:consistency_bound}, we obtain
\[
\frac{R_{\ell_{\mathrm{def}}}(T)}{T}-\frac1T\sum_{t=1}^T\gamma_t
\le N\frac{R_{\Phi_{\mathrm{def}}}(T)}{T}.
\]

\subsection{\(\Phi_{01}\) being adopted as sum losses}

\textbf{Example:} \(\Phi_{01}=\Phi^{\mathrm{sum}}_{\mathrm{sq}}\).
Plug in \(\Phi_{01}=\Phi^{\mathrm{sum}}_{\mathrm{sq}}=\sum_{y'\neq y}\Phi_{\mathrm{sq}}\!\bigl(h(\overline{x},y)-h(\overline{x},y')\bigr)\) in Definition~\ref{def:general_surrogate_def_loss}, we obtain
\[
\Phi_{\text{def}} = \sum_{y'\neq y}\Phi_{\mathrm{sq}}\!\bigl(h(\overline{x},y)-h(\overline{x},y')\bigr)
+ \sum_{n+j \in \overline{\mathcal{Y}}_A}(1-c_j(x,y))\sum_{y'\neq n+j}\Phi_{\mathrm{sq}}\!\bigl(h(\overline{x},n+j)-h(\overline{x},y')\bigr),
\]
where \(\Phi_{\mathrm{sq}}(t)=\max\{0,1-t\}^{2}\).

By \cite{Awasthi_Mao_Mohri_Zhong_2022_multi}, Table 2, \(\Phi^{\mathrm{sum}}_{\mathrm{sq}}\) admits an \(\mathcal{H}\)-consistency bound with respect to \(\ell_{01}\) with \(\Gamma(t)=\sqrt{t}\); using Corollary~\ref{cor:consistency_bound}, we obtain
\[
\frac{R_{\ell_{\mathrm{def}}}(T)}{T}-\frac1T\sum_{t=1}^T\gamma_t
\le \overline S_T
\left(\frac{R_{\Phi_{\mathrm{def}}}(T)}{T\overline S_T}\right)^{1/2}.
\]

\textbf{Example:} \(\Phi_{01}=\Phi^{\mathrm{sum}}_{\exp}\).
Plug in \(\Phi_{01}=\Phi^{\mathrm{sum}}_{\exp}=\sum_{y'\neq y}\Phi_{\exp}\!\bigl(h(\overline{x},y)-h(\overline{x},y')\bigr)\) in Definition~\ref{def:general_surrogate_def_loss}, we obtain
\[
\Phi_{\text{def}} = \sum_{y'\neq y}\Phi_{\exp}\!\bigl(h(\overline{x},y)-h(\overline{x},y')\bigr)
+ \sum_{n+j \in \overline{\mathcal{Y}}_A}(1-c_j(x,y))\sum_{y'\neq n+j}\Phi_{\exp}\!\bigl(h(\overline{x},n+j)-h(\overline{x},y')\bigr),
\]
where \(\Phi_{\exp}(t)=e^{-t}\).

By \cite{Awasthi_Mao_Mohri_Zhong_2022_multi}, Table 2, \(\Phi^{\mathrm{sum}}_{\exp}\) admits an \(\mathcal{H}\)-consistency bound with respect to \(\ell_{01}\) with \(\Gamma(t)=\sqrt{2t}\); using Corollary~\ref{cor:consistency_bound}, we obtain
\[
\frac{R_{\ell_{\mathrm{def}}}(T)}{T}-\frac1T\sum_{t=1}^T\gamma_t
\le \sqrt{2}\,\overline S_T
\left(\frac{R_{\Phi_{\mathrm{def}}}(T)}{T\overline S_T}\right)^{1/2}.
\]

\textbf{Example:} \(\Phi_{01}=\Phi^{\mathrm{sum}}_{\tau}\).
Plug in \(\Phi_{01}=\Phi^{\mathrm{sum}}_{\tau}=\sum_{y'\neq y}\Phi_{\tau}\!\bigl(h(\overline{x},y)-h(\overline{x},y')\bigr)\) in Definition~\ref{def:general_surrogate_def_loss}, we obtain
\[
\Phi_{\text{def}} = \sum_{y'\neq y}\Phi_{\tau}\!\bigl(h(\overline{x},y)-h(\overline{x},y')\bigr)
+ \sum_{n+j \in \overline{\mathcal{Y}}_A}(1-c_j(x,y))\sum_{y'\neq n+j}\Phi_{\tau}\!\bigl(h(\overline{x},n+j)-h(\overline{x},y')\bigr),
\]
where \(\Phi_{\tau}(t)=\min\{\max\{0,1-t/\tau\},1\}\).

By \cite{Awasthi_Mao_Mohri_Zhong_2022_multi}, Table 2, \(\Phi^{\mathrm{sum}}_{\tau}\) admits an \(\mathcal{H}\)-consistency bound with respect to \(\ell_{01}\) with \(\Gamma(t)=t\); using Corollary~\ref{cor:consistency_bound}, we obtain
\[
\frac{R_{\ell_{\mathrm{def}}}(T)}{T}-\frac1T\sum_{t=1}^T\gamma_t
\le \frac{R_{\Phi_{\mathrm{def}}}(T)}{T}.
\]

\subsection{\(\Phi_{01}\) being adopted as constrained losses}

\textbf{Example:} \(\Phi_{01}=\Phi^{\mathrm{cstnd}}_{\mathrm{hinge}}\).
Plug in \(\Phi_{01}=\Phi^{\mathrm{cstnd}}_{\mathrm{hinge}}=\sum_{y'\neq y}\Phi_{\mathrm{hinge}}\!\bigl(-h(\overline{x},y')\bigr)\) in Definition~\ref{def:general_surrogate_def_loss}, we obtain
\[
\Phi_{\text{def}} = \sum_{y'\neq y}\Phi_{\mathrm{hinge}}\!\bigl(-h(\overline{x},y')\bigr)
+ \sum_{n+j \in \overline{\mathcal{Y}}_A}(1-c_j(x,y))\sum_{y'\neq n+j}\Phi_{\mathrm{hinge}}\!\bigl(-h(\overline{x},y')\bigr),
\]
where \(\Phi_{\mathrm{hinge}}(t)=\max\{0,1-t\}\) with the constraint \(\sum_{y\in\overline{\mathcal{Y}}_A} h(\overline{x},y)=0\).

For zero-sum classes satisfying the pointwise richness conditions of
\cite{Awasthi_Mao_Mohri_Zhong_2022_multi}, Table~3,
\(\Phi^{\mathrm{cstnd}}_{\mathrm{hinge}}\) admits an
\(\mathcal H\)-consistency bound with respect to \(\ell_{01}\) with
\(\Gamma(t)=t\); using Corollary~\ref{cor:consistency_bound}, we obtain
\[
\frac{R_{\ell_{\mathrm{def}}}(T)}{T}-\frac1T\sum_{t=1}^T\gamma_t
\le \frac{R_{\Phi_{\mathrm{def}}}(T)}{T}.
\]

\textbf{Example:} \(\Phi_{01}=\Phi^{\mathrm{cstnd}}_{\mathrm{sq}}\).
Plug in \(\Phi_{01}=\Phi^{\mathrm{cstnd}}_{\mathrm{sq}}=\sum_{y'\neq y}\Phi_{\mathrm{sq}}\!\bigl(-h(\overline{x},y')\bigr)\) in Definition~\ref{def:general_surrogate_def_loss}, we obtain
\[
\Phi_{\text{def}} = \sum_{y'\neq y}\Phi_{\mathrm{sq}}\!\bigl(-h(\overline{x},y')\bigr)
+ \sum_{n+j \in \overline{\mathcal{Y}}_A}(1-c_j(x,y))\sum_{y'\neq n+j}\Phi_{\mathrm{sq}}\!\bigl(-h(\overline{x},y')\bigr).
\]

By \cite{Awasthi_Mao_Mohri_Zhong_2022_multi}, Table 3, \(\Phi^{\mathrm{cstnd}}_{\mathrm{sq}}\) admits an \(\mathcal{H}\)-consistency bound with respect to \(\ell_{01}\) with \(\Gamma(t)=\sqrt{t}\); using Corollary~\ref{cor:consistency_bound}, we obtain
\[
\frac{R_{\ell_{\mathrm{def}}}(T)}{T}-\frac1T\sum_{t=1}^T\gamma_t
\le \overline S_T
\left(\frac{R_{\Phi_{\mathrm{def}}}(T)}{T\overline S_T}\right)^{1/2}.
\]

\textbf{Example:} \(\Phi_{01}=\Phi^{\mathrm{cstnd}}_{\tau}\).
Plug in \(\Phi_{01}=\Phi^{\mathrm{cstnd}}_{\tau}=\sum_{y'\neq y}\Phi_{\tau}\!\bigl(-h(\overline{x},y')\bigr)\) in Definition~\ref{def:general_surrogate_def_loss}, we obtain
\[
\Phi_{\text{def}} = \sum_{y'\neq y}\Phi_{\tau}\!\bigl(-h(\overline{x},y')\bigr)
+ \sum_{n+j \in \overline{\mathcal{Y}}_A}(1-c_j(x,y))\sum_{y'\neq n+j}\Phi_{\tau}\!\bigl(-h(\overline{x},y')\bigr),
\]
where \(\Phi_{\tau}(t)=\min\{\max\{0,1-t/\tau\},1\}\) with the constraint \(\sum_{y\in\overline{\mathcal{Y}}_A} h(\overline{x},y)=0\).

By \cite{Awasthi_Mao_Mohri_Zhong_2022_multi}, Table 3, \(\Phi^{\mathrm{cstnd}}_{\tau}\) admits an \(\mathcal{H}\)-consistency bound with respect to \(\ell_{01}\) with \(\Gamma(t)=t\); using Corollary~\ref{cor:consistency_bound}, we obtain
\[
\frac{R_{\ell_{\mathrm{def}}}(T)}{T}-\frac1T\sum_{t=1}^T\gamma_t
\le \frac{R_{\Phi_{\mathrm{def}}}(T)}{T}.
\]

\section{Proof of constrained OCO with \texorpdfstring{$\Phi_{\mathrm{hinge}}^{\mathrm{cstnd}}$}{Phi}}
\label{appendix:cstnd_hinge_loss}

\subsection{Proof of constrained OCO with \texorpdfstring{$\Phi_{\mathrm{hinge}}^{\mathrm{cstnd}}$ (Theorem~\ref{thm:cstnd_bound})}{Phi}}
\label{appendix:theorem_4}

We first give a proof for Lemma~\ref{lemma:cstnd_per_round_bound}

\cstndperroundbound*

\begin{proof}
    By convexity we have

    \begin{equation}
    \begin{split}
    \widehat{\Phi}_t \big(\widetilde{W}_t|_{\mathcal{K}_{A_t}}\big) 
    - \widehat{\Phi}_t\big(\widetilde{W}^*|_{\mathcal{K}_{A_t}}\big) 
    & \leq \langle\widehat G_{t},\widetilde{W}_t|_{\mathcal{K}_{A_t}} - \widetilde{W}^*|_{\mathcal{K}_{A_t}} \rangle_{F} \\
    & = \sum_{i \in \overline{\mathcal{Y}}_{A_t}} 
    \langle \widehat G_{t,i},(\widetilde{w}_{t,i} - \mu_t) - (\widetilde{w}_i^*-\mu_{A_t}^*)\rangle \\
    & = \sum_{i \in \overline{\mathcal{Y}}_{A_t}} 
    \langle \widehat G_{t,i},\widetilde{w}_{t,i} - \widetilde{w}_i^*\rangle + \sum_{i \in \overline{\mathcal{Y}}_{A_t}} 
    \langle \widehat G_{t,i}, \mu_{A_t}^* - \mu_t \rangle \\
    & = \sum_{i \in \overline{\mathcal{Y}}_{A_t}} 
    \langle \widehat G_{t,i},\widetilde{w}_{t,i} - \widetilde{w}_i^*\rangle + \sum_{i \in \overline{\mathcal{Y}}_{A_t}} 
    \langle \frac{\sum_{i \in \overline{\mathcal{Y}}_{A_t}}\widehat G_{t,i}}{|\overline{\mathcal{Y}}_{A_t}|}, \widetilde{w}_i^* - \widetilde{w}_{t,i}\rangle \\
    & = \sum_{i \in \overline{\mathcal{Y}}_{A_t}} 
    \langle \widehat G_{t,i} - \frac{\sum_{i \in \overline{\mathcal{Y}}_{A_t}}\widehat G_{t,i}}{|\overline{\mathcal{Y}}_{A_t}|},\widetilde{w}_{t,i} - \widetilde{w}_i^*\rangle \\
    & = \sum_{i \in \overline{\mathcal{Y}}_{A_t}} \langle r_{t,i},\widetilde{w}_{t,i} - \widetilde{w}_i^*\rangle,
    \end{split}
    \end{equation}

where \(\langle \cdot,\cdot \rangle_F\) denotes the Frobenius inner product,
\(\mu_t = \frac{1}{|\overline{\mathcal{Y}}_{A_t}|} \sum_{i \in \overline{\mathcal{Y}}_{A_t}} \widetilde{w}_{t,i}\), and
\(\mu_{A_t}^{*} = \frac{1}{|\overline{\mathcal{Y}}_{A_t}|} \sum_{i \in \overline{\mathcal{Y}}_{A_t}} \widetilde{w}_{i}^{*}\).
The equalities use that inactive rows of the gradient vanish and the explicit
orthogonal-projection formula in Proposition~\ref{prop:proj_properties}.
\end{proof}

Next, we show that the projected subgradient norm can be bounded.

\begin{lemma} \label{lemma:proj_grad_bound}
For any $\widetilde W\in\mathcal H_B$, active set $A$, target
$y\in\overline{\mathcal Y}_A$, and $G\in\partial\Phi_{01}(\widetilde W,\overline x,y)$,
\[
\left\|\Pi_{\mathcal K_A}G\right\|_F^2
\leq\frac{|\overline{\mathcal Y}_A|\rho^2}{4}\leq\frac{N\rho^2}{4}.
\]
\end{lemma}

\begin{proof}
Let $K=|\overline{\mathcal Y}_A|$, $\widetilde x=(x,1)$, and let
$\alpha_i\in[0,1]$ denote the chosen hinge subgradient coefficients; set
$\alpha_y=0$ and $r=\sum_i\alpha_i$. The unprojected rows of $G$ are
$\alpha_i\widetilde x$. Orthogonal projection onto
$\mathcal K_A$ subtracts their active-row mean, so
\[
\left(\Pi_{\mathcal K_A}G\right)_i
=\left(\alpha_i-\frac rK\right)\widetilde x
\]
on active rows and is zero otherwise. Consequently,
\[
\left\|\Pi_{\mathcal K_A}G\right\|_F^2
=\left(\sum_i\alpha_i^2-\frac{r^2}{K}\right)\|\widetilde x\|_2^2
\leq\frac K4\rho^2.
\]
\end{proof}

\theoremCstndBound*

\begin{proof}
For round $t$, define the convex function
$\widehat f_t(\widetilde W)=\widehat\Phi_t(\Pi_{\mathcal K_{A_t}}\widetilde W)$ and let
$\widehat G_t^{\mathrm{proj}}=\Pi_{\mathcal K_{A_t}}\widehat G_t$.
Because $\Pi_{\mathcal K_{A_t}}$ is an orthogonal linear projection,
$\widehat G_t^{\mathrm{proj}}\in\partial\widehat f_t(\widetilde W_t)$. Projected OGD on the Frobenius ball
$\mathcal H_B$ therefore gives, for every $\widetilde W^*\in\mathcal H_B$,
\begin{equation}\label{eq:global_ogd}
\sum_{t=1}^T\bigl(\widehat f_t(\widetilde W_t)-\widehat f_t(\widetilde W^*)\bigr)
\leq \frac{2B^2}{\eta_T}+\frac12\sum_{t=1}^T\eta_t\|\widehat G_t^{\mathrm{proj}}\|_F^2.
\end{equation}
Let $K_t=|\overline{\mathcal Y}_{A_t}|$. For every target action $a$ and
$G_{t,a}\in\partial\Phi_{01}(\widetilde W,\overline x_t,a)$, Lemma~\ref{lemma:proj_grad_bound} gives
\[
\|\Pi_{\mathcal K_{A_t}}G_{t,a}\|_F^2
\leq K_t\rho^2.
\]
Moreover, $q_t(a)\geq\gamma_t/K_t$. Directly conditioning on the single
sampled action and using $(1-c_j)^2\leq1-c_j$ gives
\begin{equation}\label{eq:general_second_moment}
\mathbb E_t\|\widehat G_t^{\mathrm{proj}}\|_F^2
\leq \frac{K_t^3\rho^2}{\gamma_t}
\leq \frac{N^3\rho^2}{\gamma_t},
\end{equation}
where $\mathbb E_t$ conditions on the history and averages over the current label
and learner randomization. The estimator is conditionally unbiased, so taking
expectations in \eqref{eq:global_ogd} bounds the conditional surrogate regret.
When $K_t=2$, the algorithm uses an exact-loss subgradient; Lemma~\ref{lemma:proj_grad_bound}
and $\gamma_t\leq1$ show that the same displayed second-moment bound remains valid.
With
$\eta_t=B/(N^{3/2}\rho t^{2/3})$ and
$\gamma_t=\min\{1/2,t^{-1/3}\}$,
\[
\frac{2B^2}{\eta_T}=2BN^{3/2}\rho T^{2/3},
\]
and
\[
\frac12\sum_{t=1}^T\eta_t\mathbb E_t\|\widehat G_t^{\mathrm{proj}}\|_F^2
\leq \frac{BN^{3/2}\rho}{2}
\sum_{t=1}^T\max\{2t^{-2/3},t^{-1/3}\}
=O(BN^{3/2}\rho T^{2/3}).
\]
This proves the theorem.

\end{proof}

\subsection{Proof under the concentrated-score assumption (Theorem~\ref{thm:low-noise_assum})}
\label{appendix:low-noise_assum}

Let $K_t=|\overline{\mathcal Y}_{A_t}|$, $\nu=T^{-1/2}$, and
$a_t^*=y_{\max}(\overline x_t)$. We first establish the calibration slack that
pays for importance weighting whenever the learner's deterministic prediction is
not $a_t^*$.

\begin{lemma}[Calibration slack]\label{lemma:low_noise_slack}
Fix $K\geq3$ and nonnegative scores $(s_a)_{a=1}^K$. Let
$S=\sum_as_a$, $\overline s_a=s_a/S$, and suppose
$\overline s_{a^*}\geq1-\nu$. For a zero-sum score vector $u\in\mathbb R^K$,
let $r\in\arg\max_a u_a$ and assume $r\neq a^*$. If
\[
C(u)=\sum_{a=1}^K(S-s_a)[1+u_a]_+,
\qquad C^*=K(S-s_{a^*}),
\]
then, for $T\geq16$,
\[
\delta(u):=C(u)-C^*-(s_{a^*}-s_r)\geq\frac14.
\]
\end{lemma}

\begin{proof}
Put $v_a=[1+u_a]_+$, $U=\sum_av_a$, and
$L=\sum_{a\neq a^*}v_a$. Since $\sum_au_a=0$, $U\geq K$.
Also $u_r\geq0$ and $v_{a^*}\leq v_r\leq L$, so $L\geq U/2\geq K/2$.
Expanding the definition of $C$ gives
\[
\begin{aligned}
\delta(u)
&=(S-s_{a^*})(U-K)+(s_{a^*}-s_r)(v_r-1)\\
&\quad+\sum_{a\notin\{a^*,r\}}(s_{a^*}-s_a)v_a.
\end{aligned}
\]
For every $a\neq a^*$,
$s_{a^*}-s_a=S(\overline s_{a^*}-\overline s_a)
\geq S(1-2\nu)\geq1-2\nu$, because $S\geq1$.
All remaining terms are nonnegative, and hence
\[
\delta(u)\geq(1-2\nu)(L-1)
\geq\frac{(1-2\nu)(K-2)}2\geq\frac14.
\]
\end{proof}

\begin{lemma}[Conditional second moment]\label{lemma:low_noise_second_moment}
Let $r_t$ be the deterministic prediction of Algorithm~\ref{algo:cstnd_hinge_loss}.
The importance-weighted projected subgradient satisfies
\[
\mathbb E_t\|\widehat G_t^{\mathrm{proj}}\|_F^2
\leq\frac{N^3\rho^2}{\gamma_t}.
\]
If $r_t=a_t^*$, then the sharper bound
\[
\mathbb E_t\|\widehat G_t^{\mathrm{proj}}\|_F^2
\leq2N^2\rho^2+\frac{\nu N^3\rho^2}{\gamma_t}
\]
holds.
\end{lemma}

\begin{proof}
For each active target action $a$, the corresponding projected hinge subgradient
has squared Frobenius norm at most $K_t\rho^2$. Conditional unbiasedness and
$(1-c_j)^2\leq1-c_j$ give
\[
\mathbb E_t\|\widehat G_t^{\mathrm{proj}}\|_F^2
\leq K_t\rho^2\sum_{a\in\overline{\mathcal Y}_{A_t}}
\frac{s(\overline x_t,a)}{q_t(a)}.
\]
The first bound follows from $q_t(a)\geq\gamma_t/K_t$ and
$S(\overline x_t)\leq K_t$. If $r_t=a_t^*$, then $q_t(a_t^*)\geq1/2$ and
\[
\sum_{a\neq a_t^*}s(\overline x_t,a)
\leq\nu S(\overline x_t)\leq\nu K_t,
\]
which gives the second bound.
\end{proof}

\theoremLowNoise*

\begin{proof}
Let $\lambda=(4N^3\rho^2)^{-1}$, so that $\eta_t=\lambda\gamma_t$.
Let $u_t$ be the active zero-sum score vector and $r_t\in\arg\max_a u_{t,a}$.
Since $B\geq N$, the construction in Appendix~\ref{appendix:bound} places every
canonical vector in $\mathfrak H_B$ pointwise. Thus
\[
\mathcal C^*_{\Phi_{\mathrm{def}},\mathfrak H_B}(\overline x_t)
=K_t\bigl(S_t-s_{t,a_t^*}\bigr),
\qquad
\mathcal C^*_{\ell_{\mathrm{def}},\mathfrak H_B}(\overline x_t)
=1-s_{t,a_t^*},
\]
where $S_t=S(\overline x_t)$ and $s_{t,a}=s(\overline x_t,a)$. Consequently,
\[
\delta_t
=\Delta\mathcal C_{\Phi_{\mathrm{def}},\mathfrak H_B}(u_t,\overline x_t)
-\Delta\mathcal C_{\ell_{\mathrm{def}},\mathfrak H_B}(r_t,\overline x_t)
=C(u_t)-K_t(S_t-s_{t,a_t^*})-(s_{t,a_t^*}-s_{t,r_t}).
\]
Randomized exploration increases conditional true loss by at most $\gamma_t$.
Summing this identity, using
$\mathcal M_{\Phi_{\mathrm{def}},\mathfrak H_B}=0$, dropping the nonnegative
true minimizability gap, and applying projected OGD yields
\begin{equation}\label{eq:low_noise_master}
\mathbb E R_{\ell_{\mathrm{def}},\mathfrak H_B}(T)
\leq \frac{2B^2}{\eta_T}
+\mathbb E\!\left[\sum_{t=1}^T\left(\frac{\eta_t}{2}
\mathbb E_t\|\widehat G_t^{\mathrm{proj}}\|_F^2-\delta_t\right)\right]
+\sum_{t=1}^T\gamma_t,
\end{equation}
On a round with $K_t\geq3$ and $r_t\neq a_t^*$,
Lemmas~\ref{lemma:low_noise_slack} and~\ref{lemma:low_noise_second_moment} give
\[
\frac{\eta_t}{2}\mathbb E_t\|\widehat G_t^{\mathrm{proj}}\|_F^2-\delta_t
\leq\frac18-\frac14\leq0.
\]
When $K_t=2$, no expert is active and binary correctness feedback reveals the
label; using the exact full-feedback gradient contributes at most
$2\lambda\rho^2\kappa\sqrt T$ in total.

On rounds with $r_t=a_t^*$, Lemma~\ref{lemma:low_noise_second_moment},
$\sum_{t=1}^T\gamma_t\leq2\kappa\sqrt T$, and $\nu=T^{-1/2}$ imply
\[
\frac12\sum_{t:r_t=a_t^*}\eta_t\mathbb E_t\|\widehat G_t^{\mathrm{proj}}\|_F^2
\leq 2\lambda N^2\rho^2\kappa\sqrt T
+\frac{\lambda N^3\rho^2}{2}\sqrt T.
\]
Finally,
\[
\frac{2B^2}{\eta_T}
\leq16B^2N^3\rho^2+\frac{8B^2N^3\rho^2}{\kappa}\sqrt T.
\]
Substituting $\kappa=BN^{3/2}\rho$ in \eqref{eq:low_noise_master} gives
\[
\mathbb E R_{\ell_{\mathrm{def}},\mathfrak H_B}(T)
=O\!\left(BN^{3/2}\rho\sqrt T+B^2N^3\rho^2\right),
\]
as claimed.
\end{proof}

\subsection{Analysis of the minimizability gap}
\label{appendix:minimizability_gap}

\begin{restatable}{theorem}{theoremMinimizabilityGap}(Characterization of minimizability gaps)
\label{thm:minimizability_gap}
Assume that every $h\in\mathcal H$ has zero active score sum and that, for
every $\overline x=(x,A)$, the class $\mathcal H$
contains a hypothesis whose active score vector equals
$|\overline{\mathcal Y}_A|-1$ on a maximizer $y_{\max}$ of
$s(\overline x,\cdot)$ and equals $-1$ on every other active action. Then, for
any sequence of inputs \(\overline{\mathbf{x}} = (\overline{x}_1, \overline{x}_2, \dots, \overline{x}_T)\) and any distribution, the minimizability gap is
characterized as follows:
\[
\mathcal{M}_{\Phi_{\text{def}}, \mathcal{H}}(\overline{\mathbf{x}})
= \mathcal{R}^{*}_{\Phi_{\text{def}}, \mathcal{H}}(\overline{\mathbf{x}}) - \sum_{t \in [T]} |\overline{\mathcal{Y}}_{A_t}| \big(S(\overline{x}_t) - s(\overline{x}_t, y_{t,\max})\big).
\]
\end{restatable}

\begin{proof}
By Lemma~\ref{lemma:cali_gap_L} we can write the conditional surrogate \(\Phi_{\text{def}}\)-risk as follows:
\begin{equation*}
    \begin{split}
        \mathcal{C}_{\Phi_{\text{def}}}(h,\overline{x}) & = \sum_{y \in \overline{\mathcal{Y}}_A} s(\overline x,y) \widetilde\Phi_{01}(h,\overline x,y) \\
        & = \sum_{y \in \overline{\mathcal{Y}}_A} s(\overline x,y) \left( \sum_{y' \neq y, y' \in \overline{\mathcal{Y}}_A} \max \{0, 1 + h(\overline{x},y')\}\right) \\
        & = \sum_{y \in \overline{\mathcal{Y}}_A}\max \{0, 1 + h(\overline{x},y)\}(S(\overline{x}) - s(\overline{x},y)),
    \end{split}
\end{equation*}

where \(\sum_{y \in \overline{\mathcal{Y}}_A} h(\overline{x}, y) = 0\). Define \(\mathcal{Y}_{\leq -1} = \{y \in \overline{\mathcal{Y}}_A: h(\overline{x},y) \leq -1\}\) and \(\mathcal{Y}_{> -1} = \overline{\mathcal{Y}}_A \setminus \mathcal{Y}_{\leq -1}\). We can rewrite the conditional surrogate risk as

\[
\mathcal{C}_{\Phi_{\text{def}}}(h,\overline{x}) = \sum_{y \in \mathcal{Y}_{> -1}}\big(1 + h(\overline{x},y)\big) \big(S(\overline{x}) - s(\overline{x},y)\big).
\]

Define \(y_{\max} = \arg\max_{y \in \overline{\mathcal{Y}}_{A}} s(\overline{x}, y)\), we have \(S(\overline{x}) - s(\overline{x}, y) \geq S(\overline{x}) - s(\overline{x}, y_{\max})\) and 
\[
\sum_{y \in \mathcal{Y}_{> -1}}\big(1 + h(\overline{x},y)\big) = |\mathcal{Y}_{> -1}| - \big(\sum_{y \in \mathcal{Y}_{\leq -1}} h(\overline{x},y)\big) \geq |\mathcal{Y}_{> -1}| + |\mathcal{Y}_{\leq -1}| = |\overline{\mathcal{Y}}_A|.
\]

Combining the two aforementioned inequalities gives us
\[
\mathcal{C}_{\Phi_{\text{def}}}(h,\overline{x}) \geq |\overline{\mathcal{Y}}_A| \big(S(\overline{x}) - s(\overline{x}, y_{\max})\big).
\]

Equality occurs when \(h(\overline{x}, y) = \begin{cases}
    |\overline{\mathcal{Y}}_A| - 1 & \text{if } y = y_{\max} \\
    -1 & \text{otherwise}.
\end{cases}\)

Note that there exists such \(h\) in \(\mathcal{H}\), which ends our proof.

\end{proof}

\section{Experiments}
\label{appendix:experiments}

\subsection{Details of Experiment Settings}
\label{appendix:exp_settings}

\newcolumntype{C}[1]{>{\centering\arraybackslash}m{#1}}
\begin{table}[H]
\centering
\caption{Characteristics of the three experimental settings.}
\begin{tabularx}{0.88\textwidth}{C{0.18\textwidth}C{0.20\textwidth}X}
\toprule
\textbf{Varying expertise} & \textbf{Varying availability} & \textbf{Description} \\
\midrule
\ding{55} & \ding{55} & Expertise and availability are fixed. \\
\ding{55} & \ding{51} & Expertise is fixed and availability drifts. \\
\ding{51} & \ding{51} & Both expertise and availability drift. \\
\bottomrule
\end{tabularx}
\label{tab:expertise_availability}
\end{table}

\subsubsection{Setting 1: Fixed Availability and Expertise}

Here, all experts are continuously available, and each expert’s domain of competence remains stationary throughout the learning horizon. That is each expert achieves perfect accuracy on its region of expertise, but is uniformly random outside of it. This setting corresponds to the standard assumption in classical L2D frameworks \citep{madras2018predict, mozannar2021consistent}, but we study it here under the online learning protocol, where data arrives sequentially.

A representative example is a clinical decision-support system that integrates several automated diagnostic models, each specialized for a distinct imaging modality (e.g., chest X-rays, MRIs, or CT scans).  
These expert models remain permanently accessible and their competence boundaries do not change over time, while the system continuously processes incoming patient data.  
This setup reflects realistic deployments where the expert pool is fixed but the input stream evolves, such as hospital triage platforms or other continuous-monitoring systems.

\subsubsection{Setting 2: Drifting Availability and Fixed Expertise}

Let $G_{j,t}\in\{0,1\}$ denote whether expert $g_j$ is available at round $t$. Availability is sampled through $G_{j,t}\sim \text{Bern}(p_{j,t})$, with $p_{j,t}$ following a random walk initialized at $0.7$ with standard deviation $\sigma=2\times 10^{-3}$ on $[0,1]$.

A practical realization of this setting occurs in multi-model inference systems deployed under dynamic resource constraints.  
For example, a pool of pre-trained expert networks—each specialized for a distinct input modality or domain—may share a limited GPU or compute budget.  
At any given time, only a subset of models can be queried due to bandwidth or scheduling limitations, even though their predictive functions remain unchanged.  
Similar behavior appears in edge-device ensembles, where network connectivity or energy constraints temporarily restrict access to otherwise stable models.

\subsubsection{Setting 3: Drifting Availability and Expertise}

Expert availability evolves as before, following the same random-walk process. 
Meanwhile, each expert's region of expertise drifts gradually over time. 
We fix both the initial and final regions, separated by $10^5$ rounds. 
With these boundaries fixed, the experts' regional accuracies evolve according to a Brownian bridge.
This construction models environments where both participation and competence vary over time while preserving the overall structure of the expert population.  

A concrete analogue arises in collaborative clinical decision-support systems.  
Clinicians differ in their availability due to shifts or workload, and their diagnostic expertise naturally evolves as they gain experience, adapt to new medical guidelines, or incorporate feedback from prior cases.  
The learning system must therefore adapt continuously to changing access to human experts and to gradual shifts in their decision boundaries.  
This dynamic also captures other human–AI ecosystems—such as annotation pipelines or hybrid moderation teams—where both availability and skill levels evolve throughout deployment.

\subsection{Details of Synthetic Datasets Experiments}
\label{appendix:synthetic_dataset}

We evaluate our approach on synthetic datasets with \( n_e = 3 \) experts and \( n = 6 \) classes. Following the SynSep dataset in \cite{kakade2008efficient,van2021beyond}, we generate inputs \( x_t \in \mathbb{R}^d \) with \( d = 120 \). Each input–label pair \( (x_t, y_t)_{t=1}^{T} \) is sampled from one of \( n \) linearly separable clusters \( \Omega_y = \{(x_t, y_t): y_t = y\} \).

To introduce label noise, for each \( x_t \in \Omega_y \) we flip the true label to a uniformly random alternative class in \( [n] \setminus \{y\} \) with probability \( p_{y,t} \). The noise vector at round \( t \) is \( \mathbf{p}_t = (p_{y,t})_{y \in [n]} \). We initialize \( \mathbf{p}_0 = [0.3, 0.3, 0.3, 0.3, 0.0, 0.0] \) and evolve it via a random walk with Gaussian perturbations of mean \( 0 \) and standard deviation \( \sigma = 2 \times 10^{-3} \).

At initialization \(t=0\), expert \(g_1\) is knowledgeable on classes \(\{1,2\}\) and predicts the \emph{post-noise} labels correctly on inputs from \(\Omega_1\) and \(\Omega_2\), while predicting uniformly at random on other clusters. Expert \(g_2\) is knowledgeable on \(\{3,4\}\) with the same behavior. Expert \(g_3\) is a weak expert that predicts uniformly at random on all inputs.

In all settings, we run our algorithm with learning rate \(\eta_t = \frac{\eta_0}{\sqrt{t}}\) where \(\eta_0\) is tuned to \(5 \times 10^{-4}\). The exploration is fixed to \(\gamma_t = \min\{\frac{1}{2}, \frac{10}{\sqrt{t}}\}\).
All reported true deferral losses use the normalized costs defined in Section~\ref{section:problem_formulation}. The random baseline samples uniformly from the active set of classifier and available-expert actions at each round.

In the first two settings, where expertise remains constant, the lowest-cost active action is an available knowledgeable expert on regions \(\{1,2,3,4\}\) and the classifier on regions \(\{5,6\}\); if the knowledgeable expert is unavailable, the learner selects among the remaining active actions. In the third setting, the experts' expertise regions gradually drift among \(\{1,2,3,4\}\).

\paragraph{Setting 1: Fixed Availability and Expertise.} We report the results of this setting in Figure~\ref{fig:syn_exp1}. Accuracies of experts are given in Table~\ref{tab:average_acc_setting_1}. These unconditional averages are lower than the accuracies on queried rounds in Figure~\ref{fig:syn_exp1_expert_acc_queried}; the queried-round curves describe which experts the learned policy selects.

\begin{table}[H]
\centering
\caption{Average Accuracies of Experts for setting 1. }
\renewcommand{\arraystretch}{1.4}
\begin{tabular}{cccc}
\toprule
& \textbf{Expert $g_1$} & \textbf{Expert $g_2$} & \textbf{Expert $g_3$} \\
\midrule
Accuracy (\%) & $45.55$ & $45.58$ & $16.66$ \\
\bottomrule
\end{tabular}
\label{tab:average_acc_setting_1}
\end{table}


Our method's predictive accuracy increases while its normalized true deferral loss decreases over the observed rounds (Figures~\ref{fig:syn_exp1_mean_acc} and~\ref{fig:syn_exp1_mean_def_loss}).
Figures~\ref{fig:syn_exp1_expert_acc_queried} and~\ref{fig:syn_exp1_def_ratio} show that experts \(g_1\) and \(g_2\) are queried most frequently and have high accuracy on selected rounds.
In contrast, the weaker expert \(g_3\) has both lower queried accuracy and a small deferral ratio.
As expected in this fixed setting, expert availability remains constant across all rounds (Figure~\ref{fig:syn_exp1_expert_avail}).

\begin{figure}[H]
    \centering
    \captionsetup{font=footnotesize}
    \begin{subfigure}[t]{0.32\textwidth}
        \centering
        \includegraphics[height=3.2cm]{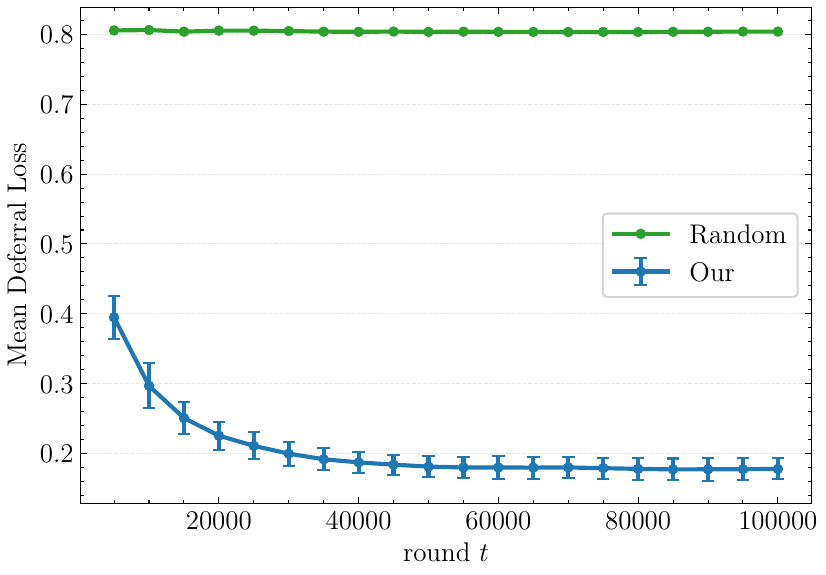}
        \caption{Average true deferral loss.}
        \label{fig:syn_exp1_mean_def_loss}
    \end{subfigure}\hfill
    \begin{subfigure}[t]{0.32\textwidth}
        \centering
        \includegraphics[height=3.2cm]{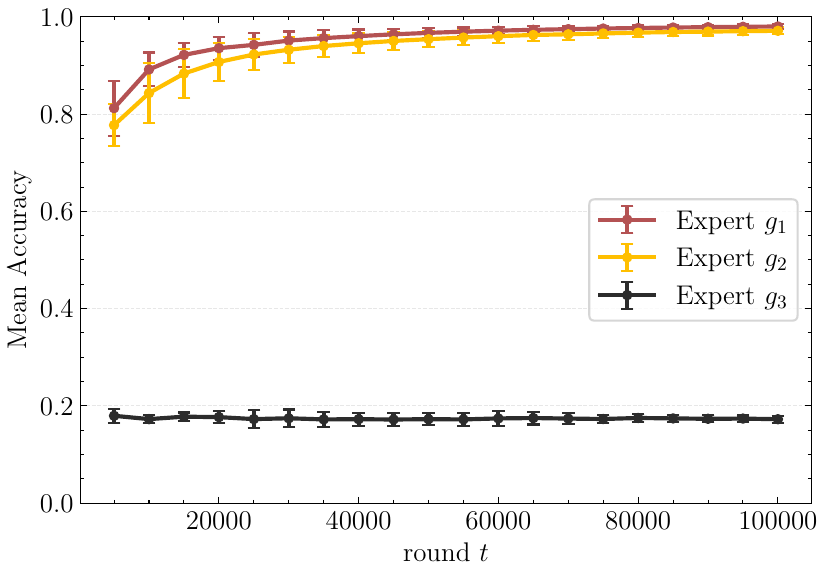}
        \caption{Average accuracy of the experts queried by the algorithm.}
        \label{fig:syn_exp1_expert_acc_queried}
    \end{subfigure}\hfill
    \begin{subfigure}[t]{0.32\textwidth}
        \centering
        \includegraphics[height=3.2cm]{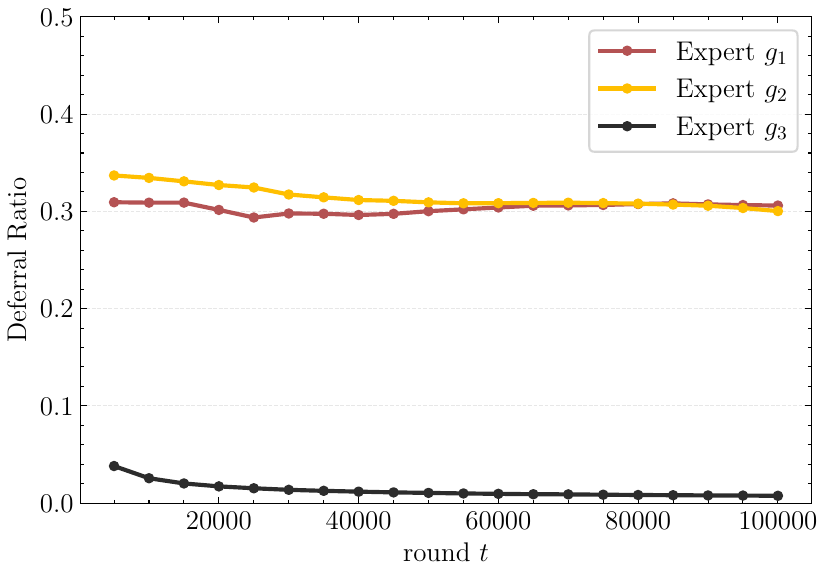}
        \caption{Average expert deferral ratio.}
        \label{fig:syn_exp1_def_ratio}
    \end{subfigure}

    \vspace{0.35em}

    \begin{subfigure}[t]{0.48\textwidth}
        \centering
        \includegraphics[height=3.2cm]{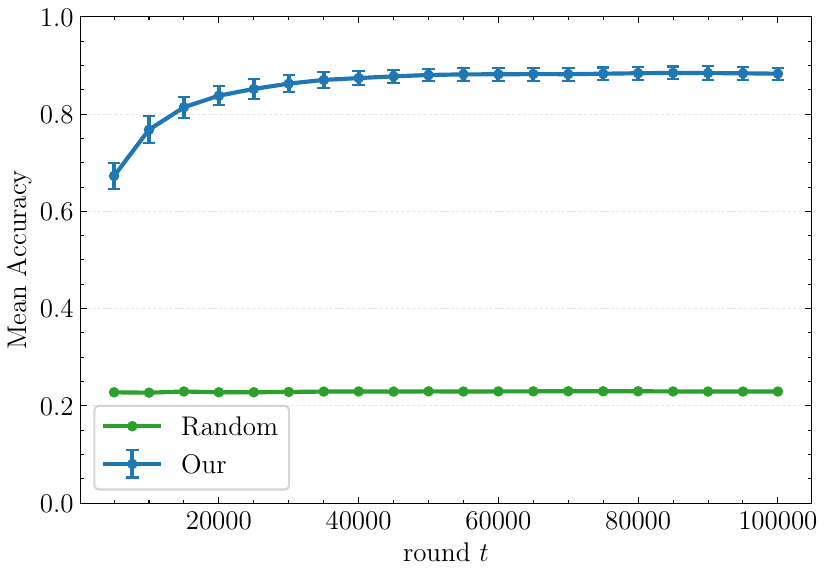}
        \caption{Average accuracy. }
        \label{fig:syn_exp1_mean_acc}
    \end{subfigure}\hfill
    \begin{subfigure}[t]{0.48\textwidth}
        \centering
        \includegraphics[height=3.2cm]{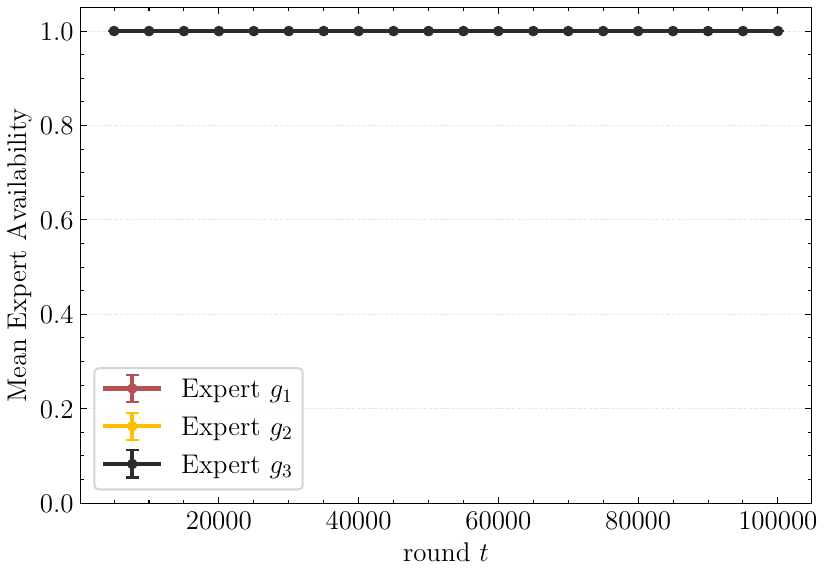}
        \caption{Expert availability. }
        \label{fig:syn_exp1_expert_avail}
    \end{subfigure}

    \caption{Results of the synthetic experiment for \textbf{setting 1: fixed availability and expertise}. 
    Error bars represent standard deviations across five independent runs.}
    \label{fig:syn_exp1}
\end{figure}

\paragraph{Setting 2: Drifting Availability and Fixed Expertise.} We report results for this setting in Figure~\ref{fig:syn_exp2}. Accuracies of experts are given in Table~\ref{tab:average_acc_setting_2}. 

\begin{table}[H]
\centering
\caption{Average Accuracies of Experts for setting 2. }
\renewcommand{\arraystretch}{1.4}
\begin{tabular}{cccc}
\toprule
& \textbf{Expert $g_1$} & \textbf{Expert $g_2$} & \textbf{Expert $g_3$} \\
\midrule
Accuracy (\%) & $45.55$ & $45.58$ & $16.66$ \\
\bottomrule
\end{tabular}
\label{tab:average_acc_setting_2}
\end{table}


Compared with the previous experiment (Figure~\ref{fig:syn_exp1}), the plotted accuracy decreases slightly under reduced and time-varying expert availability.
Figure~\ref{fig:syn_exp2_def_ratio} reports the resulting deferral ratios, while Figure~\ref{fig:syn_exp2_expert_acc_queried} shows high expert accuracy on the rounds assigned to each expert.

\begin{figure}[H]
    \centering
    \captionsetup{font=footnotesize}
    \begin{subfigure}[t]{0.32\textwidth}
        \centering
        \includegraphics[height=3.2cm]{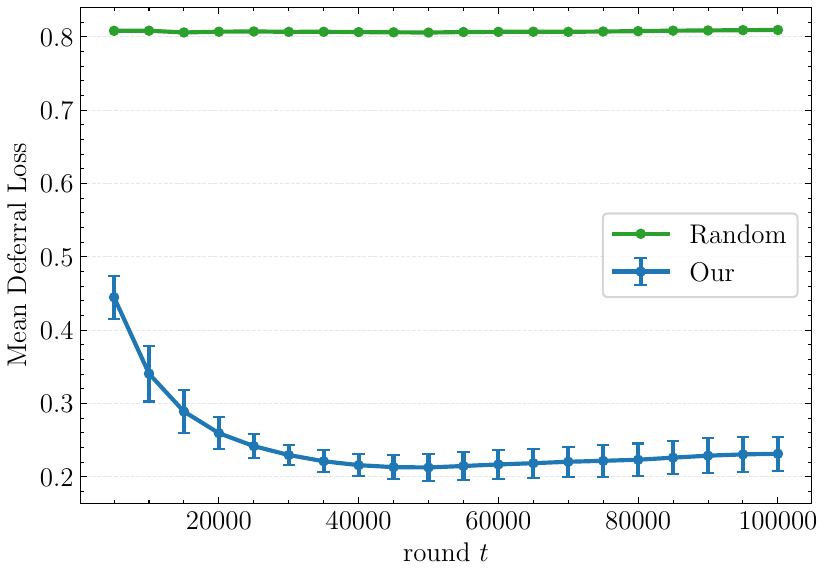}
        \caption{Average true deferral loss.}
        \label{fig:syn_exp2_mean_def_loss}
    \end{subfigure}\hfill
    \begin{subfigure}[t]{0.32\textwidth}
        \centering
        \includegraphics[height=3.2cm]{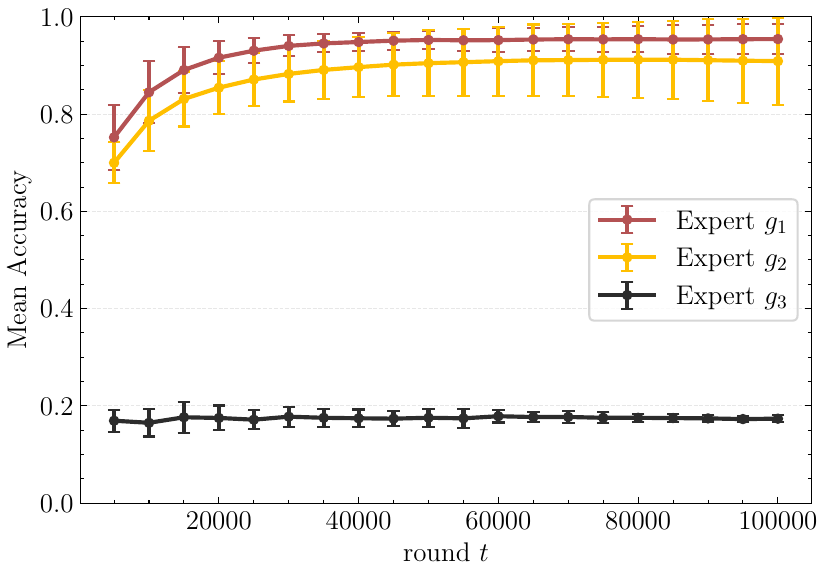}
        \caption{Average accuracy of the experts queried by the algorithm.}
        \label{fig:syn_exp2_expert_acc_queried}
    \end{subfigure}\hfill
    \begin{subfigure}[t]{0.32\textwidth}
        \centering
        \includegraphics[height=3.2cm]{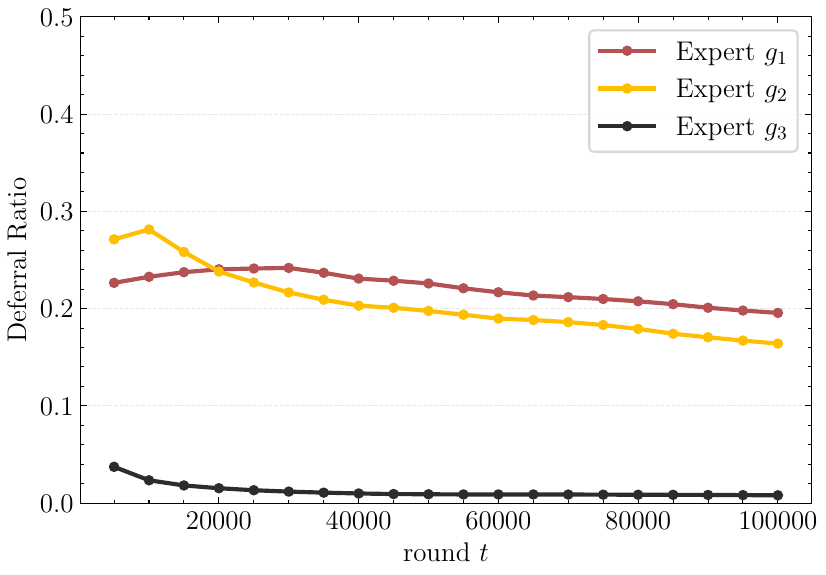}
        \caption{Average expert deferral ratio.}
        \label{fig:syn_exp2_def_ratio}
    \end{subfigure}

    \vspace{0.35em}

    \begin{subfigure}[t]{0.48\textwidth}
        \centering
        \includegraphics[height=3.2cm]{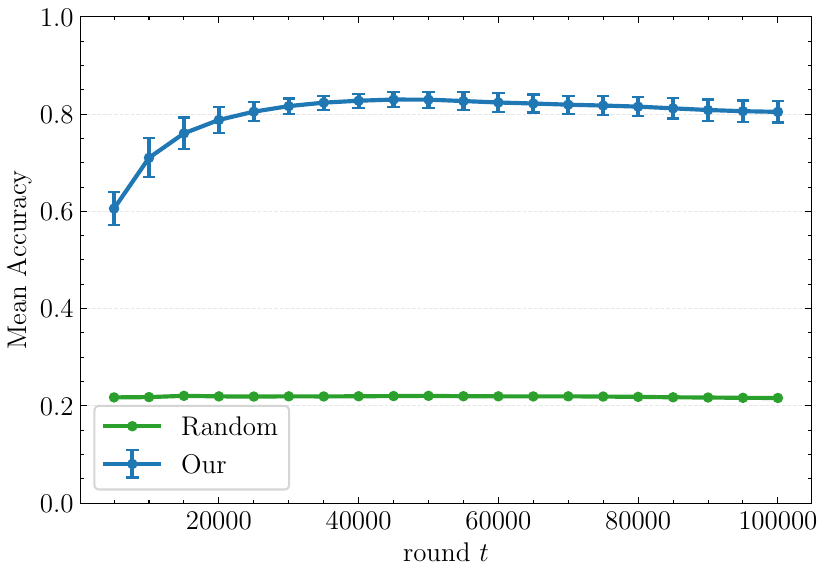}
        \caption{Average accuracy. }
        \label{fig:syn_exp2_mean_acc}
    \end{subfigure}\hfill
    \begin{subfigure}[t]{0.48\textwidth}
        \centering
        \includegraphics[height=3.2cm]{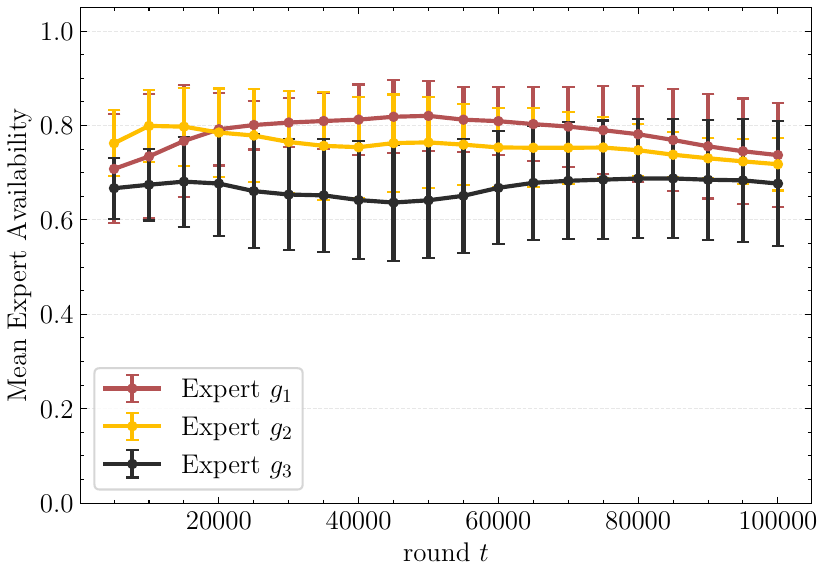}
        \caption{Expert availability. }
        \label{fig:syn_exp2_expert_avail}
    \end{subfigure}

    \caption{Results of the synthetic experiment for \textbf{setting~2: drifting availability and fixed expertise}.  
    Error bars represent standard deviations across five independent runs.}
    \label{fig:syn_exp2}
\end{figure}

\paragraph{Setting 3: Drifting Availability and Expertise} We report results for this setting in Figure~\ref{fig:syn_exp3}. Varying expertise of expert is shown in Figure \ref{fig:syn_exp3_expert_acc_by_region}.  


\begin{figure}[htbp]
  \centering
  \includegraphics[width=0.25\linewidth]{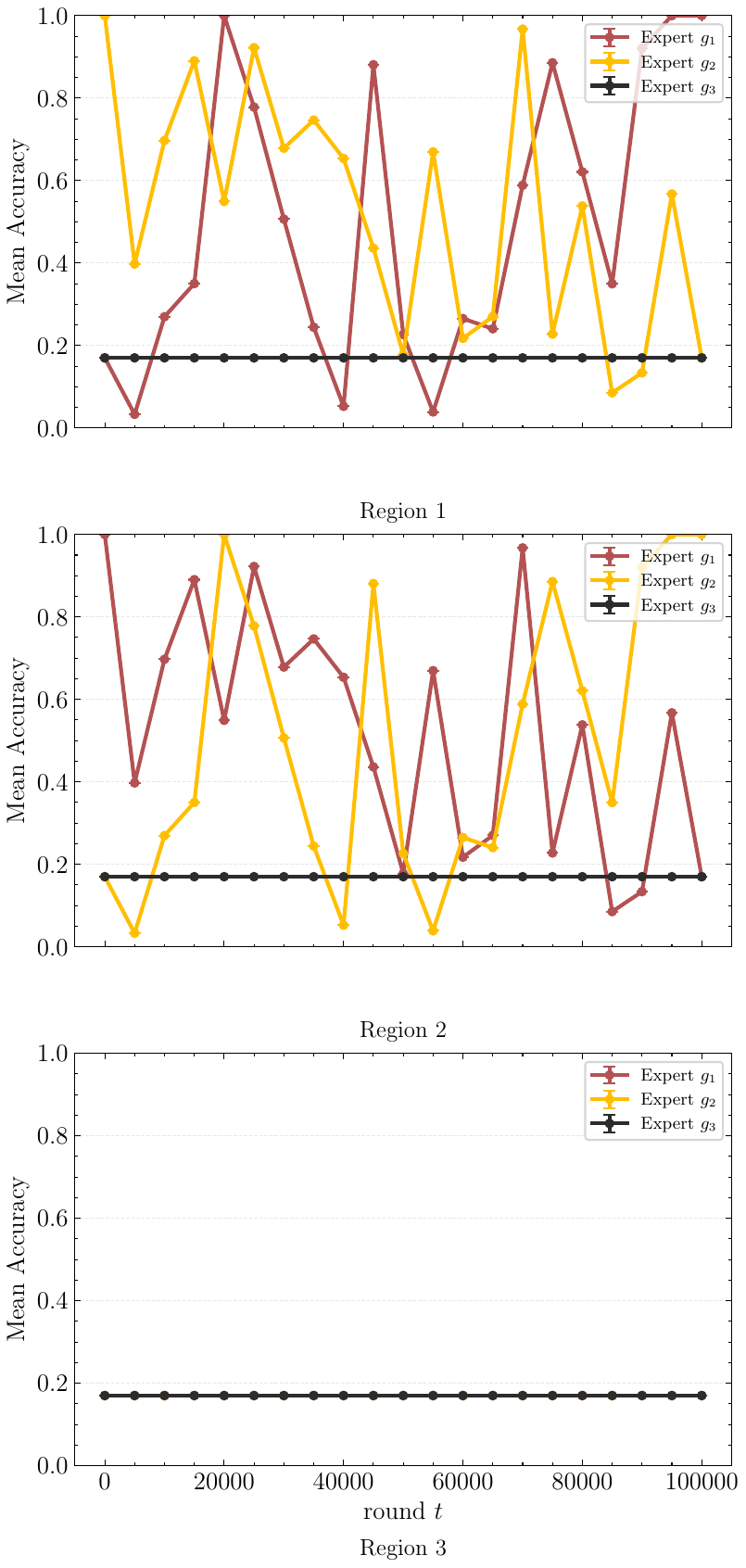}
  \caption{Evolution of the Expert Accuracies by Regions}
  \label{fig:syn_exp3_expert_acc_by_region}
\end{figure}

Figure~\ref{fig:syn_exp3_expert_acc_by_region} shows that expert \(g_1\) gradually shifts its expertise from region~2 to region~1, while expert \(g_2\) moves in the opposite direction, from region~1 to region~2. Expert \(g_3\) has no region of expertise, as it is uniformly random.
The queried-expert accuracy and deferral-ratio curves change with these regional shifts (Figures~\ref{fig:syn_exp3_expert_acc_queried} and~\ref{fig:syn_exp3_def_ratio}).
The corresponding task-accuracy and normalized deferral-loss curves are shown in Figures~\ref{fig:syn_exp3_mean_acc} and~\ref{fig:syn_exp3_mean_def_loss}.

\begin{figure}[H]
    \centering
    \captionsetup{font=footnotesize}
    \begin{subfigure}[t]{0.32\textwidth}
        \centering
        \includegraphics[height=3.2cm]{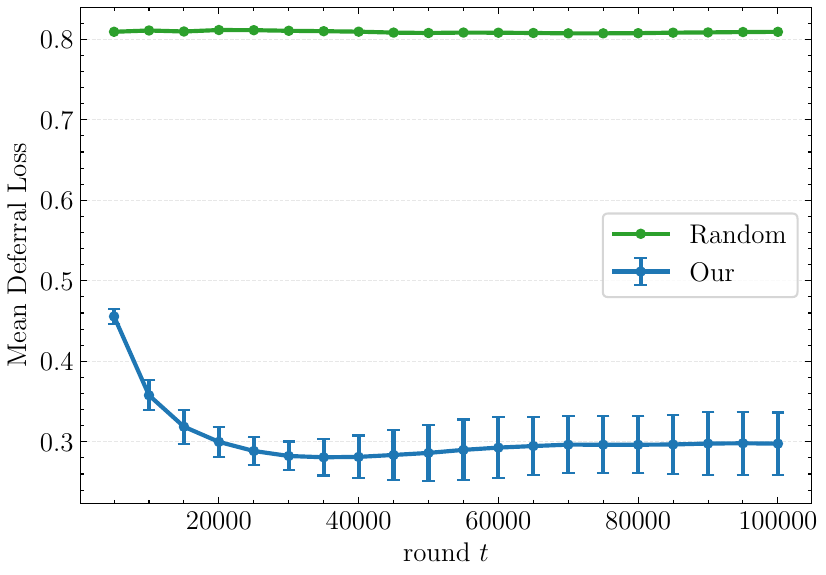}
        \caption{Average true deferral loss. }
        \label{fig:syn_exp3_mean_def_loss}
    \end{subfigure}\hfill
    \begin{subfigure}[t]{0.32\textwidth}
        \centering
        \includegraphics[height=3.2cm]{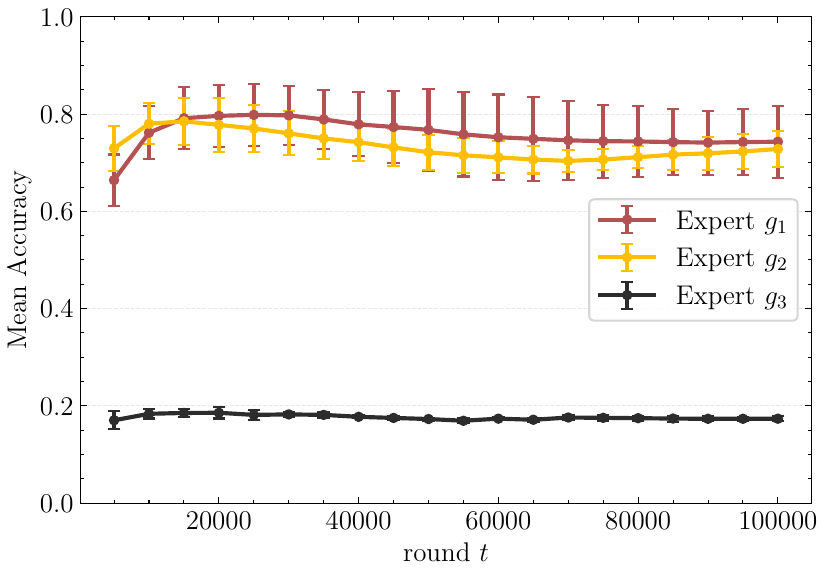}
        \caption{Average accuracy of the experts queried by the algorithm.}
        \label{fig:syn_exp3_expert_acc_queried}
    \end{subfigure}\hfill
    \begin{subfigure}[t]{0.32\textwidth}
        \centering
        \includegraphics[height=3.2cm]{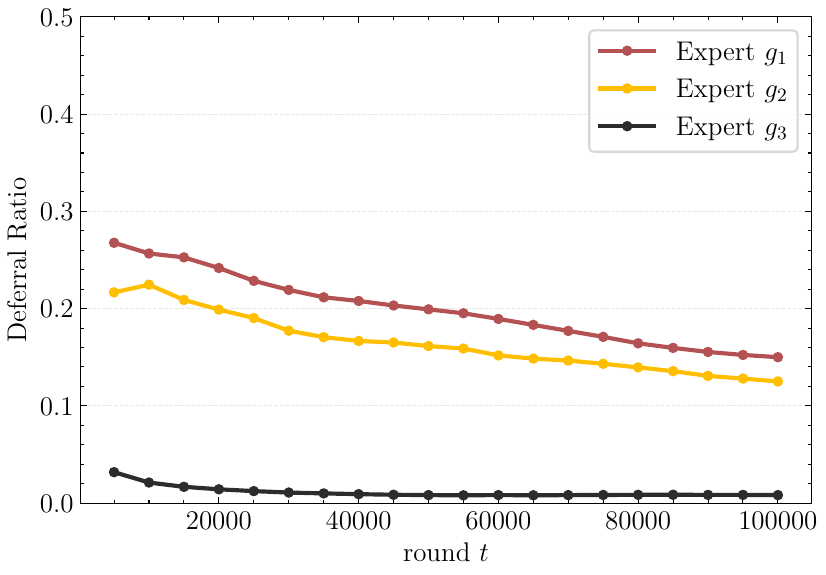}
        \caption{Average expert deferral ratio.}
        \label{fig:syn_exp3_def_ratio}
    \end{subfigure}

    \vspace{0.35em}

    \begin{subfigure}[t]{0.48\textwidth}
        \centering
        \includegraphics[height=3.2cm]{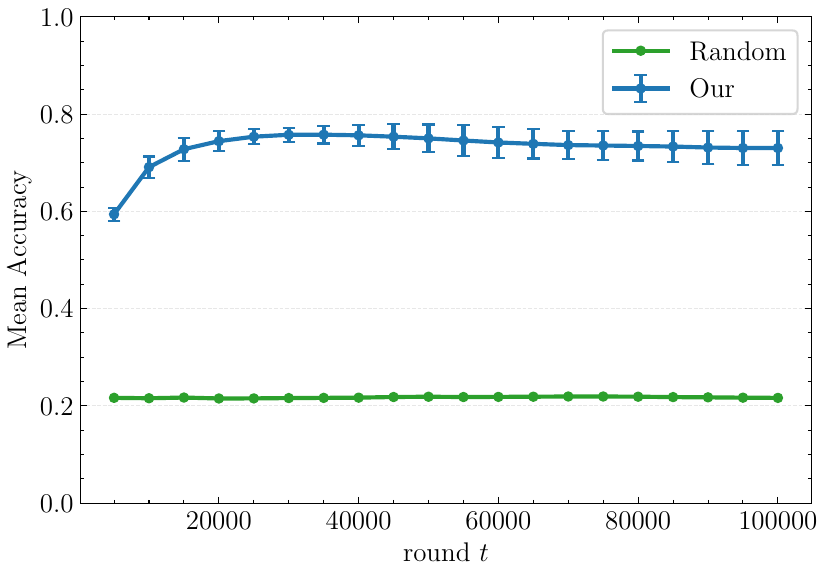}
        \caption{Average accuracy. }
        \label{fig:syn_exp3_mean_acc}
    \end{subfigure}\hfill
    \begin{subfigure}[t]{0.48\textwidth}
        \centering
        \includegraphics[height=3.2cm]{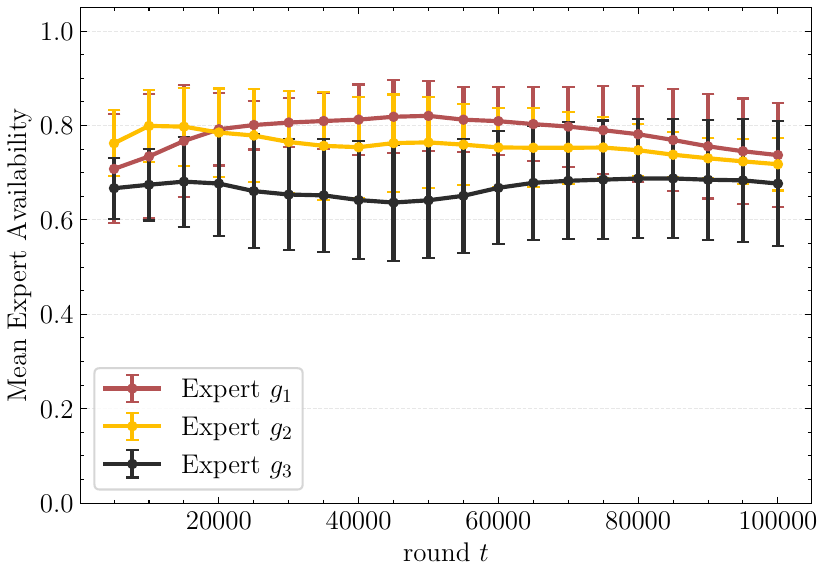}
        \caption{Expert availability. }
        \label{fig:syn_exp3_expert_avail}
    \end{subfigure}

    \caption{Results of the synthetic experiment for \textbf{setting~3: drifting availability and drifting expertise}. Whiskers denote standard deviations computed over 5 independent runs.
    }
    \label{fig:syn_exp3}
\end{figure}

\subsection{Details of Reuters4 Dataset Experiments}
\label{appendix:real_world_dataset}

\paragraph{Setting 1: Fixed Availability and Expertise.} The results are summarized in Figure~\ref{fig:real_exp1}. Accuracies of experts are given in Table~\ref{tab:average_acc_setting_1_reuters}. The plotted accuracy rises faster than the random baseline, while queried rounds have high expert accuracy and substantial expert deferral ratios.

\begin{table}[H]
\centering
\caption{Average Accuracies of Experts for setting 1 in Reuters4. }
\renewcommand{\arraystretch}{1.4}
\begin{tabular}{cccc}
\toprule
& \textbf{Expert $g_1$} & \textbf{Expert $g_2$} & \textbf{Expert $g_3$} \\
\midrule
Accuracy (\%) & $62.21$ & $61.09$ & $50.01$ \\
\bottomrule
\end{tabular}
\label{tab:average_acc_setting_1_reuters}
\end{table}

\begin{figure}[ht]
    \centering
    \captionsetup{font=footnotesize}
    \begin{subfigure}[t]{0.32\textwidth}
        \centering
        \includegraphics[height=3.2cm]{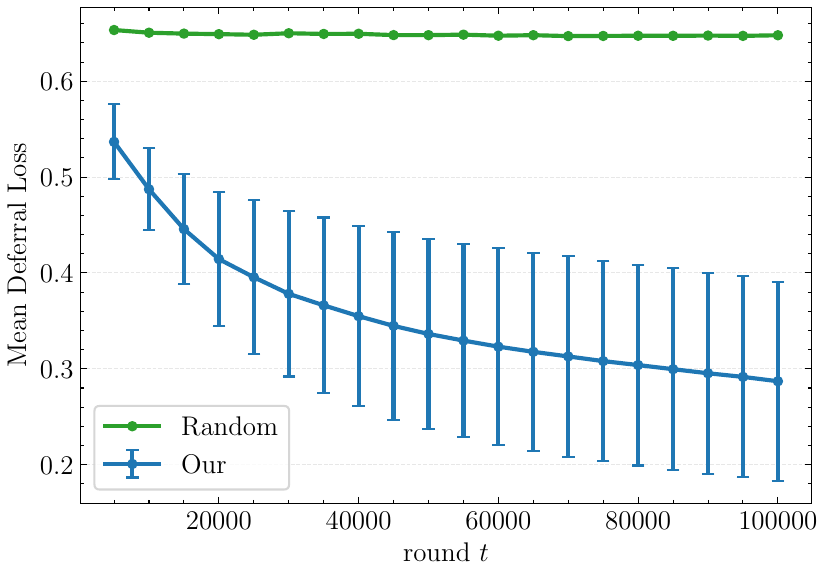}
        \caption{Average true deferral loss. }
        \label{fig:real_exp1_mean_def_loss}
    \end{subfigure}\hfill
    \begin{subfigure}[t]{0.32\textwidth}
        \centering
        \includegraphics[height=3.2cm]{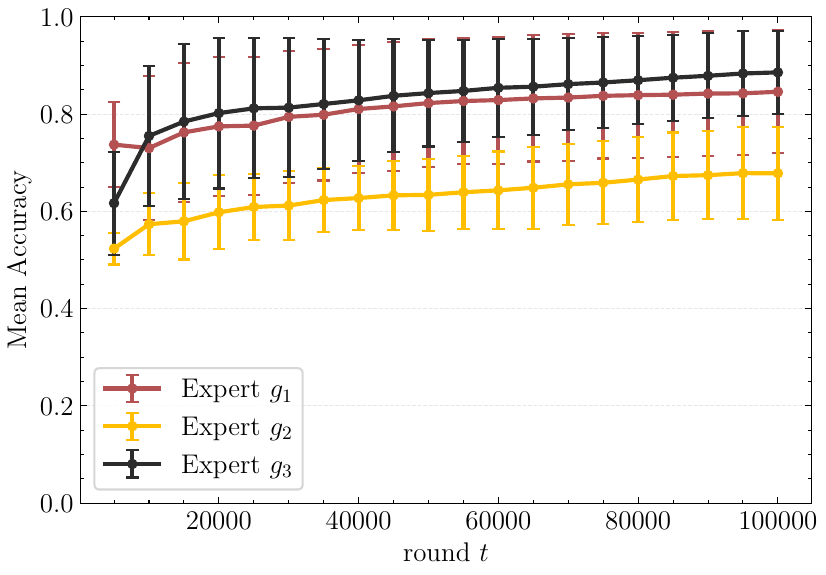}
        \caption{Average accuracy of the experts queried by the algorithm.}
        \label{fig:real_exp1_expert_acc_queried}
    \end{subfigure}\hfill
    \begin{subfigure}[t]{0.32\textwidth}
        \centering
        \includegraphics[height=3.2cm]{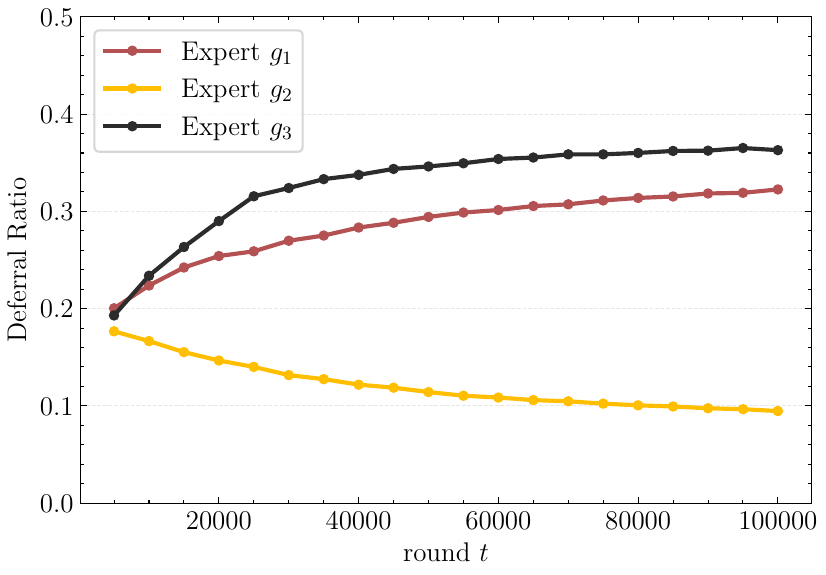}
        \caption{Average expert deferral ratio.}
        \label{fig:real_exp1_def_ratio}
    \end{subfigure}

    \vspace{0.35em}

    \begin{subfigure}[t]{0.48\textwidth}
        \centering
        \includegraphics[height=3.2cm]{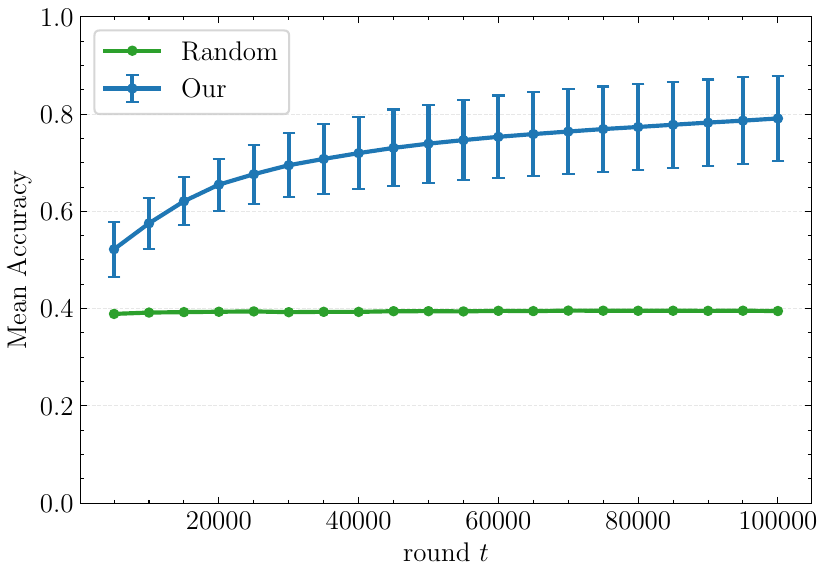}
        \caption{Average accuracy. }
        \label{fig:real_exp1_mean_acc}
    \end{subfigure}\hfill
    \begin{subfigure}[t]{0.48\textwidth}
        \centering
        \includegraphics[height=3.2cm]{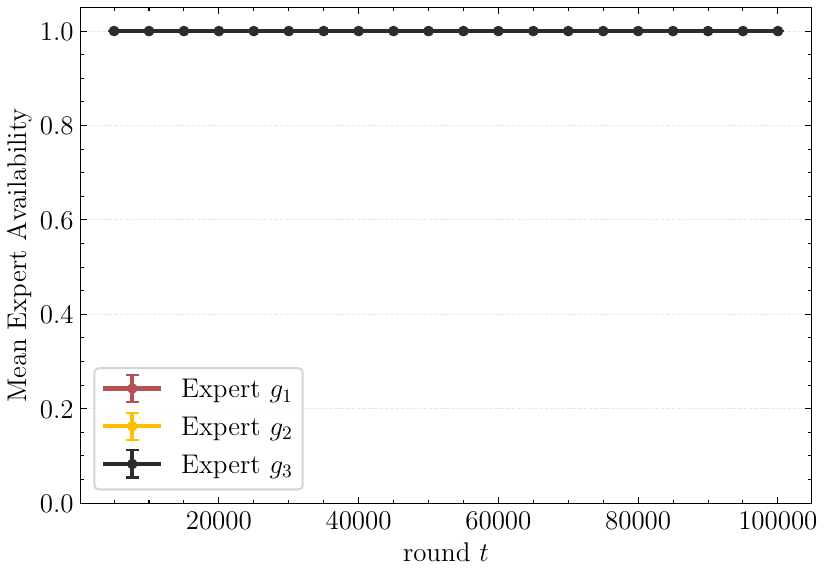}
        \caption{Expert availability. }
        \label{fig:real_exp1_expert_avail}
    \end{subfigure}

    \caption{
        Results of experiment on Reuters4 for \textbf{setting 1: fixed availability and expertise}. Whiskers denote standard deviations computed over 5 independent runs.
    }
    \label{fig:real_exp1}
\end{figure}

\paragraph{Setting 2: Drifting Availability and Fixed Expertise.} Accuracies of experts are given in Table~\ref{tab:average_acc_setting_2_reuters}. 
Figure~\ref{fig:real_exp2} shows the accuracy, normalized deferral loss, queried-expert accuracy, and deferral ratios under varying expert availability.

\begin{table}[H]
\centering
\caption{Average Accuracies of Experts for setting 2 in Reuters4. }
\renewcommand{\arraystretch}{1.4}
\begin{tabular}{cccc}
\toprule
& \textbf{Expert $g_1$} & \textbf{Expert $g_2$} & \textbf{Expert $g_3$} \\
\midrule
Accuracy (\%) & $62.21$ & $61.09$ & $50.01$ \\
\bottomrule
\end{tabular}
\label{tab:average_acc_setting_2_reuters}
\end{table}

\begin{figure}[ht]
    \centering
    \captionsetup{font=footnotesize}
    \begin{subfigure}[t]{0.32\textwidth}
        \centering
        \includegraphics[height=3.2cm]{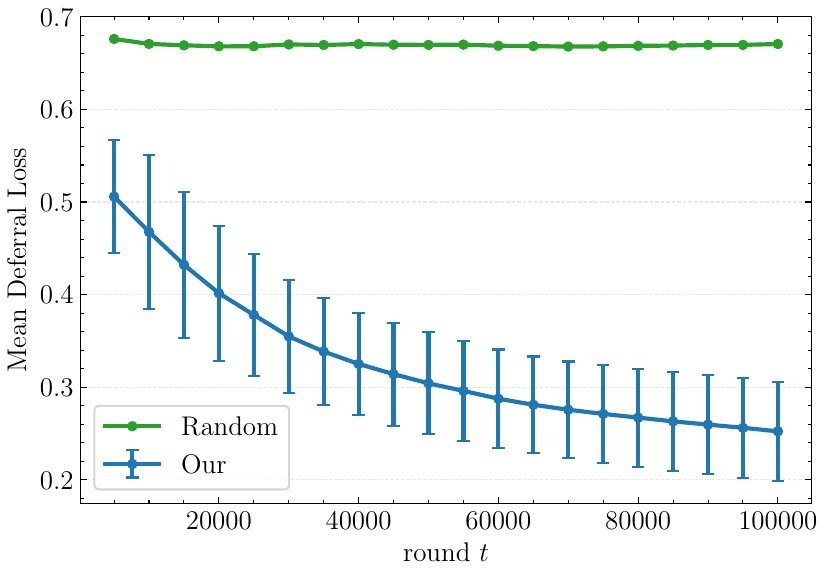}
        \caption{Average true deferral loss. }
        \label{fig:real_exp2_mean_def_loss}
    \end{subfigure}\hfill
    \begin{subfigure}[t]{0.32\textwidth}
        \centering
        \includegraphics[height=3.2cm]{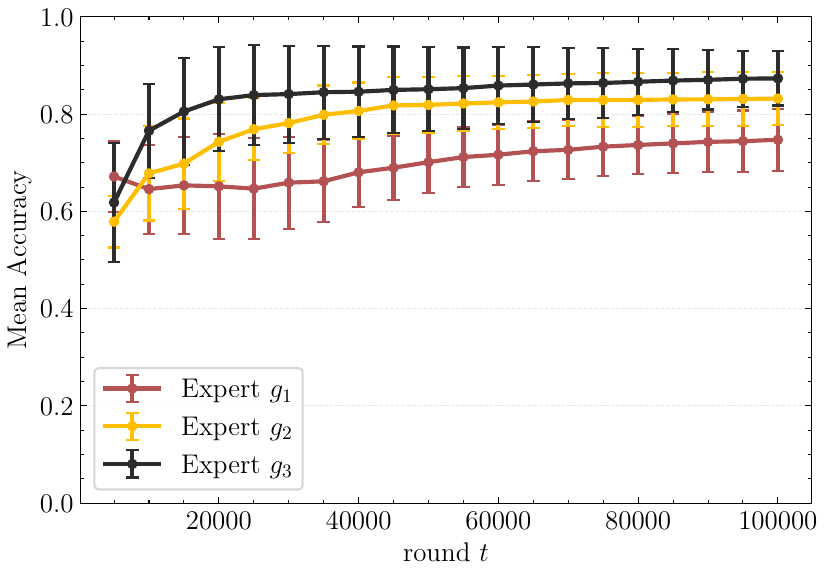}
        \caption{Average accuracy of the experts queried by the algorithm.}
        \label{fig:real_exp2_expert_acc_queried}
    \end{subfigure}\hfill
    \begin{subfigure}[t]{0.32\textwidth}
        \centering
        \includegraphics[height=3.2cm]{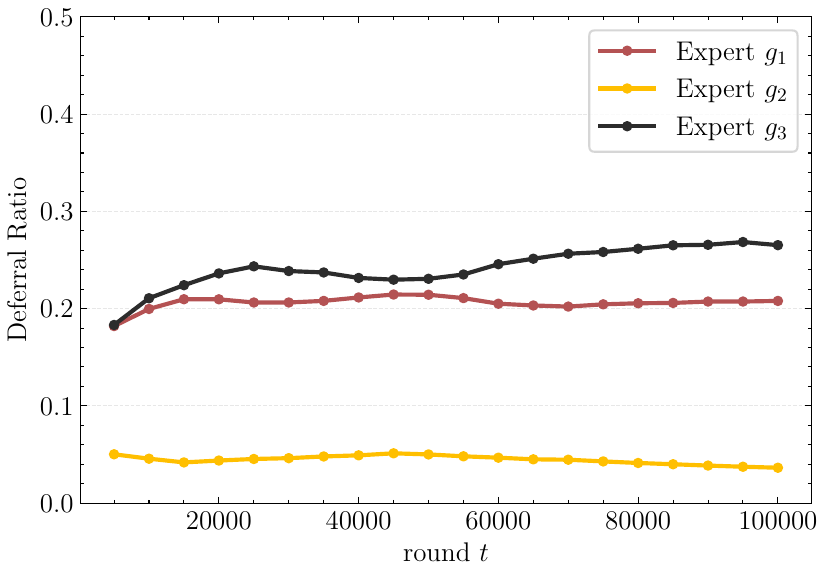}
        \caption{Average expert deferral ratio.}
        \label{fig:real_exp2_def_ratio}
    \end{subfigure}

    \vspace{0.35em}

    \begin{subfigure}[t]{0.48\textwidth}
        \centering
        \includegraphics[height=3.2cm]{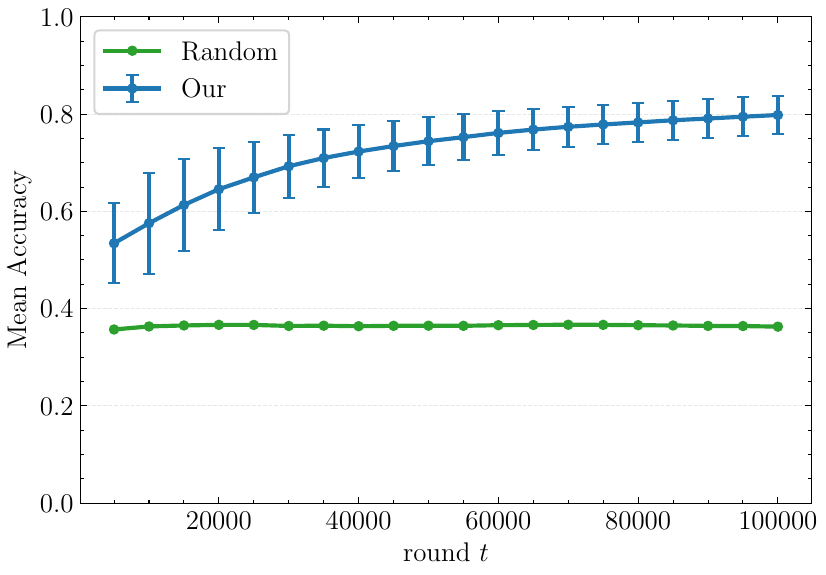}
        \caption{Average accuracy. }
        \label{fig:real_exp2_mean_acc}
    \end{subfigure}\hfill 
    \begin{subfigure}[t]{0.48\textwidth}
        \centering
        \includegraphics[height=3.2cm]{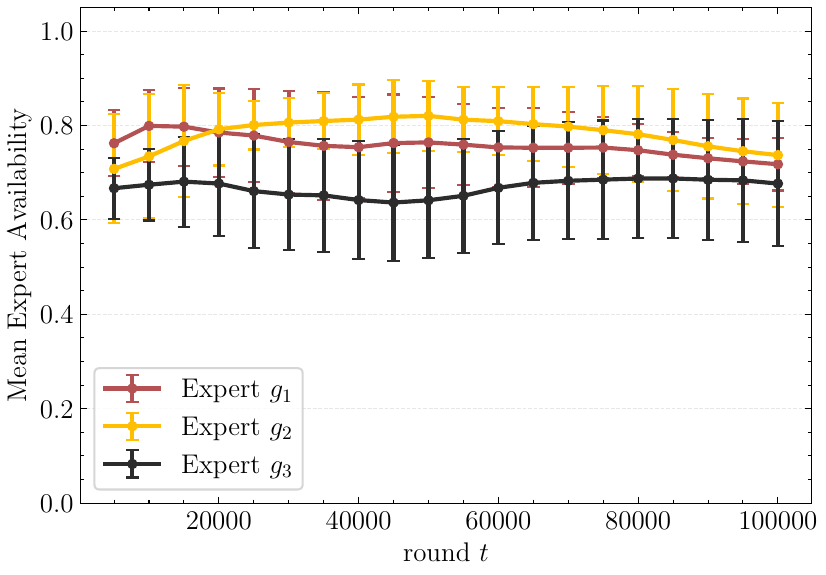}
        \caption{Expert availability. }
        \label{fig:real_exp2_expert_avail}
    \end{subfigure}

    \caption{
        Results of experiment on Reuters4 for \textbf{setting 2: drifting availability and fixed expertise}. Whiskers denote standard deviations computed over 5 independent runs.
    }
    \label{fig:real_exp2}
\end{figure}


\paragraph{Setting 3: Drifting Availability and Drifting Expertise.} 
Figure~\ref{fig:real_exp3} summarizes the results. The changing regional expert accuracies are shown in Figure~\ref{fig:real_exp3_expert_acc_by_region}.

Figure~\ref{fig:real_exp3_expert_acc_by_region} shows that expert \(g_1\) gradually shifts its expertise from classes~1 and 2 to classes~3 and 4, expert \(g_2\) shifts from classes~2 and 3 to classes~1 and 4, and expert \(g_3\) shifts from classes~3 and 4 to classes~1 and 2.
Figure~\ref{fig:real_exp3} reports the corresponding queried-expert accuracies, deferral ratios, task accuracy, and normalized deferral loss.

\begin{figure}[htbp]
  \centering
  \includegraphics[width=0.6\linewidth]{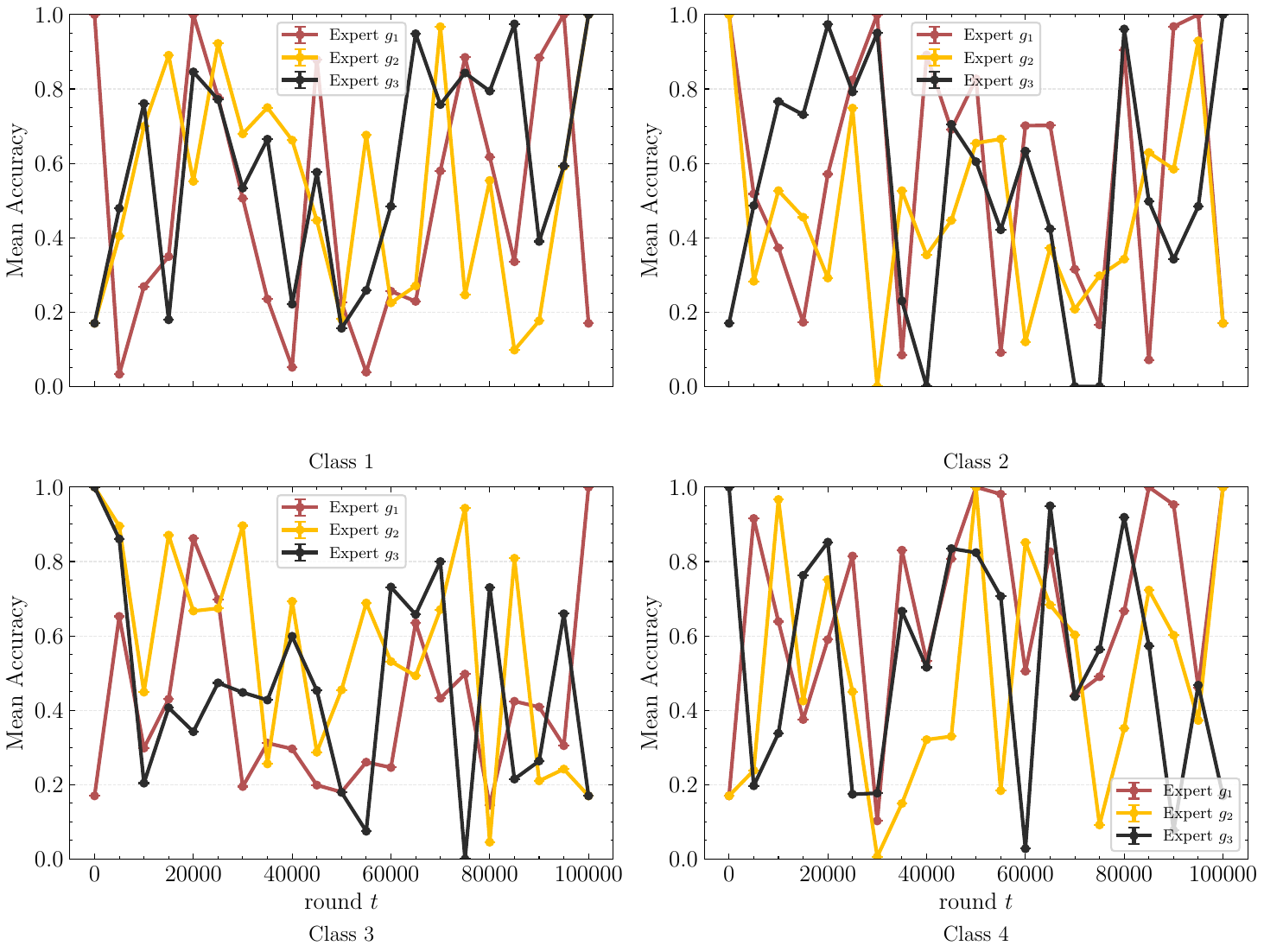}
  \caption{Expert Accuracies by Regions on Reuters4.}
  \label{fig:real_exp3_expert_acc_by_region}
\end{figure}

\begin{figure}[ht]
    \centering
    \captionsetup{font=footnotesize}
    \begin{subfigure}[t]{0.32\textwidth}
        \centering
        \includegraphics[height=3.2cm]{figures/real_exp3/mean_def_loss.pdf}
        \caption{Average true deferral loss. }
        \label{fig:real_exp3_mean_def_loss}
    \end{subfigure}\hfill
    \begin{subfigure}[t]{0.32\textwidth}
        \centering
        \includegraphics[height=3.2cm]{figures/real_exp3/expert_acc_queried.pdf}
        \caption{Average accuracy of the experts queried by the algorithm.}
        \label{fig:real_exp3_expert_acc_queried}
    \end{subfigure}\hfill
    \begin{subfigure}[t]{0.32\textwidth}
        \centering
        \includegraphics[height=3.2cm]{figures/real_exp3/expert_def_ratio.pdf}
        \caption{Average expert deferral ratio.}
        \label{fig:real_exp3_def_ratio}
    \end{subfigure}

    \vspace{0.35em}

    \begin{subfigure}[t]{0.48\textwidth}
        \centering
        \includegraphics[height=3.2cm]{figures/real_exp3/mean_acc.pdf}
        \caption{Average accuracy. }
        \label{fig:real_exp3_mean_acc}
    \end{subfigure}\hfill
    \begin{subfigure}[t]{0.48\textwidth}
        \centering
        \includegraphics[height=3.2cm]{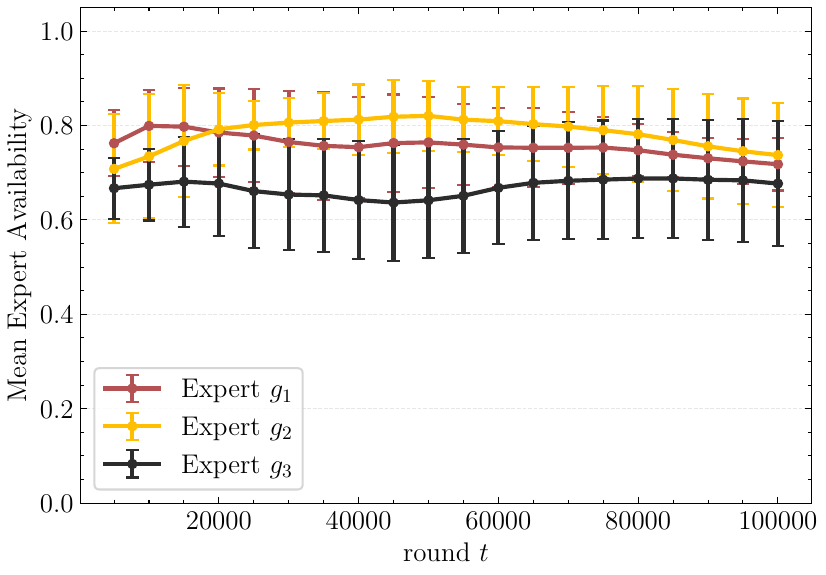}
        \caption{Expert availability. }
        \label{fig:real_exp3_expert_avail}
    \end{subfigure}

    \caption{
        Results of experiment on Reuters4 for \textbf{setting 3: drifting availability and drifting expertise}. Whiskers denote standard deviations computed over 5 independent runs.
    }
    \label{fig:real_exp3}
\end{figure}

\subsection{Details of CIFAR10H Dataset Experiments}
\label{appendix:cifar10h_dataset}

In this section, we report additional results on the CIFAR10H dataset \citep{peterson2019human}, which augments CIFAR10 \citep{krizhevsky2009learning} with real human-annotated labels. We aggregate the human annotations into three groups, each treated as a human expert \(g_i, \; i=1,2,3\). For an input image \(x\), expert \(g_i\) predicts by sampling uniformly at random from the labels in group \(i\). We further introduce expert-specific noise regions to model heterogeneous expertise: \(g_1\) is knowledgeable on labels \(\{6,7,8,9\}\), \(g_2\) is knowledgeable on labels \(\{0,1,2,9\}\), \(g_3\) is knowledgeable on labels \(\{1,2,3,4\}\). The mean accuracy of each expert after noise injection is reported in Table~\ref{tab:cifar10h:expert_acc}. Each expert uses the normalized cost \(c_j(x,y) = (\mathbf{1}\{g_j(x) \neq y\} + 0.05)/1.05\).

\begin{table}[h]
\centering
\caption{Average accuracies of experts in CIFAR10H post-noise.}
\label{tab:cifar10h:expert_acc}
\begin{tabular}{lccc}
\toprule
\textbf{Metric} & \textbf{Expert \(g_1\)} & \textbf{Expert \(g_2\)} & \textbf{Expert \(g_3\)} \\
\midrule
Mean Accuracy & 0.445 & 0.443 & 0.432 \\
\bottomrule
\end{tabular}
\end{table}

For this experiment, we use a Wide Residual Network (WideResNet) as the predictor \(h \in \mathcal{H}\), instead of a linear hypothesis. To make expert consultation meaningful, we evaluate on CIFAR10C inputs corrupted with Gaussian noise at severity level~2 \cite{hendrycks2019benchmarking}. As a benchmark, we compare against a baseline that uses the classifier when its confidence is at least \(0.5\) and otherwise defers uniformly to one of the available experts.

The results under fixed availability and expertise settings are summarized in Table~\ref{tab:cifar10h:main_res}, and the learning curves are shown in Figure~\ref{fig:cifar10h:exp}. Our method achieves both lower mean deferral loss and higher mean accuracy than the baseline. This indicates that the learned policy attains a better trade-off between prediction performance and consultation cost, rather than relying on deferral in a naive manner. The queried expert accuracies are also substantially higher under our method, showing that the algorithm does not defer uniformly, but instead routes examples toward more suitable experts. The deferral ratios further support this behavior: the algorithm learns a non-uniform consultation pattern, which suggests that it is able to distinguish between examples that should be handled by the classifier and those for which expert advice is more beneficial.

\begin{table}[ht]
\centering
\caption{Comparison between our method and the baseline, averaged over 5 runs. Lower mean deferral loss and higher accuracy are highlighted; deferral ratios are descriptive.}
\label{tab:cifar10h:main_res}
\begin{tabular}{lcc}
\toprule
\textbf{Metric} & \textbf{Our Method} & \textbf{Baseline} \\
\midrule
Mean Deferral Loss               & \textbf{0.3953} & 0.4133 \\
Mean Accuracy                    & \textbf{0.6519} & 0.6145 \\
Expert \(g_1\) Queried Accuracy  & \textbf{0.7253} & 0.4228 \\
Expert \(g_2\) Queried Accuracy  & \textbf{0.7967} & 0.3454 \\
Expert \(g_3\) Queried Accuracy  & \textbf{0.5764} & 0.4291 \\
Expert \(g_1\) Deferral Ratio    & 0.3233 & 0.1843 \\
Expert \(g_2\) Deferral Ratio    & 0.2266 & 0.1850 \\
Expert \(g_3\) Deferral Ratio    & 0.3981 & 0.1847 \\
\bottomrule
\end{tabular}
\end{table}

\begin{figure}[ht]
    \centering
    \captionsetup{font=footnotesize}

    \begin{subfigure}[t]{0.48\textwidth}
        \centering
        \includegraphics[height=4.2cm]{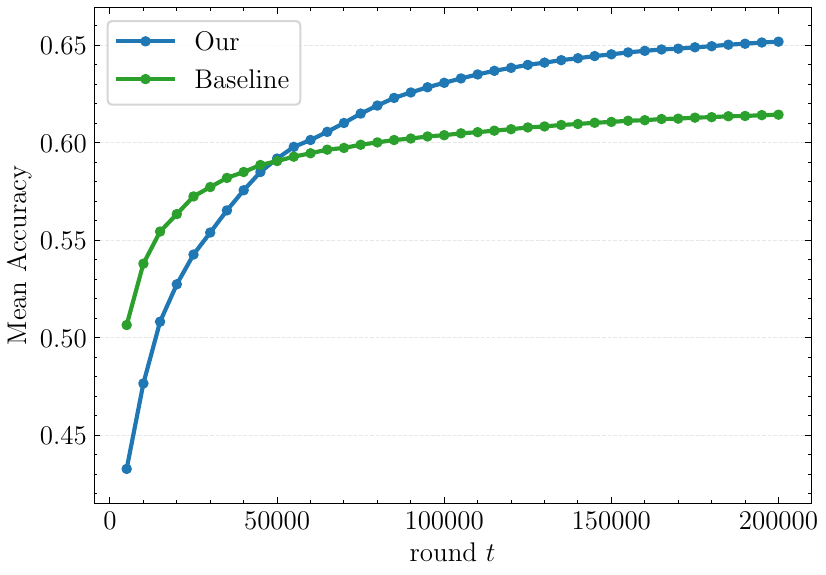}
        \caption{Average accuracy.}
        \label{fig:cifar10h:mean_acc}
    \end{subfigure}\hfill
    \begin{subfigure}[t]{0.48\textwidth}
        \centering
        \includegraphics[height=4.2cm]{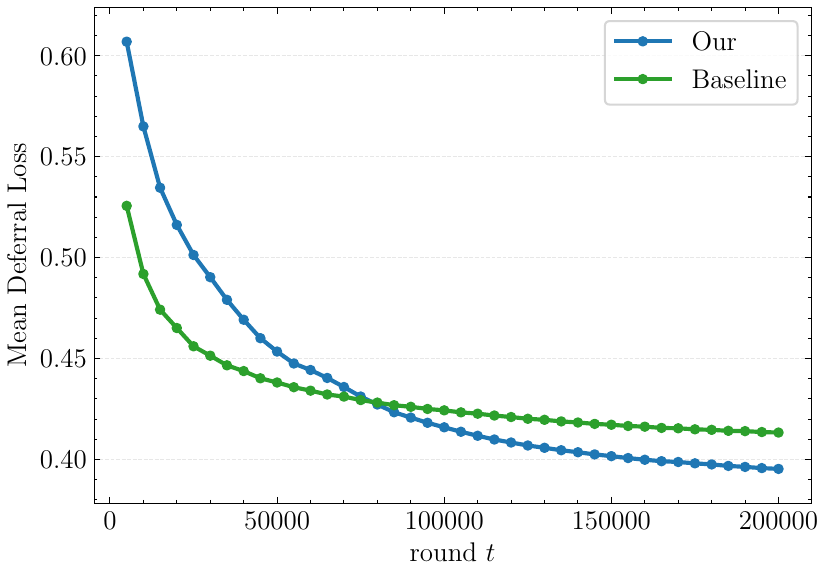}
        \caption{Average true deferral loss.}
        \label{fig:cifar10h:mean_def_loss}
    \end{subfigure}

    \caption{
        Results of experiment on CIFAR10H with image corruption from CIFAR10C. Results are averaged over 5 independent runs.
    }
    \label{fig:cifar10h:exp}
\end{figure}

\end{document}